\documentclass{article}

\usepackage{microtype}
\usepackage{graphicx}
\usepackage{dsfont}
\usepackage{subfigure}
\usepackage{booktabs} % for professional tables
\usepackage{csquotes}

\usepackage{hyperref}

\usepackage[accepted]{icml2025}
\usepackage{amsmath}
\usepackage{amssymb}
\usepackage{mathtools}
\usepackage{amsthm}
\usepackage{bm}
\usepackage{csquotes}
\usepackage{ragged2e}
\usepackage[capitalize,noabbrev]{cleveref}
\usepackage[user,hyperref,savepos]{zref}

\theoremstyle{plain}
\newtheorem{theorem}{Theorem}[section]
\newtheorem{proposition}[theorem]{Proposition}
\newtheorem{lemma}[theorem]{Lemma}

\newtheorem{remark}[theorem]{Remark}
\theoremstyle{definition}

\usepackage[textsize=tiny]{todonotes}

\usepackage{multirow}       % For multirow spanning cells (MAYBE REMOVE)
\usepackage{xspace}         % spacing adjustments (for instance, when \method is used inline) and flexibility
\usepackage[inline,shortlabels]{enumitem}       % for inline itemization and margin control
\DeclareMathOperator*{\argmin}{arg\,min}

\newcommand{\method}{\textbf{$\mathsf{FLOW}$}\xspace}
\newcommand{\methodbold}{{$\boldsymbol{\mathsf{FLOW}}$}\xspace}
\newcommand{\methoditalic}{\textit{\textsf{FLOW}}\xspace}

\icmltitlerunning{
Upweighting Easy Samples in Fine-Tuning Mitigates Forgetting
}

\begin{document}

\twocolumn[
\icmltitle{
Upweighting Easy Samples in Fine-Tuning Mitigates Forgetting
%Pre-trained Loss-Based Sample Weighting to Mitigate Forgetting During Fine-Tuning
% Mitigating Catastrophic Forgetting in Fine-tuning with Pre-trained Loss-Based Sample Weighting
% Loss-driven sample weighting to mitigate catastrophic forgetting in fine-tuning
%during Finetuning can be (provably) mitigated through sample-wise weighted loss
}



\icmlsetsymbol{equal}{*}

\begin{icmlauthorlist}
\icmlauthor{Sunny Sanyal}{equal,ut}
\icmlauthor{Hayden Prairie}{equal,ut}
\icmlauthor{Rudrajit Das}{equal,goo}
\icmlauthor{Ali Kavis}{equal,ut}
\icmlauthor{Sujay Sanghavi}{ut}
\end{icmlauthorlist}

\icmlaffiliation{ut}{University of Texas at Austin}
\icmlaffiliation{goo}{Google Research}

\icmlcorrespondingauthor{Sunny Sanyal}{sanyal.sunny@utexas.edu}
\icmlcorrespondingauthor{Hayden Prairie}{haydenprairie@utexas.edu}
\icmlcorrespondingauthor{Rudrajit Das}{dasrudrajit@google.com}
\icmlcorrespondingauthor{Ali Kavis}{kavis@austin.utexas.edu}
\icmlcorrespondingauthor{Sujay Sanghavi}{sanghavi@mail.utexas.edu}

% % \icmlaffiliation{yyy}{Department of XXX, University of YYY, Location, Country}
% % \icmlaffiliation{comp}{Company Name, Location, Country}
% % \icmlaffiliation{sch}{School of ZZZ, Institute of WWW, Location, Country}


\icmlkeywords{Catastrophic Forgetting, Sample Weighting, Fine-tuning, Pre-trained}

\vskip 0.3in
]



\printAffiliationsAndNotice{\icmlEqualContribution} % otherwise use the standard text.

\begin{abstract}
Fine-tuning a pre-trained model on a downstream task often degrades its original capabilities, a phenomenon known as \enquote{catastrophic forgetting}. This is especially an issue when one does not have access to the data and recipe used to develop the pre-trained model. Under this constraint, most existing methods for mitigating forgetting are inapplicable. To address this challenge, we propose a \textit{sample weighting scheme for the fine-tuning data} solely based on the pre-trained model's losses. Specifically, we upweight the easy samples on which the pre-trained model's loss is low and vice versa to limit the drift from the pre-trained model. Our approach is orthogonal and yet complementary to existing methods; while such methods mostly operate on parameter or gradient space, we concentrate on the sample space. We theoretically analyze the impact of fine-tuning with our method in a linear setting, showing that it stalls learning in a certain subspace, which inhibits overfitting to the target task. We empirically demonstrate the efficacy of our method on both language and vision tasks. As an example, when fine-tuning Gemma 2 2B on MetaMathQA, our method results in only a $0.8\%$ drop in accuracy on GSM8K (another math dataset) compared to standard fine-tuning, while preserving $5.4\%$ more accuracy on the pre-training datasets.

\end{abstract}

% \rd{Alternate title 1: \enquote{A Sample-Weighting Approach for Mitigating Catastrophic Forgetting}}

% \rd{Alternate title 2 (mentions fine-tuning explicitly): \enquote{A Sample-Weighting Approach for Mitigating Catastrophic Forgetting During Fine-Tuning}}

\section{Introduction}
\label{sec:introduction}

In the modern era of large-scale machine learning, one of the central goals is to design models capable of performing multiple tasks. 
Traditionally, this is achieved by training an appropriately large model over datasets of multiple tasks, ensuring that the model jointly learns multiple tasks at once. 
% Expanding the capabilities of the model to new tasks then requires reiteration of the whole process with the addition of the data for each new task. 
Unfortunately, it is not viable to repeat this process with every new additional task due to the scale of contemporary models, necessitating effective strategies that can essentially learn without full retraining.
A resource-efficient convention in machine learning is to take a \emph{pre-trained} model which is trained on some vast and diverse dataset, and 
% augment its capabilities incrementally along a particular direction by 
\emph{fine-tune} it on a new dataset/task. 
% \ali{Examples for fine-tuning on pre-trained models.}
% \draft{For instance a pre-trained base language model is fine-tuned to follow human instructions through supervised fine-tuning (SFT) as a standard practice in prior literature \cite{dubey2024llama}. Similarly pre-trained vision models are also finetuned on task specific datasets for applications such as fine-grained image classification and generation.} \ali{We need several examples, in more concise words.} 
Such pre-trained models are typically large and expensive to train from scratch but perform well on a variety of tasks while offering a versatile basis for learning a new task. 
% Fine-tuning achieve application specificity, ensure user safety, and mitigate toxicity.

Fine-tuning is a delicate process that should ideally serve multiple objectives simultaneously; we would like to use the base model and its capabilities to facilitate learning a strong model on the downstream task, and in the meantime, preserve the existing abilities of the pre-trained model. 
On this particular front, the major challenge in standard, unregulated fine-tuning is the \emph{catastrophic forgetting} phenomenon.
In broad terms, it describes the performance decline of the pre-trained model on previously observed data/tasks after fine-tuning on a new one. 
When the learning process for the downstream task interferes with the previously-learned representations beyond tolerable margins, the pre-trained model loses its prior capabilities and significantly under-performs on previously-learned tasks.
% There may be various reasons behind the forgetting phenomenon, such as but not limited to, shift in the data distribution across tasks, alterations of learned representations, misalignment of the tasks, limited representation capacity of the base model.

Mitigating catastrophic forgetting is an active area of research with many fundamental questions awaiting solutions. 
The key idea is to constrain the fine-tuning process to prevent the degeneration of the learned representations while guiding the learning of the new task to augment existing capabilities. 
The literature on the topic offers various approaches based on the available knowledge pertaining to the pre-training process. 
% the degree of data availability 
% and the nature of the learning setup. 
In fact, pre-training-specific data availability and how it is treated predominantly dictates the success of mitigating forgetting. 
% as the overall performance of the final model relies heavily on the appropriate use of the pre-training-specific information.
In many real-life scenarios, however, the data and the training recipe used for generating the pre-trained model are not available \citep{radford2021learning, touvron2023llama, Touvron2023Llama2O, grattafiori2024llama3herdmodels, jiang2024mistral}. 
Naturally, one needs to approach the forgetting phenomenon accordingly to design realistic methods.
% The recent advances of our decade, especially within the context of large language models, have reinforced time and again that careful, high-quality data curation is one of the most valuable pieces of the puzzle. 

Therefore, we focus on the case in which we have \emph{\textbf{no access} to the pre-training-specific information} during the fine-tuning process; we call it the \textbf{\emph{data-oblivious}} setting.
The only piece of information available during fine-tuning is indeed the pre-trained model. 
Therefore, one needs to devise a strategy to regulate and guide the fine-tuning process to preserve the pre-trained model capabilities while learning the new task in the absence of prior knowledge.
Under this challenging setting, we present an answer to the question:
\begin{center}
    \textit{Can we design a principled method that mitigates forgetting during fine-tuning in the data-oblivious setting?}
\end{center}
% With limited number of exceptions, the existing strategies to mitigate forgetting rely on the existence of pre-training-specific information, and more often than not, require access to the dataset. 
% Unfortunately, this is hardly the case in practice and
% in the absence of pre-training knowledge that helps find a jointly favorable region both in terms of the pre-trained and fine-tuning loss. 
In this paper, we propose \textbf{F}ine-tuning with Pre-trained \textbf{L}oss-\textbf{O}riented \textbf{W}eighting (\methodbold) to mitigate catastrophic forgetting in the data-oblivious setting. 
%{\color{red}Drawing inspiration from distributionally robust optimization (DRO) literature \citep{qi2021online}, we upweight the \enquote{easy samples} on which the pre-trained model's loss is low and vice versa.} 
Our key insight is upweighting the \enquote{easy} samples on which the pre-trained model's loss is low and vice versa.
We believe that boosting the samples on which the pre-trained model performs well (i.e., has low loss)
%which already achieve small loss with respect to the pre-trained model 
will introduce supervised bias to the gradient updates in favor of the pre-trained model. Intuitively, this will prevent the parameters from deviating %uncontrollably
too much from the {initial pre-trained state}, thus mitigating forgetting. 

%Sample-wise importance weighting has been studied in optimization \citep{needell2014stochastic, zhao2015stochastic, Alain2015VarianceRI, stich2017safe} and machine learning \citep{Loshchilov2015OnlineBS, shrivastava2016training, katharopoulos2018not, kawaguchi2020ordered, dasunderstanding}, where the goal is to speed up the optimization process by reducing the variance of the gradient updates. 
%It has been shown that assigning weights to the samples proportional to the norm of their gradient will yield the desired results, and several papers use the per-sample loss as a reliable surrogate.
% where preferred strategy is to assign weights \enquote{proportional} to the gradient norm or loss value associated with the sample. The idea is to speed up the optimization/training process by up-weighting the samples with large gradients or high losses. 

Some prior papers assign more importance to {samples} with \textit{larger losses} to accelerate the training process \citep{Loshchilov2015OnlineBS, shrivastava2016training, katharopoulos2017biased, kawaguchi2020ordered, dasunderstanding}. 
We follow the reciprocal reasoning; we tweak the fine-tuning process in favor of the pre-trained model by assigning \emph{larger} weights to samples with \emph{smaller} pre-trained loss values. 
% taking controlled steps for fine-tuning task to mitigate performance decline in pre-trained tasks. 
We elaborate on this while stating our \textbf{contributions} next.
\begin{enumerate}[topsep=0.1cm]%leftmargin=0.5cm
    \item %We propose \method, a one-shot sample-wise weighting scheme which assigns weights per sample \emph{inversely proportional} to their pre-trained losses. In particular, we fine-tune the pre-trained model using a modified loss function, scaled by the per-sample weights. 
    To mitigate forgetting, we propose \methodbold, which fine-tunes the pre-trained model using a sample-wise weighted loss.
    Inspired by robust optimization ideas, we derive the $i^\text{th}$ sample's weight to be $\exp( -\ell_{i}/\tau)$, where $\ell_{i}$ is the $i^\text{th}$ sample's pre-trained loss and $\tau$ is a parameter which we set as median$(\ell_{i})$ in practice. %of the $\ell_{i}$'s. 
    %We propose a one-shot sample-weighting scheme inspired by robust optimization wherein the $i^\text{th}$ sample's weight is $\exp( -\ell_{i}/\tau)$, where $\ell_{i}$ is the $i^\text{th}$ sample's pre-trained loss and $\tau$ is a parameter which we set $=$ median of the $\ell_{i}$'s. 
    %In this sense, 
    Thus, our method is essentially \textit{parameter-free}. %in practice.
    %To mitigate forgetting, we fine-tune the pre-trained model with a modified objective function, wherein the per-sample losses are scaled with our weighting scheme. 
    %We call our fine-tuning method \methodbold.
    
    \item %We demonstrate that models fine-tuned using \method mitigate catastrophic forgetting of the pre-training task while incurring only a minimal decrease in task-specific fine-tuning accuracy compared to standard full fine-tuning. Furthermore, our method compares favorably against four widely used baselines from the catastrophic forgetting literature in the data-oblivious setting. 
    We demonstrate the superiority of \method %compared to standard fine-tuning as well as four other baselines 
    over relevant baselines (model averaging, $\ell_2$ regularization, LoRA, etc.) 
    %in the data-oblivious setting 
    in both vision and language model experiments. For instance, ResNet-50 fine-tuned with \method on six image classification datasets achieves {$\sim \mathbf{17}$\% %16.83\% 
    higher average accuracy} (over pre-training and fine-tuning data) than standard fine-tuning, while also surpassing other relevant baselines (see \cref{table:main_vision_table}). When fine-tuning Gemma 2 2B on math datasets, the corresponding improvement of \method over standard fine-tuning is $\sim \mathbf{4}\%$ %$3.86\%$ 
    (see \cref{tab:main_lang_table}).
    
    %\item %We empirically show that \method is complementary to existing methods for mitigating forgetting 
    % , e.g., model averaging, LoRA, and $\ell_2$-regularization, 
    %such 
    %in the sense
    %that it effectively improves the performance of the standalone methods when combined (see \cref{table:complementary_vision_table,tab:lang-ours-baseline}).

    \item We also empirically show that combining \method with existing methods for mitigating forgetting \textit{improves the performance} of the base methods (see \cref{table:complementary_vision_table,tab:lang-ours-baseline}).

    %\sunny{Base methods seems confusing. Use standalone methods.}
    %\rd{base is clear and it saves one line}

    \item We theoretically analyze the effect of fine-tuning with \method for linear models. %In particular, \textit{training is stalled} along a certain direction, \textit{impeding overfitting to the fine-tuning task} (\Cref{rmk-2}).
    In particular, the covariance matrix of the fine-tuning data weighted by \method has a small eigenvalue and \textit{training is stalled} along the corresponding eigenvector, \textit{impeding overfitting to the fine-tuning task} (see \Cref{rmk-2}).

    %\rd{How about re-phrasing the last point as "We also empirically show that combining \method with existing methods for mitigating forgetting improves the performance of the standalone methods (see \cref{table:complementary_vision_table,tab:lang-ours-baseline})."
    %"Our method is complementary to existing methods for mitigating forgetting. We empirically show that applying our method in conjunction with existing methods improves the performance of the standalone method (see \cref{table:complementary_vision_table,tab:lang-ours-baseline}).
    %}

    % \item In the data-oblivious setting where we cannot evaluate performance on pre-training data, \method works consistently well across different experiments with heuristic initialization, $\tau \approx \text{median}_i\,\, \ell(\bW^*; (\bx_i, y_i))$. 
    % \item We test it against several baselines for vision and language tasks. We provide comprehensive empirical evidence that our method shows comparable if not better performance as a standalone \finetuning method.
    % \item We conduct further experiments to
    % % where we combine our sample-wise weighting with existing strategies to mitigate forgetting. Our study 
    % demonstrate that \method can also operate as a plug-in framework that improves the performance of standalone baselines.
\end{enumerate}

We end this section with a preview of the comparison of our method \method with some relevant baselines (in the data-oblivious setting) in \Cref{fig:baseline_figure-1}.
\begin{figure}[t!]
    \centering
    \includegraphics[width=0.95\columnwidth]{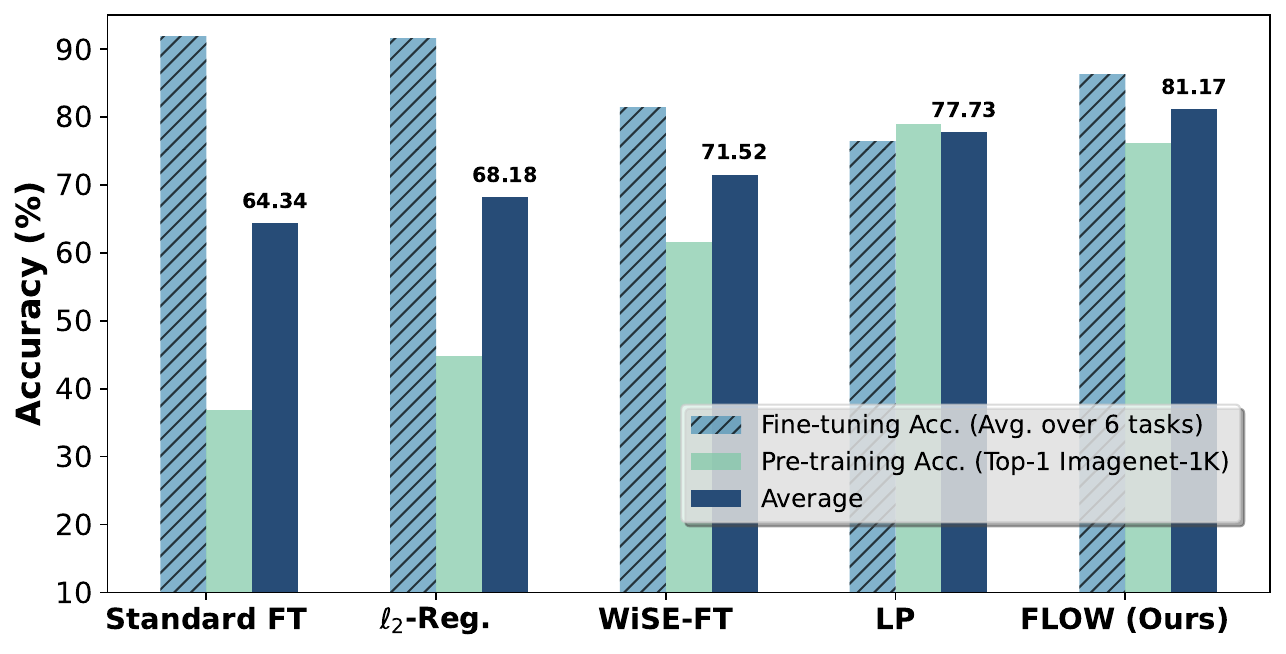}
     % \caption{\textbf{Comparison of \methodbold with different values of $\tau$ and other baselines also with different hyper-parameter values.} This plot is for ResNet-50 on the Stanford cars dataset. \method's plot in red is with $\tau = \{10,20,30,40,50\}$ percentile of the per-sample losses. As the name \enquote{random subset} may imply, we just pick a random subset of the fine-tuning data and train on this subset (to limit the drift from the pre-trained model).
     % Tunable param for Ours (temp=10, 20, 30, 40, 50) and Random subset (random sample size= [10\%, 20\%, 30\%, 40\% 50\%]).} 
     \caption{\method versus standard fine-tuning (FT) and relevant baselines for a ResNet-50 model pre-trained on ImageNet-1K (from \Cref{table:main_vision_table}). \methodbold \textbf{achieves the best average accuracy} (between pre-training and target fine-tuning accuracies).}
    \label{fig:baseline_figure-1}
\end{figure}

\section{Related Work}
\label{sec:related-work}
\subsection{Mitigating Catastrophic Forgetting}
% The literature catastrophic forgetting is vast to the extent that it is quite challenging to cover all the related work in this manuscript. 
% It has been studied to great depths in the context of continual learning, in which the focal point is learning multiple tasks in sequence.
%We do a comprehensive survey of the literature with a specific focus on the most related work to our proposed setting.
We begin by summarizing the vast literature on catastrophic forgetting with a focus on prior works most relevant to our proposed setting.
For a streamlined presentation, we survey prior work in two settings -- data-aware and data-oblivious. 
Due to space limitations, we refer the readers to \cref{app:related-work} for a more detailed and explanatory review of the literature. 
% Depending on the type of the data-related information and how it is treated, the latter setting is segregated into subcategories.
% \rd{We need to shrink the data-aware part}

\subsubsection{Data-aware approaches}% for mitigating forgetting}
The majority of the approaches for mitigating %catastrophic 
forgetting assume task-specific knowledge access to different extents; either (a subset of) the pre-training dataset itself or some information/statistic computed from pre-training data. Below, we describe the data-aware approaches based on how they make use of task-specific knowledge.

\textbf{Regularization-based methods. } This line of work aims to preserve existing capabilities by keeping the parameters close to the pre-trained model. The key idea is to introduce task-specific regularization to penalize modifications along the ``important'' directions for the old tasks %, i.e., pre-trained model 
\citep{ahn2019uncertainty}. 
% The simplest example is the standard $\ell_2$ regularization, i.e., we optimize the regularized loss function $\mathcal L(\theta) = \ell_B(\theta) + \frac{\lambda}{2} \| \theta - \theta^* \|_2^2$, where $\theta^*$ is the weights of the base model and $\| \cdot \|_2$ is the Euclidean norm. Minimizing for the $\ell_2$-regularized loss precludes drifting away from $\theta^*$. The main downside of naive regularization is that the update is constrained uniformly across all directions; one could design more sophisticated and mathematically-grounded techniques in the presence of task-specific information. 
\citet{kirkpatrick2016overcoming} introduces the elastic weight consolidation (EWC) algorithm, which estimates the important directions by approximating the Fisher information matrix. 
Several variants of EWC have been proposed \citep{schwarz2018progress, ritter2018online, Lee2020ContinualLW, liu2018rotate}. \citet{zenke2017continual, aljundi2018memory} infer the importance of each parameter by their variational effect on the outputs. In a similar spirit, \citet{lee2017overcoming} aims to match the posteriors of the pre-trained and fine-tuned models. 

\textbf{Optimization-driven methods. } Another perspective to mitigating forgetting is guiding the optimization process by constraining the algorithms directly as opposed to manipulating the loss function. 
The core idea is to keep track of \enquote{important directions} for the old tasks, and train on the new task \enquote{orthogonally.} 
This could be done by storing prior data samples or gradients in a buffer \citep{lopezpaz2017gradient, farajtabar2020orthogonal, chaudhry2018efficient} or by incrementally expanding the subspace of important directions without storing task-specific information \citep{zeng2019continual, wang2021training, wang2023orthogonal}.

\textbf{Replay-based methods. } A more direct approach is to store old task samples in buffers and introduce them into the training process for the new task to refresh task-specific representations periodically.
There are several components to such methods. Some prior work focus on \emph{data selection} based on the nature of old data access \citep{rebuffi2017icarl, aljundi2019gradient, Bang2021RainbowMC, chaudhry2019continual, isle2018selective, delange2021continual, borsos2020coresets, tiwari2021gcr} (e.g., streaming versus on-demand). Another important perspective is the \emph{re-introduction strategy} of the stored information into the fine-tuning process \citep{silver2002task, li2016learning, Triki2017EncoderBL, Lee2019OvercomingCF, dhar2019learning, rebuffi2017icarl, riemer2019learning, chaudhry2019continual, delange2021continual, tiwari2021gcr}.

\textbf{Architecture-driven methods. } Another technique to limit interference between tasks is to allocate a separate trainable set of parameters per task. 
This could be done by initializing a sub-networks per new task \citep{Rusu2016ProgressiveNN, aljundi2017expertgate, patrick2020routing, rajasegaran2019random, Ramesh2021ModelZA, wang2023incorporating, wang2022coscl}, gradually expanding the parameters of a base network \citep{yoon2018lifelong, Ostapenko2019LearningTR, hung2019compacting}, or segregating a fixed model into task-specific subsets \citep{mallya2018piggyback, kang2022forgetfree, serra2018overcoming, worstman2020supermasks, Mallya2017PackNetAM, gurbuz2022nispa, jung2020continual}. 
% While some parameters are task-specific, 
% and are not altered by the training process for other tasks, 
% parts of the model could be shared to enable knowledge transfer. 
The main downside with this line of work is that task identities must be known for inference to (de)activate relevant sub-networks \citep{aljundi2017expertgate}. 
% \citet{Aljundi2016ExpertGL} develop dedicated strategies to overcome the need for task identification by automatizing task-specific parameter activation.

\subsubsection{Data-oblivious approaches}
In the less-explored data-oblivious setting, it is particularly challenging to devise a principled approach, as there is no access to any data-specific information, except for the pre-trained model.
%we have no access to any data-specific information.
%As one of the simplest methods, a line of work explores the concept of ``model averaging'' (MA) which essentially takes the combination of parameters of two models which individually do well on separate tasks. 
One line of work explores the simple idea of ``model averaging'' (MA) which essentially does a convex combination of the parameters of the pre-trained model and that of the fully fine-tuned model for the new task. 
MA and more sophisticated model merging variants have been studied in relevant context to forgetting \citep{lubana2021quadratic, wortsman2021robust, ilharco2023editing, lin2023mitigating, kleiman2025soupgomitigatingforgetting}. 
Some recent works \citet{chen2024mofo,panda2406lottery} introduce different strategies to selectively update a subset of parameters in a pre-training data-agnostic manner.
Finally, \citet{biderman2024lora} has shown that LoRA \cite{hu2022lora} could be effective for mitigating catastrophic forgetting in transformers. 
%A recent submission \citet{anonymous2024mofo} proposes a method that selectively updates a subset of parameters per fine-tuning iteration. %Essentially, they combine Adam with a parameter-selection mechanism; parameters with the highest $\alpha \%$ of momentum values are updated.
Unlike the methods discussed above which focus on the parameter or gradient space, ours focuses on the sample space.

\subsection{Sample Selection and Weighting}
Sample-wise importance selection/weighting has been studied in optimization papers \citep{needell2014stochastic, zhao2015stochastic, Alain2015VarianceRI, stich2017safe} and ML papers \citep{Loshchilov2015OnlineBS, shrivastava2016training, katharopoulos2017biased, katharopoulos2018not, kawaguchi2020ordered, dasunderstanding} to speed up the optimization/training process by reducing the variance of the gradient updates. %It has been shown that assigning weights to the samples proportional to the norm of their gradient will yield the desired results, and several papers use the per-sample loss as a reliable surrogate.
Such papers advocate focusing on \enquote{hard} samples with high-gradient norms or losses. In contrast, we focus on \enquote{easy} samples to mitigate forgetting. Another line of work focuses on robust learning under uncertain data distributions. Distributionally robust optimization (DRO) proposes to minimize the worst-case weighted loss, where the sample weights are constrained or regularized \citep{ben2013robust,levy2020large,duchi2021learning,qi2021online}. 
%\draft{Our weighted loss is motivated by the formulatio/ of \citet{qi2021online}, but it is {exactly the opposite} in spirit; this is discussed at the end of \Cref{sec:method}.} \ali{Let's remove this for space saving, we already talk about it in the proposed method}
Some recent works \citep{xie2024doremi,chen2024take,anonymous2025dynamic} propose dynamic sample-weighting strategies for LLM training based on the previously discussed ideas. %classical ideas.
%Also, unlike the papers discussed here, the sample weights in our paper are one-shot (or static) and depend only on the pre-trained model.

\section{Notation and Definitions}
%We denote vectors with lower-case boldface letter and matrices are in upper-case bold font. 
$\mathds{1}(.)$ denotes the indicator variable. For any $n \in \mathbb{N}$, the set $\{1,\ldots,n\}$ is denoted by $[n]$. 
Vectors and matrices are in lowercase and uppercase bold font, respectively. 
The $\ell_p$ norm of a vector $\mathbf{v}$ is denoted by $\|\mathbf{v}\|_p$. The inner product between two vectors $\mathbf{v}$ and $\mathbf{v}'$ is denoted as $\langle \mathbf{v}, \mathbf{v}' \rangle$. A set of $n$ linearly independent $n$-dimensional vectors $\{\mathbf{u}_1,\ldots,\mathbf{u}_n\}$ is said to be 
an orthonormal basis for $\mathbb{R}^n$ if $\langle \mathbf{u}_i, \mathbf{u}_j \rangle = \mathds{1}(i=j)$. A vector $\mathbf{v} = [\text{v}_1,\ldots,\text{v}_n]^\top$ is said to belong to the $n$-dimensional probability simplex $\bm{\Delta}_n$ if $\sum_{i=1}^n \text{v}_i = 1$ and $\text{v}_i \geq 0$ $\forall$ $i \in [n]$. For any $n \in \mathbb{N}$, $\mathbf{I}_n$ denotes the identity matrix of dimension $n$. 
%A random variable ${Z}$ drawn from a probability distribution $\mathcal{P}$ is denoted by ${Z} \sim \mathcal{P}$. 
%Let $x$ be a random variable sampled from some distribution $\mathcal D$. We define that $x$ follows the law of $\mathcal D$ by $x \sim \mathcal D$. 
%\rd{I don't think we need to explain the above}

\section{Proposed Algorithm}
\label{sec:method}
Our proposed algorithm consists of two main steps: 
\begin{enumerate*}
[(\itshape i\hspace*{1pt})]
    \item computing weights for the samples based on their respective pre-trained loss values; and
    \item fine-tuning with a weighted loss wherein the per-sample losses are scaled by their respective weights.
\end{enumerate*}
The sample-wise weights are computed once and used throughout the entire fine-tuning process. %without change. 
%The recipe for per-sample weight computation is created with inspirations from the distributionally robust optimization literature. 
We formally state our proposed fine-tuning protocol in \cref{alg:method} and delve into its design details in the sequel.
% We will now present our algorithm. Suppose the non-negative loss function we use for training is denoted by $\ell$, e.g., the cross-entropy loss, the squared loss, etc. 
%For shorthand notation, we denote the per-sample loss as $f_i(\bW) = \ell(\bW; (\bx_i, y_i))$. 

%\rd{Change sample weights to $w$}

\begin{algorithm}[ht] %!htb
   \caption{\textbf{F}ine-tuning with Pre-trained \textbf{L}oss-\textbf{O}riented \textbf{W}eighting ({\methodbold})}
   \label{alg:method}
\begin{algorithmic}
    \STATE {\bfseries Input:} %Original weights 
    Pre-trained model $\bm{\theta}^{*}$, dataset $\{(\mathbf{x}_i, \text{y}_i)\}_{i=1}^n$ for the %downstream 
    new task, and temperature parameter $\tau$. %, and temperature $\tau > 0$. %A sample weighting function $g(.)$..
    \vspace{0.15 cm}
    % \STATE {\bfseries Hyper-parameter:} Temperature $\tau > 0$.
    % \vspace{0.15 cm}
    \STATE $f_i(\bm{\theta}) \rightarrow$ $i^\text{th}$ sample's loss at $\bm{\theta}$, with a non-negative loss function (e.g., cross-entropy loss). 
    \vspace{0.15 cm}
    \STATE 1. Compute sample weights: $w_i = \exp \left(- \frac{f_i(\bm{\theta}^{*})}{\tau} \right)$.%, $\forall i \in [n]$.
    \vspace{0.15 cm}
    \STATE 2. Weighted loss: $\mathcal L(\bm{\theta}) = \sum_{i=1}^n w_i f_i(\bm{\theta}).$
    \vspace{0.15 cm}
    \STATE 3. Fine-tune with weighted loss: $\widehat{\bm{\theta}}^{*} := \argmin\limits_{\bm{\theta}} \mathcal L(\bm{\theta}).$
    % \vspace{0.15 cm}
    \STATE \textbf{Output:} %Downstream 
    Fine-tuned model $\widehat{\bm{\theta}}^{*}$.
    \end{algorithmic}
\end{algorithm}

%\rd{Also this section addresses this algo as "Algorithm 1", maybe we should switch over to SWPL soon because we address it as SWPL in other sections.}
\begin{remark}
    Depending on the setting, our model might have task-specific components, such as per-task prediction heads (e.g., in vision). %Our algorithm is perfectly applicable in the presence of task-specific components, as well. 
    \cref{alg:method} can be slightly modified in the presence of task-specific components to enhance performance. {Refer to \Cref{alg:multi-head} for these modifications.} 
    %and implementation details.
\end{remark}

\begin{remark}
\label{rmk-tau}
% [\textbf{Regarding temperature $\tau$}] 
As a heuristic prescription, we set $\tau = \textup{median}\left(f_i(\bm{\theta}^{*})\right)$ in all our experiments (unless otherwise stated), which leads to consistently good %empirical
performance. \textbf{Thus, our algorithm is essentially parameter-free in practice}.
\end{remark}

%\textbf{Derivation of our weighting function.}
\textbf{Algorithm design.} Our main intuition is that we can control forgetting by not drifting away too much from the pre-trained model (i.e., $\bm{\theta}^{*}$) during fine-tuning. In the presence of pre-training data, this is done by introducing data-dependent constraints on the parameter space or gradient space. Since we have \emph{no access} to pre-training data, we %concentrate on possible strategies on the sample space. 
redirect our focus towards \textit{strategies on the \textbf{sample space} depending only on the \textbf{pre-trained model}}.

To that end, we propose to infer the easiness 
% (or hardness) 
of each sample of the fine-tuning dataset with respect to the pre-trained model, based on the per-sample losses $f_i(\bm{\theta}^{*})$'s (see Alg. \ref{alg:method}). We say that the $i^\text{th}$ sample is \enquote{easy} if $f_i(\bm{\theta}^{*})$ is \enquote{small}.\footnote{This is not a formal definition and so \enquote{small} is not quantified.} Intuitively, prioritizing the \enquote{easy} samples during fine-tuning would limit the drift from $\bm{\theta}^*$. On the other hand, over-focusing on the \enquote{easy} samples would probably lead to poor performance on the fine-tuning task. Thus, it is important to strike a balance. 

Let us formalize these ideas mathematically. For fine-tuning on the new task, let us consider the objective function $\mathcal{L}_{\bm{\pi}}(\bm{\theta}) = \sum_{i=1}^n \pi_i f_i(\bm{\theta})$, where $\bm{\pi} = [\pi_1,\ldots,\pi_n]^\top$ is a \textbf{static} design-choice %belonging to the $n$-dimensional simplex 
$\in \bm{\Delta}_n$ (i.e., $\sum_{i=1}^n \pi_i = 1$ and $\pi_i \geq 0$ $\forall$ $i \in [n]$) which we allow to only depend on the pre-trained model's losses $\{f_i(\bm{\theta}^{*})\}_{i=1}^n$ (and \textit{not} the current model's losses $\{f_i(\bm{\theta})\}_{i=1}^n$). We would like to design $\bm{\pi}$ so that: 
\begin{enumerate} [topsep=0.1cm]
    \item for all $i \neq j$ such that $f_i(\bm{\theta}^{*}) \leq f_j(\bm{\theta}^{*})$, $\pi_i \geq \pi_j$,
    \item $\bm{\pi}$ does not concentrate around one or a few samples but rather spreads uniformly over the samples.
    %$\bm{\pi}$ does not concentrate around one or just a few samples, but spreads rather evenly to induce exploration
    % the entropy of $\bm{\pi}$ is as high as possible (in simple words, $\bm{\pi}$ is as close to the uniform distribution as possible).
\end{enumerate}
%\rd{"exploration" is not a standard term in ML training so let's not have that.}
These two requirements can be enforced by minimizing the following function (w.r.t. $\bm{\pi}$) involving negative \emph{entropic regularization}:
\begin{equation}
    g(\bm{\pi}) = \sum_{i=1}^n \pi_i f_i(\bm{\theta}^{*}) 
    + \tau \sum_{i=1}^n \pi_i \log \pi_i.
\end{equation}
Here $\tau > 0$ is a parameter controlling the extent of the second requirement which is facilitated by the entropy term. 
%promotes %a more even spread over the samples. 
%the uniform distribution over the samples. 
We now state the minimizer of $g(\bm{\pi})$ (proof is in \Cref{pf-prop-pi}). 
\begin{proposition} %[Proof in \Cref{pf-prop-pi}]
\label{prop-pi}
Let $\bm{\pi}^{*} = [\pi_1^{*}, \ldots, \pi_n^{*}]^\top = \argmin\limits_{\bm{\pi} \in \bm{\Delta}_n} g(\bm{\pi})$. Then we have  ${\pi}_i^{*} = \frac{1}{Z} \exp \left(- \frac{f_i(\bm{\theta}^{*})}{\tau} \right)$, where $Z$ is the normalizing factor. 
%\begin{equation*}
%    {\pi}_i^{*} = \frac{1}{Z} \exp \left( - \frac{f_i(\bm{\theta}^{*})}{\tau} \right), \text{ where }
%\end{equation*}
%where 
%$Z = \sum_{j=1}^n \exp \left( - \frac{f_j(\bm{\theta}^{*})}{\tau} \right)$ being the normalizing factor. 
\end{proposition}
%The proof of \Cref{prop-pi} is straightforward and can be found in \Cref{pf-prop-pi}. Finally, 
Modulo the normalizing factor $Z$ (it does not matter when optimizing w.r.t. $\bm{\theta}$), note that $w_i$ and $\mathcal L(\bm{\theta})$ in  \Cref{alg:method} are equivalent to ${\pi}_i^{*}$ and $\mathcal{L}_{\bm{\pi}^{*}}(\bm{\theta})$, respectively. %Our algorithm is named \method to reflect the pre-trained loss-based exponential weighting scheme.
%The name of our algorithm (\method) is self-explanatory based on the final weighting scheme.

% \textbf{(Dis)connection with distributionally robust optimization (DRO).}
\textbf{Distributionally robust optimization (DRO) perspective.}
Our formulation above is motivated by prior work on DRO \citep{qi2021online}, but it is  \textbf{exactly the opposite} of DRO in spirit. Specifically, in our setting, \citet{qi2021online} consider the following min-max problem: %(translated to our setting):
\begin{equation}
    \label{eq:DRO}
    %\mathcal{L}_\textup{DRO}(\textbf{W})
    \min_{\bm{\theta}} \max_{\bm{\pi} \in \bm{\Delta}_n} \sum_{i=1}^n \pi_i f_i(\bm{\theta}) - \tau \sum_{i=1}^n \pi_i \log \pi_i.
\end{equation}
The first term in Eq. (\ref{eq:DRO}) is the \textit{worst-case} weighted loss at $\bm{\theta}$, while the second term (i.e., entropic regularization) promotes uniform weights. The optimal solution to the inner max function w.r.t. $\bm{\pi}$ turns out to be $\pi_i^{*} \propto \exp \left(\frac{f_i({\mathbf{\bm{\theta}}})}{\tau} \right)$. Note that this is essentially the \textit{inverse of our weighting function} (modulo the normalizing factor) because it assigns a higher weight to samples with larger losses (i.e., the \enquote{hard} samples). The weighting function of DRO would be very conducive to forgetting because it focuses more on the \enquote{hard} samples. Further, our weighting function is static (or one-shot) as it depends only on the losses at ${\mathbf{{\bm{\theta}}}}^{*}$. On the other hand, the weighting function of DRO is dynamic (i.e., it depends on the current point ${\mathbf{{\bm{\theta}}}}$). In fact, after plugging in the optimal value of $\pi$ into Eq. (\ref{eq:DRO}) and simplifying, the DRO objective reduces to $\min_{\bm{\theta}} \sum_{i=1}^n \exp\left(\frac{f_i(\bm{\theta})}{\tau}\right)$; this is noticeably different from our objective $\mathcal L(\bm{\theta})$ in \Cref{alg:method}.

\section{Experimental Setup}
\label{sec: experiments}
%\rd{Fix Appendix F.1 and notation in Appendix}

We empirically evaluate the performance of \method (\Cref{alg:method}) on vision and language tasks, showcasing its effectiveness across different model architectures and modalities.\footnote{Our code is publicly available \href{https://github.com/sanyalsunny111/FLOW_finetuning}{\texttt{here}}.} Here, we explain details of our experiments: baselines, model architectures, datasets, and evaluation metrics.

\textbf{Baselines.} In our language and vision experiments, we compare \method against relevant baselines in the \textit{data-oblivious setting}, namely, standard fine-tuning (fine-tuning with vanilla unweighted loss), $\ell_2$-regularization [following  \citet{kirkpatrick2016overcoming}], and WiSE-FT \cite{wortsman2021robust} (model averaging of pre-trained and standard fine-tuned models). Additionally, we compare against linear probing (fine-tuning only the classification head, keeping the body frozen)  %\textcolor{red}{a distillation-based method for mitigating forgetting called \enquote{learning without forgetting} (LwF) \cite{li2016learning} in vision experiments}, 
and low-rank adaptation (LoRA) \cite{hu2022lora} in language experiments. More details on the  baselines can be found in \cref{app:add-baseline-detials}.

\subsection{Vision Experiments} 
\label{exp-vis}
We compare the performance of \method and associated baselines in a transfer learning setup. %with vision models using six datasets.

\textbf{Models.} {We experimented with %ResNet-18, ResNet-50 \cite{wightman2021resnet}, and ViT-B/16 \cite{dosovitskiy2020vit} each of which is  
ResNet-18 and ResNet-50 \cite{wightman2021resnet} pre-trained on Imagenet-1K (IN-1K).   %\cite{russakovsky2015imagenet} 
%and CLIP ViT-B/32 \cite{radford2021learning} pre-trained on an undisclosed dataset.
} %that we subsequently fine-tuned for our experiments.

\textbf{Datasets.} We used seven widely-used image classification datasets: CIFAR-10 \cite{Krizhevsky09learningmultiple}, CIFAR-100 \cite{Krizhevsky09learningmultiple}, Flowers102 \cite{nilsback2008automated}, Caltech101
\cite{Li2022}, Cars \cite{krause20133d}, and Dogs \cite{parkhi2012dog}.
%, and \textcolor{red}{Resisc45 \cite{Cheng_2017}}. %and Food101 \cite{bossard14}}. 

%\textcolor{red}{For ResNets, we used the first six datasets. For CLIP-ViT-B/32, we use all the datasets above.} 
%, while for ViT-B/16, we used only the large-scale Food-101 dataset.}
% For the vision models, we consider two Imagenet-1K pre-trained \cite{russakovsky2015imagenet} ResNet18 and ResNet50 taken from \citet{wightman2021resnet} that we subsequently fine-tuned for our experiments.
% \rd{Put a footnote to say the exact algorithm for vision models involving task-specific heads is in \Cref{alg:multi-head}.}

\textbf{Evaluation metrics.} %The vision models are trained in a slightly different way compared to language models, due to task-specific parts\footnote{See \Cref{alg:multi-head} for \method with task-specific parts.}, such as classification head (head) and batch norm (BN). 
Vision models are trained with task-specific parts, such as classification head (head) and batch-norm (BN); see \Cref{alg:multi-head} for how \method works with with task-specific parts. 
Forgetting is measured by how much the model's top-1 validation accuracy on ImageNet-1K (subsequently referred to as IN-1K accuracy) reduces after fine-tuning. 
% The overall degradation in IN-1K accuracy from the initial evaluation of the pre-trained model to post-fine-tuning indicates the extent of catastrophic forgetting. 
We report the fine-tuning performance in terms of \emph{average fine-tuning accuracy} over all the related datasets following \cite{goyal2022finetune, ilharco2023editing}. For IN-1K evaluation \emph{after} fine-tuning, we replace the task-specific components of the {fine-tuned} model with their pre-trained counterparts. 
%\textcolor{red}{ For the CLIP ViT-B/32 model, we used zero-shot evaluation on IN-1K to assess its pre-trained capabilities.}
An extended discussion on experimental details, evaluation, and hyper-parameters are in \cref{app:vision-hyperparameters}.
We also report the \textbf{average} of IN-1K accuracy and averaged fine-tuning accuracy for each method; this is a reasonable \textit{unified metric} to evaluate the performance of a method jointly on the pre-training and fine-tuning data.

% \paragraph{Evaluation}

% To evaluate the base model's pre-trained capabilities, we first assess its top-1 accuracy on the ImageNet-1K validation set. After fine-tuning the model on a task-specific dataset, we also report its top-1 accuracy on ImageNet-1K. The overall degradation in accuracy from the initial evaluation to post-fine-tuning indicates catastrophic forgetting. Next, we also report the task-specific performance in terms of accuracy on that particular task's test/val set following prior works.

\subsection{Language Model Experiments}
\label{sec:llm-experiments-setup}

We follow a similar setup to \citet{biderman2024lora, chen2024mofo}, where a language model's general capabilities are evaluated before and after fine-tuning on a mathematical reasoning dataset. All training for language experiments is done with HuggingFace \texttt{peft} \cite{peft}, \texttt{transformers} \cite{wolf-etal-2020-transformers}, \texttt{datasets} \cite{lhoest-etal-2021-datasets}, and \texttt{accelerate} \cite{accelerate}.

\textbf{Models.} We use Gemma 2 2B \cite{gemmateam2024gemma2improvingopen} and Llama 3.2 3B \cite{grattafiori2024llama3herdmodels} as our base language models. Further details on training hyper-parameters can be found in \cref{app:language-hyperparameters}.

\textbf{Datasets.} Following previous work \cite{biderman2024lora, chen2024mofo}, we fine-tune on MetaMathQA \cite{yu2023metamath}, a mathematical reasoning dataset that is bootstrapped from the training set of GSM8K \cite{cobbe2021training} and MATH \cite{hendrycksmath2021} using a LLM. We train with all 395K samples in MetaMathQA.

\textbf{Evaluation metrics.} To evaluate the validity of \method, we break down our metrics into \emph{general capability} and \emph{target fine-tuning} evaluations. To evaluate general capabilities, we again follow a similar setup to \citet{chen2024mofo}, where we use commonsense reasoning, 5-shot MMLU \cite{hendryckstest2021}, and 3-shot MBPP \cite{austin2021program} metrics. To evaluate the target domain, we use 5-shot GSM8K \cite{cobbe2021training}. 
All evaluations are performed with \texttt{lm-evaluation-harness} \cite{eval-harness}. More details on evaluation and the commonsense metric can be found in \cref{app:further-language-evaluation-details}. Similar to vision, we also report the \textbf{average} of general capabilities and the target fine-tuning accuracies as a \textit{unified metric}.

\section{Experimental Results}
\label{sec:results}
% \rd{Sunny make the baseline names consistent with Hayden}
% \ali{
% \begin{enumerate}
%     \item Best performing method in \textbf{bold}, second best performing \underline{underlined}.
%     \item We have two subsections: 
%         \begin{enumerate}
%             \item[(I)] Ours vs standalone baselines (for visions and language combined).
%             \item[(II)] Baselines + Ours vs. Standalone baselines (for vision and language combined.)
%         \end{enumerate}
% \end{enumerate}
% }
% \draft{In this section, we present our experimental results to demonstrate the efficacy of \method.} % Hayden: IMO this can be removed

%\vspace{-0.5cm}

\begin{table}[!t]
\centering
 \caption{{{
 %\textbf{Performance of Vision Models Pre-trained on ImageNet-1K and Fine-tuned on Six Image Classification Tasks.} 
 \textbf{Performance of \methodbold with ResNet vision models.}
 %We evaluate catastrophic forgetting by measuring the top-1 accuracy on the ImageNet-1K (IN-1K) validation set, quantifying the model's retained pre-trained capabilities. 
 %Task-specific learning is assessed through the average accuracy across six fine-tuning datasets. 
 %Fine-tuning performance is measured by the average accuracy across six fine-tuning datasets. 
 \textbf{Bolded} and \underline{underlined} values indicate the \textbf{best} and \underline{second-best} %\textit{average} of ImageNet-1K (IN-1K) and fine-tuning 
 accuracies within each column (and for each model). 
 Deltas (in color) for IN-1K and target performance are computed w.r.t. the pre-trained and standard fine-tuned models. %\methodbold{} consistently achieves the highest average accuracy on all architectures.
 \methodbold \textbf{attains the best average accuracy} and is better than the second-best method (linear probing) by 2.94\% and 3.44\% for ResNet-18 and ResNet-50, respectively.
 }
 }}
 \vspace{0.1 cm}
\small
 \resizebox{0.97\columnwidth}{!}{%
\begin{tabular}{ll|l|l|c}
\toprule
 & \multicolumn{1}{c|}{{\textbf{Method}}} & \multicolumn{1}{c|}{{\textbf{IN-1K Acc.}}} & \multicolumn{1}{c|}{{\textbf{Target Acc.}}} & \textbf{Average} \\
\midrule
\multirow{6}{*}{\rotatebox{90}{\textbf{ResNet-18}}} 
 % & Ideal         & 69.76  & 89.07  & -- \\
 & Pre-trained   & \textbf{69.76} \hfill {\tiny \textcolor{blue}{(+0.00)}} & --  & -- \\
 & Standard FT   & 19.58 \hfill {\tiny \textcolor{red}{(-50.18)}} & \textbf{89.07} \hfill {\tiny \textcolor{blue}{(+0.00)}} & 54.60 \\
 & Linear Probe       & \textbf{69.76} \hfill {\tiny \textcolor{blue}{(+0.00)}} & 73.57 \hfill {\tiny \textcolor{red}{(-15.50)}} & \underline{71.63} \\
 & $\ell_2$-Reg.         & 34.78 \hfill {\tiny \textcolor{red}{(-34.98)}} & \underline{88.12} \hfill {\tiny \textcolor{red}{(-0.95)}} & 61.45 \\
 & WiSE-FT       & 54.15 \hfill {\tiny \textcolor{red}{(-15.61)}} & 80.23 \hfill {\tiny \textcolor{red}{(-8.84)}} & 67.19 \\
 & \methodbold (Ours) & \underline{65.21} \hfill {\tiny \textcolor{red}{(-4.55)}} & 83.93 \hfill {\tiny \textcolor{red}{(-5.14)}} & \textbf{74.57} \\
\midrule
\multirow{6}{*}{\rotatebox{90}{\textbf{ResNet-50}}} 
 % & Ideal         & 79.02  & 91.78  & -- \\
 & Pre-trained   & \textbf{79.02} \hfill {\tiny \textcolor{blue}{(+0.00)}} & --  & -- \\
 & Standard FT   & 36.91 \hfill {\tiny \textcolor{red}{(-42.11)}} & \textbf{91.78} \hfill {\tiny \textcolor{blue}{(+0.00)}} & 64.34 \\
 & Linear Probe        & \textbf{79.02} \hfill {\tiny \textcolor{blue}{(+0.00)}} & 76.45 \hfill {\tiny \textcolor{red}{(-15.33)}} & \underline{77.73} \\
 & $\ell_2$-Reg.         & 44.78 \hfill {\tiny \textcolor{red}{(-34.24)}} & \underline{91.58} \hfill {\tiny \textcolor{red}{(-0.20)}} & 68.18 \\
 & WiSE-FT       & 61.65 \hfill {\tiny \textcolor{red}{(-17.37)}} & 81.38 \hfill {\tiny \textcolor{red}{(-10.40)}} & 71.52 \\
 & \methodbold (Ours) & \underline{76.09} \hfill {\tiny \textcolor{red}{(-2.93)}} & 86.25 \hfill {\tiny \textcolor{red}{(-5.53)}} & \textbf{81.17} \\
%\midrule
% \midrule
%\multirow{6}{*}{\rotatebox{90}{\textbf{ViT-B/16}}} 
% & Pre-trained         & \textbf{81.10} \hfill {\tiny \textcolor{blue}{(+0.00)}} & --                 & --    \\
% & Standard FT         & 56.11 \hfill {\tiny \textcolor{red}{(-24.99)}}        & \underline{91.60} \hfill {\tiny \textcolor{blue}{(+0.00)}} & 73.86 \\
% & Linear Probe        & \textbf{81.10} \hfill {\tiny \textcolor{blue}{(+0.00)}} & 83.86 \hfill {\tiny \textcolor{red}{(-7.74)}}           & 82.48 \\
% & $\ell_2$-Reg.       & 59.18 \hfill {\tiny \textcolor{red}{(-21.92)}}        & \textbf{91.66} \hfill {\tiny \textcolor{blue}{(+0.06)}} & 75.42 \\
% & LwF                 & 76.39 \hfill {\tiny \textcolor{red}{(-4.71)}}         & 91.23 \hfill {\tiny \textcolor{red}{(-0.37)}}           & \underline{83.81} \\
% & \methodbold (Ours)  & \underline{77.94} \hfill {\tiny \textcolor{red}{(-3.16)}} & 90.57 \hfill {\tiny \textcolor{red}{(-1.03)}}           & \textbf{84.26} \\
\bottomrule
\end{tabular}
}
\label{table:main_vision_table}
\end{table}

\begin{table*}[t!]
    \centering
    \caption{\textbf{Performance of \methodbold with LLMs.} After fine-tuning Gemma 2 2B and Llama 3.2 3B on MetaMathQA, we compare the target fine-tuning performance (GSM8K) with general capability performance. \textbf{Bolded} and \underline{underlined} values indicate the \textbf{best} and \underline{second-best} results within each column (and for each model). 
    %a given model's metrics, respectively. 
    %The delta for general capability metrics is the difference between a given method's accuracy and the pre-trained model's accuracy. The delta for target fine-tuning metrics is the difference between a given method and the standard fine-tuned model.
    Deltas (in color) for general capability metrics and fine-tuning metrics are computed w.r.t. the pre-trained and standard fine-tuned model's accuracy, respectively.  
    We see that \textbf{\methodbold, on average, has the best performance} on general capabilities and target domain, achieving within $\sim0.8\%$ (Gemma 2 2B) and $\sim1.4\%$ (Llama 3.2 3B) of standard fine-tuning's target performance, while significantly mitigating the degradation of general pre-training capabilities in comparison to other baselines. 
    %\textbf{Bolded} and \underline{underlined} values indicate the \textbf{best} and \underline{second-best} results within a given model's metrics, respectively.
    }
    \vspace{0.1 cm}
    \small
    \begin{tabular}{ll|ccc|c|c}
         \toprule
        & \textbf{} &  \multicolumn{3}{c}{\textbf{General Capability Acc.}} & \textbf{Target Acc.} & \\ 
        \cmidrule(lr){3-5}
        \cmidrule(lr){6-6}
         &\multicolumn{1}{c|}{{\textbf{Method}}} &  \textbf{Commonsense} & \textbf{MMLU} & \textbf{MBPP} & \textbf{GSM8K} & \textbf{Average}\\
        \midrule
        \multirow{6}{*}{\rotatebox{90}{\textbf{Gemma 2 2B}}} 
         & Pre-trained & \underline{57.23} \relax {\tiny\textcolor{blue}{(+0.00)}} & \underline{49.59} \relax {\tiny\textcolor{blue}{(+0.00)}} & \textbf{28.40} \relax {\tiny\textcolor{blue}{(+0.00)}} & 24.49 \relax {\tiny\textcolor{red}{(-38.89)}}& 40.79 \\
         & Standard Fine-tuning & 55.07 {\tiny\relax\textcolor{red}{(-2.16})} & 45.59 \relax{\tiny\textcolor{red}{(-4.00)}}& 16.80 \relax{\tiny\textcolor{red}{(-11.60)}}& \textbf{63.38} \relax{\tiny\textcolor{blue}{(+0.00)}}& 46.31 \\
         % & Linear Probing & 56.18 \relax\textcolor{red}{(-1.05)} &  48.63 \relax\textcolor{red}{(-0.96)} & 24.40 \relax\textcolor{red}{(-4.00)}& 25.10 \relax\textcolor{red}{(-38.28)}& 39.17 \\
         & WiSE-FT ($\alpha=0.5$) & 57.28 \relax{\tiny\textcolor{blue}{(+0.05)}} & \textbf{50.13} \relax{\tiny\textcolor{blue}{(+0.54)}} & 25.60 \relax{\tiny\textcolor{red}{(-2.80)}} & 53.30 \relax{\tiny\textcolor{red}{(-10.08)}}&  47.60\\
         & LoRA ($r=64$) &  55.67 \relax{\tiny\textcolor{red}{(-1.56)}} & 44.28 \relax{\tiny\textcolor{red}{(-5.31)}}& 25.80 \relax{\tiny\textcolor{red}{(-2.60)}}& 60.43 \relax{\tiny\textcolor{red}{(-2.95)}}& 47.05\\
         & $\ell_2$-Regularization & 57.01 \relax{\tiny\textcolor{red}{(-0.22)}} & 48.43 \relax{\tiny\textcolor{red}{(-1.16)}}& 24.80 \relax{\tiny\textcolor{red}{(-3.60)}}& \underline{62.85} \relax{\tiny\textcolor{red}{(-0.53)}} &  \underline{49.19} \\
         & \methodbold (Ours) & \textbf{57.59} \relax{\tiny\textcolor{blue}{(+0.36)}} & 49.31 \relax{\tiny\textcolor{red}{(-0.28)}}& \underline{26.80} \relax{\tiny\textcolor{red}{(-1.60)}}& 62.55 \relax{\tiny\textcolor{red}{(-0.83)}}& \textbf{49.98} \\
        \midrule
        \multirow{6}{*}{\rotatebox{90}{\textbf{Llama 3.2 3B}}} 
         & Pre-trained & \underline{54.48} \relax{\tiny\textcolor{blue}{(+0.00)}} & \textbf{54.34} \relax{\tiny\textcolor{blue}{(+0.00)}}& \textbf{38.00} \relax{\tiny\textcolor{blue}{(+0.00)}}& 26.01 \relax{\tiny\textcolor{red}{(-40.94)}} & 44.28 \\
         & Standard Fine-tuning & 50.68 \relax{\tiny\textcolor{red}{(-3.80)}}& 45.29 \relax{\tiny\textcolor{red}{(-9.05)}}& 17.80 \relax{\tiny\textcolor{red}{(-20.20)}} & \textbf{66.95} \relax{\tiny\textcolor{blue}{(+0.00)}}& 46.10 \\
         % & Linear Probing & 53.29 \relax\textcolor{red}{(-1.19)}& \underline{53.75} \relax\textcolor{red}{(-0.59)}& 31.80 \relax\textcolor{red}{(-6.20)}& 27.52 \relax\textcolor{red}{(-39.43)}&  41.69 \\
         & WiSE-FT ($\alpha=0.5$) & \textbf{54.54} \relax{\tiny\textcolor{blue}{(+0.04)}}& 53.33 \relax{\tiny\textcolor{red}{(-1.01)}}& 34.60 \relax{\tiny\textcolor{red}{(-3.40)}}& 57.01 \relax{\tiny\textcolor{red}{(-9.94)}}&  50.75\\
         & LoRA ($r=64$) & 53.10 \relax{\tiny\textcolor{red}{(-1.38)}}& 50.95 \relax{\tiny\textcolor{red}{(-3.39)}}& 34.00 \relax{\tiny\textcolor{red}{(-4.00)}}& 63.84 \relax{\tiny\textcolor{red}{(-3.15)}}&  51.66 \\
         & $\ell_2$-Regularization & 53.60 \relax{\tiny\textcolor{red}{(-0.88)}}& 51.28 \relax{\tiny\textcolor{red}{(-3.06)}}& 33.60 \relax{\tiny\textcolor{red}{(-4.40)}}& \underline{66.87} \relax{\tiny\textcolor{red}{(-0.08)}}& \underline{52.30}\\
         & \methodbold (Ours) & 54.30 \relax{\tiny\textcolor{red}{(-0.18)}}& 51.86 \relax{\tiny\textcolor{red}{(-2.48)}}& \underline{36.00} \relax{\tiny\textcolor{red}{(-2.00)}}& 65.58 \relax{\tiny\textcolor{red}{(-1.37)}}& \textbf{52.87} \\
         \bottomrule
    \end{tabular}
    \label{tab:main_lang_table}
\end{table*}

%with in depth discussions to verify \method's performance on language and visions models. 

% in conjunction with multiple architectures in language model and vision settings, highlighting \method's ability to effectively mitigate catastrophic forgetting in fine-tuning, agnostic to model architecture and data modality. 

\subsection{Comparing \methodbold and Related Baselines}
% \draft{\paragraph{Language models trained with \methodbold significantly mitigate catastrophic forgetting in comparison to standard fine-tuning.}} 
% Make it shorter, preferrably one line.
For our \textbf{vision} experiments on ResNets, Table \ref{table:main_vision_table} lists the accuracies %(as described in the evaluation metrics)
of all the baselines and \method. %Our findings are consistent among the vision models, 
We observe similar trends for both ResNet models, so we discuss the results for the larger ResNet-50 model here.  
The pre-trained ResNet-50 model achieves a top-1 accuracy of 79.02\% on ImageNet-1K's validation set. Standard fine-tuning experiences a significant 42.11\% drop in IN-1K accuracy, while achieving an average fine-tuning accuracy of 91.78\% across the target datasets. In contrast, \method suffers only a 2.93\% drop in IN-1K accuracy and exhibits a reasonable 86.25\% average accuracy on target fine-tuning datasets, demonstrating a significant improvement over standard fine-tuning. Overall, \method's \textbf{average} on IN-1K and target domain accuracy 
%(average accuracy over all 6 tasks) 
is \textbf{16.83\% higher} than standard fine-tuning. 
%{We observe thematically similar results for ResNet-18.}% and ViT-B/16.} 
%Notably, these findings extend quite well to the ResNet18 model, albeit with different numerical results depending on the specific use case. 

\begin{table}[htb!]
\centering
\tiny
 \caption{\textbf{WiSE-FT with \methodbold vs. WiSE-FT in vision.}
 %Ablation on WiSE-FT with \method on an average over six image classification tasks (task-specific fine-tuning accuracy) and ImageNet-1k (pre-trained base capability). We use $+$ sign to indicate a combination of a fine-tuning method with \method. \textbf{Bolded} values indicate the \textbf{best} results within a given model. 
 \enquote{WiSE-FT+} denotes WiSE-FT with \method in the table. Comparison here is in the same setting as \Cref{table:main_vision_table}. Note that \textbf{{WiSE-FT+} is significantly better than WiSE-FT}. 
 }
 \vspace{0.9 cm}
\small
\begin{tabular}{ll|c|c|c}
\toprule
 & \textbf{Method} & \textbf{IN-1K} & \textbf{Target} & \textbf{Average} \\
\midrule
\multirow{2}{*}{\textbf{ResNet-18}} 
& WiSE-FT & 54.15 & 80.23 & 67.19 \\
& WiSE-FT+ & 68.71 & 74.03 & \textbf{71.37} \\ %$\text{WiSE-FT}^\dagger$
\midrule
\multirow{2}{*}{\textbf{ResNet-50}} 

& WiSE-FT      & 61.65  & 81.38 & 71.52 \\
& WiSE-FT+ & 78.29 & 73.80 & \textbf{76.04} \\
\bottomrule
\end{tabular}
\label{table:complementary_vision_table}
\end{table}

%Next, we compare our method against \textit{two catastrophic forgetting baselines} (in the data-oblivious setting) and standard linear probing. Our results in Table \ref{table:main_vision_table} show that ResNet models trained with \method outperform all other baselines. 
Going beyond standard fine-tuning, our results in Table \ref{table:main_vision_table} show that \method comprehensively outperforms other baselines. Interestingly, despite its simplicity, linear probing is the second-best method. Overall, \method outperforms other baselines, when \textbf{averaging} IN-1K and target fine-tuning accuracies, by \textbf{a 3.44\% advantage} over the closest competitor, linear probing.
Although linear probing completely prevents forgetting, it learns significantly less during fine-tuning compared to \method. %\textcolor{red}{Interestingly, when we fine-tune a reasonably large ViT-B/16 model on the large-scale Food-101 dataset (101K samples), it outperforms all baselines, including the distillation-based LwF method, which was the second-best performer.}
% Compared to linear probing, \method achieves \textbf{3.44\% higher average} accuracy.
{The accuracies on individual fine-tuning datasets and corresponding accuracies for IN-1K} 
% each of the six datasets (for the results in \Cref{table:main_vision_table}) 
can be found in \cref{add-vis-results}. %Tables \ref{table:app_vision_results1} and \ref{table:app_vision_results3}.} %This should maybe reference a subsection in the vision appendix.

%{\color{red}\Cref{table:vision_clip_table} lists the accuracies for our experiment on CLIP VIT-B/32. The trend is similar to \Cref{table:main_vision_table}. Standard fine-tuning results in a colossal 59.68\% loss in the pre-training accuracy, though it achieves 88.80\% on the fine-tuning data. \method leads to only a 0.90\% drop in pre-training accuracy while surprisingly attaining a higher fine-tuning accuracy of 90.53\%. This leads to \method achieving the best overall average accuracy (75.64\%), \textbf{outperforming the second-best method}, linear probing, by \textbf{2.39}\%.}

{Additionally, in \Cref{app:dist}, we compare \method with a distillation-based approach for mitigating forgetting called \enquote{learning without forgetting} (LwF) \cite{li2016learning} on ViT-B/16, which is a larger model than ResNets and on the Food101 \cite{bossard14} dataset, which is relatively larger than each of the six datasets here. %While \method has better performance on the forgetting front, LwF has the edge on fine-tuning performance; \method outperform LwF in terms of average performance. Besides, note that LwF has additional tunable parameters.
In summary, \method outperforms LwF in terms of average accuracy; while LwF has the edge on fine-tuning performance, \method does better on the forgetting front. 
}

Our \textbf{language model} results are in \Cref{tab:main_lang_table}. Results for Gemma 2 2B %in \cref{tab:main_lang_table} 
show that \method helps preserve (and even somewhat enhance) the general capabilities of the pre-trained model. {Specifically, compared to standard fine-tuning,   \method improves general capability accuracy by 2.52\% in commonsense reasoning, 3.73\% in MMLU, and 10.00\% in MBPP, with a minor degradation of 0.83\% in GSM8K. We see a similar trend in our Llama 3.2 3B experiments.} Furthermore, while alternative baselines show specific strengths (such as WiSE-FT's general capability performance and $\ell_2$-regularization's target fine-tuning performance), \textbf{\methodbold outperforms all baselines, on average,} for both models, striking the best balance between preserving general capabilities and achieving good target fine-tuning performance. Additional details on commonsense reasoning results are in \cref{app:extended-commonsense-results} and an ablation for our choice of sample weighting in LLMs is in \cref{app:token-wise-ablation}.

% In summary, \methoditalic{} strikes a good balance between learning a new task and retaining knowledge from pre-training.

In summary, \textit{\methoditalic strikes a good balance between learning a new task and retaining knowledge from pre-training.}

% \draft{\paragraph{Vision models trained with \methodbold performs significantly better than standard fine-tuning and performs favorably against baselines.}}
% Make it shorter, preferrably one line.

\subsection{Combining \methodbold with Baselines}
% \textbf{\methodbold augments model averaging and improves upon the standalone averaging baseline.}
To complement our results in \cref{table:main_vision_table,tab:main_lang_table}, we investigate the performance of baselines when \textbf{combined with \methodbold}. 
In the vision setting, we consider uniform model averaging with WiSE-FT (with $\alpha=0.5$) and report its performance with and without \method in \cref{table:complementary_vision_table}. 
% , which corresponds to uniform model averaging. 
% As shown in Table \ref{table:complementary_vision_table}, we conducted model averaging both with and without our method. 
Interestingly, averaging the pre-trained IN-1K model and the fine-tuned model obtained with \method \textbf{improves} over standard WiSE-FT (i.e., averaging the pre-trained IN-1K model and the standard fine-tuned model) by \textbf{4.18}\% and \textbf{4.52}\% for ResNet-18 and ResNet-50, respectively, in average performance. 
% we observe that models first trained with \method and then averaged exhibit surprising improvements, highlighting the effectiveness of our approach in improving generalization in fine-tuned models when combined with model averaging. 
% Ablation of WiSE-FT in \draft{language modeling} can be found in \cref{app:extended-wa-results}.

Further, as seen in \cref{tab:lang-ours-baseline}, \method %demonstrates strong complementarity with 
\textbf{boosts the performance} of other baselines in language modeling. When combined with $\ell_2$-regularization, we observe improvements in general capability between $0.5\%$ and $1.80\%$, with only a $0.83\%$ reduction in GSM8K performance. Furthermore, the integration of \method with LoRA yields even stronger results, enhancing general capability performance by $1.07\%$ to $3.40\%$, while simultaneously improving GSM8K performance by $1.06\%$. Further details and discussion combining \method with $\ell_2$-regularization and LoRA are in \cref{app:extended-commonsense-results}.

% \textbf{\methodbold is complementary to other baselines in language models.} 
\section{Theoretical Analysis}
\label{sec: theory}
%In this section, we present the convergence of our scheme for a simplified scenario where the objective function is linear. We argue that our proposed strategy is provably better than fine-tuning for either of the tasks, suggesting that it is necessary to approach the sequential training process with a strategy that jointly benefits the performance of both tasks.
%\rd{Motivate linear case with NTK etc}
%\rd{Mention that the results extend to general variances also.}
%\rd{Add footnote explaining why we aren't considering the finite sample case.}
%\rd{Add intuition about implications on weight space with our sample selection algo.}
Here we consider linear pre-training and fine-tuning  tasks\footnote{Our insights carry over to neural networks following the dynamics of linear models under gradient descent \cite{lee2019wide}. %It is also worth mentioning that the analysis for \method is non-trivial even for linear models.
} and theoretically analyze the effect of fine-tuning with our proposed method \method (Alg. \ref{alg:method}). %and compare it with vanilla fine-tuning. 
%We analyze the impact of fine-tuning the model on the second task with our proposed method \method \space and compare it with vanilla fine-tuning. 
Specifically, we compare the non-asymptotic trajectories of \method and vanilla fine-tuning. A key insight of our analysis is that \method stalls training in a certain direction, %\draft{(namely, the top eigenvector of the covariance matrix of the fine-tuning data weighted by \method)} 
impeding overfitting to the fine-tuning task (see \Cref{rmk-2}). We also demonstrate that \method goes beyond the simple idea of model averaging (see \Cref{rmk-3}).

We begin by describing the {problem setting.} 

\textbf{Pre-training task:} The label $\text{y} \in \mathbb{R}$ for a $d$-dimensional data point $\mathbf{x} \sim \mathcal{P}$ is given by $\text{y} = \langle \bm{\theta}_{*}, \mathbf{x} \rangle$, where $\bm{\theta}_{*} \in \mathbb{R}^{d}$ is the ground-truth model. 
Let $\mathcal{D}$ denote the joint distribution of $(\mathbf{x}, \text{y})$, where $\mathbf{x} \sim \mathcal{P}$. Let $\bm{\Sigma} = \mathbb{E}_{\mathbf{x} \sim \mathcal{P}}\big[\mathbf{x} \mathbf{x}^\top\big]$ be the data covariance matrix. 
Without loss of generality, let $\bm{\Sigma} \succeq \mathbf{I}_d$. 
%\mathbf{I}_d \preceq \bm{\Sigma} \preceq \lambda \mathbf{I}_d$. {\color{red} (Is $\lambda$ needed?)}

\textbf{Fine-tuning task:} The label $\widetilde{\text{y}} \in \mathbb{R}$ for a $d$-dimensional data point $\widetilde{\mathbf{x}} \sim \widetilde{\mathcal{P}}$ is given by $\widetilde{\text{y}} = \big\langle \widetilde{\bm{\theta}}_{*}, \widetilde{\mathbf{x}} \big\rangle$, where $\widetilde{\bm{\theta}}_{*} \in \mathbb{R}^{d}$ is the ground-truth model. Let $\widetilde{\mathcal{D}}$ denote the joint distribution of $(\widetilde{\mathbf{x}}, \widetilde{\text{y}})$, where $\widetilde{\mathbf{x}} \sim \widetilde{\mathcal{P}}$. Also, let $$\mathbf{e} := \bm{\theta}_{\ast} - \widetilde{\bm{\theta}}_{*}, \overline{\mathbf{e}} := {\mathbf{e}}/{\|\mathbf{e}\|_2},$$
and $\overline{\mathbf{{e}}}_{\perp}$ be a unit vector orthogonal to $\overline{\mathbf{e}}$. We consider the case of $\widetilde{\mathcal{P}} = \mathcal{N}(\vec{\bm{0}}_d, \widetilde{\bm{\Sigma}})$, where 
\begin{flalign}
    \label{eq:1-jan15}
    \widetilde{\bm{\Sigma}} = \mathbf{{I}}_d + \rho  \big(\overline{\mathbf{{e}}} \overline{\mathbf{{e}}}_{\perp}^\top + \overline{\mathbf{{e}}}_{\perp} \overline{\mathbf{{e}}}^\top\big),
\end{flalign}
where $\rho \in [0,1)$ is a constant. Note that $\widetilde{\bm{\Sigma}}$ is the data covariance matrix here.

\begin{table}[htb!]
    \centering
    \caption{\textbf{$\mathbf{\ell_2}$-Reg./LoRA + \methodbold vs. $\mathbf{\ell_2}$-Reg./LoRA in language.} \enquote{$\ell_2$+} and \enquote{LoRA+} denote $\ell_2$-Reg. with \method and LoRA with \method, respectively. 
    %Ablation of Gemma 2 2B trained on MetaMathQA, using $r=64$ for LoRA.
    The results below are for Gemma 2 2B in the same setup as \Cref{tab:main_lang_table}.  
    We let A1, A2, A3, and B1 represent {Commonsense}, {MMLU}, {MBPP}, and {GSM8K}, respectively. 
    % We use $\dagger$ to indicate a combination of a fine-tuning method with \method. 
    %We use $+$ sign to indicate a combination of a fine-tuning method with \method.
    %\textbf{Bold} values indicate the best results within each method. 
    %Just like \Cref{table:complementary_vision_table}, 
    We see that \textbf{{$\mathbf{\ell_2}$+} and {LoRA+} are better than {$\mathbf{\ell_2}$} and {LoRA}}, respectively.
    }
    \vspace{0.1 cm}
    \small
    \begin{tabular}{c|ccc|c|c}
        \toprule
        \textbf{Method} & \textbf{A1} & \textbf{A2} & \textbf{A3} & \textbf{B1} & \textbf{Avg.}\\
        \midrule
        % $L_1$ & 56.95 & 48.57 & 18.29 & \textbf{62.85} & 47.59\\
        % $L_1^\dagger$ & \textbf{57.51} & \textbf{49.59} & \textbf{18.90} & 62.62 & \textbf{48.08} \\
        % \midrule
        $\ell_2$ &  57.01 & 48.43 & 24.80 & \textbf{62.85} & 49.19 \\
        $\ell_2$+ & \textbf{57.53} & \textbf{49.38} & \textbf{26.60} & 62.02 & \textbf{49.79} \\
        \midrule
        LoRA & 55.67 & 44.28 & 25.80 & 60.43 & 47.05 \\
        $\text{LoRA}$+ & \textbf{56.74} & \textbf{47.68} & \textbf{28.80} & \textbf{61.49} &  \textbf{49.31} \\ 
        \bottomrule
    \end{tabular}
   \label{tab:lang-ours-baseline}
\end{table}

\begin{remark}[\textbf{Regarding $\widetilde{\bm{\Sigma}}$}]
    \label{rmk-just-jan30}
    We study the case of $\widetilde{\bm{\Sigma}}$ as given in Eq. (\ref{eq:1-jan15}) because it is the minimal analytically tractable case where we can show that \method goes beyond model averaging (MA) (see \Cref{rmk-3}). Specifically, if $\rho = 0$ and $\widetilde{\bm{\Sigma}} = \mathbf{{I}}_d$, then \method reduces to MA. Moreover, for an arbitrary $\widetilde{\bm{\Sigma}}$, characterizing the eigen-spectrum of the matrix dictating the trajectory of \method becomes intractable. For the analysis to be tractable, we need some relationship between $\widetilde{\bm{\Sigma}}$ and $\mathbf{e}$ (i.e., the difference between the optima of the pre-training and fine-tuning tasks).\footnote{See \Cref{sec:gen_sigma} for more details.}
    %For our case (Eq. \ref{eq:1-jan15}), the eigen-spectrum computation is tractable. {\color{red} \textbf{Why the specific rank-2 choice?}}
    %, although fairly tedious (see the discussion after Eq. (\ref{eq:13-jan18})). 
    %Finally, regarding the dependence of the rank-2 perturbation to $\mathbf{{I}}_d$ in Eq. (\ref{eq:1-jan15}) on $\overline{\mathbf{e}}$ and $\overline{\mathbf{{e}}}_{\perp}$, 
\end{remark}

For a model parameterized by $\bm{\theta} \in \mathbb{R}^d$, let 
\begin{equation*}
    \text{err}_1(\bm{\theta}) := \mathbb{E}_{\mathcal{D}} \Big[\big(\text{y} - \langle \bm{\theta}, \mathbf{x}\rangle \big)^2\Big] = \big(\bm{\theta} - \bm{\theta}_{\ast}\big)^\top \bm{\Sigma} \big(\bm{\theta} - \bm{\theta}_{\ast}\big), 
\end{equation*}
\begin{equation}
    \label{eq:4-jan17}
    \text{err}_2(\bm{\theta}) := \mathbb{E}_{\widetilde{\mathcal{D}}} \Big[\big(\widetilde{\text{y}} - \langle \bm{\theta}, \widetilde{\mathbf{x}} \rangle\big)^2\Big] = \big(\bm{\theta} - \widetilde{\bm{\theta}}_{*}\big)^\top \widetilde{\bm{\Sigma}} \big(\bm{\theta} - \widetilde{\bm{\theta}}_{*}\big)
\end{equation}
be the population errors on the pre-training and fine-tuning tasks, respectively. Also, the total error with $\bm{\theta}$ on the two tasks is denoted by $\text{err}_\text{tot}(\bm{\theta}) = \text{err}_1(\bm{\theta}) + \text{err}_2(\bm{\theta})$. 

We assume that initially, we learn $\bm{\theta}_{*}$ with the pre-training data; so $\bm{\theta}_{*}$ is our pre-trained model. Note that 
\begin{equation}
    \label{eq:4-jan18}
    \text{err}_\text{tot}(\bm{\theta}_{\ast}) = \text{err}_2(\bm{\theta}_{\ast}) = \mathbf{e}^\top \widetilde{\bm{\Sigma}} \mathbf{e} = \|\mathbf{e}\|_2^2,
\end{equation}
where the last step follows by using %the value of 
$\widetilde{\bm{\Sigma}}$ from \Cref{eq:1-jan15}.

We start fine-tuning %the model 
%on the data of Task 2 
starting from $\bm{\theta}_{*}$. %with \method. 
Specifically, we assume access to the population%\footnote{The finite sample case will only complicate the (already tedious) analysis. We leave it for future work.}  
 $(\widetilde{\mathbf{x}}, \widetilde{\text{y}}) \sim \widetilde{\mathcal{D}}$ of the fine-tuning task, but we lose access to the pre-training data. 

%\ali{
%\begin{enumerate}
    %\item Mention in the beginning that Task 1 is pre-training, Task 2 is fine-tuning.
    %\item Move the characterization of $\widetilde{\mathcal P}$ after \cref{eq:4-jan17}. Introduce the $\text{err}_1, \text{err}_2$ and $\text{err}_\text{tot}$ first. Unify $\text{err}_1, \text{err}_2$ in a single aligned block?
    %\item Motivate the choice of $\be$ using the error definition. \enquote{We need the fine-tuning task to be aligned with pre-trained model in order to prove that it is efficiently learnable starting from the pre-trained parameters. We introduce the following condition on $\widetilde{\mathcal P}$ to form an inherent connection between the tasks \dots}.
    %\item Present the covariance for Task 2, i.e., $\widetilde{\bm{\Sigma}}$ and its form as in \cref{eq:1-jan15}.
%\end{enumerate}
%}
%\rd{Analytical tractability in point 3. Some kind of alignment between the two tasks...}

\textbf{Vanilla fine-tuning (FT):} We minimize $\text{err}_2(\bm{\theta})$ (Eq. (\ref{eq:4-jan17})) with gradient descent (GD) starting from $\bm{\theta}_{\ast}$ using a constant learning rate $\overline{\eta}$. Our iterate $\overline{\bm{\theta}}_K$ at the $K^\text{th}$ iteration is given by (using the value of  $\widetilde{\bm{\Sigma}}$ from Eq. (\ref{eq:1-jan15}) and $\bm{\theta}_{\ast} - \widetilde{\bm{\theta}}_{*} = \mathbf{e}$):
%$\overline{\bm{\theta}}_K = \widetilde{\bm{\theta}}_{*}  + \big(\mathbf{I}_d - 2 \overline{\eta} \widetilde{\bm{\Sigma}}\big)^K \big(\bm{\theta}_{\ast} - \widetilde{\bm{\theta}}_{*}\big)$. Plugging in $\widetilde{\bm{\Sigma}}$ from Eq. (\ref{eq:1-jan15}) and recalling that $\bm{\theta}_{\ast} - \widetilde{\bm{\theta}}_{*} = \mathbf{e}$, we get:
\begin{equation}
    \label{eq:6-jan17}
    \overline{\bm{\theta}}_K = \widetilde{\bm{\theta}}_{*} + \Big(\mathbf{I}_d - 2 \overline{\eta} \Big(\mathbf{{I}}_d + \rho  \big(\overline{\mathbf{{e}}} \overline{\mathbf{{e}}}_{\perp}^\top + \overline{\mathbf{{e}}}_{\perp} \overline{\mathbf{{e}}}^\top\big)\Big)^K \mathbf{e}.
\end{equation}

\textbf{\methodbold:} For some temperature $\tau$, the weight of $(\widetilde{\mathbf{x}}, \widetilde{\text{y}}) \sim \widetilde{\mathcal{D}}$ is $w(\widetilde{\mathbf{x}}, \widetilde{\text{y}}) = \exp\Big(-\frac{(\widetilde{\text{y}} -  \langle \bm{\theta}_{\ast}, \widetilde{\mathbf{x}}\rangle)^2}{\tau}\Big)$. We minimize
    \begin{equation}
        \label{eq:3}
        \widehat{\text{err}}_2(\widehat{\bm{\theta}}) := \mathbb{E}_{\widetilde{\mathcal{D}}} \Big[w(\widetilde{\mathbf{x}}, \widetilde{\text{y}}) \big(\widetilde{\text{y}} - \langle \widehat{\bm{\theta}}, \widetilde{\mathbf{x}} \rangle\big)^2\Big], 
    \end{equation}
    with GD starting from $\bm{\theta}_{\ast}$ using a constant learning rate $\widehat{\eta}$. Suppose our iterate at the $K^\text{th}$ iteration is $\widehat{\bm{\theta}}_{K}$. %after starting from $\bm{\theta}_{\ast}$.
    %$\widehat{\bm{\theta}}_{K}$ after running $K$ steps of GD with some appropriate learning rate $\eta$. %We would like to upper-bound $\text{err}_\text{tot}(\widehat{\bm{\theta}}_{K})$.

\begin{theorem}[\textbf{\methodbold}]
    \label{thm-ploss}
    %Fix some $\beta \in [0,1]$. There exists a threshold $\tau$ and a learning rate $\eta$  depending on $\beta$ such that
    %\small
    %\begin{equation*}
    %\widehat{\bm{\theta}}_K = \widetilde{\bm{\theta}}_{*} + \Bigg(\frac{\widehat{\lambda}_1^K + \widehat{\lambda}_2^K \beta^2 \rho^2}{1 + \beta^2 \rho^2} \Bigg) \textbf{\textup{e}} -\beta \rho \Bigg(\frac{\widehat{\lambda}_1^K -  \widehat{\lambda}_2^K}{1 + \beta^2 \rho^2} \Bigg) \|\textbf{\textup{e}}\|_2 \overline{\textbf{\textup{e}}}_{\perp, 2},
    %\end{equation*}
    %\normalsize
    %where $\widehat{\lambda}_1 = \frac{1 + \beta \rho^2}{1 + \beta}$ and $\widehat{\lambda}_2 = \rho^2\Big(\frac{1 - \beta}{1 - \beta \rho^2}\Big)$.
    Let $\mu = \Big(\frac{\tau}{\tau + 2\|\mathbf{{e}}\|_2^2}\Big)^{1/2}$. Then:
    \begin{flalign}
    \label{eq:8-jan17}
    \widehat{\bm{\theta}}_K =  \widetilde{\bm{\theta}}_{*} + \Big(\mathbf{{I}}_d - 2 \widehat{\eta} \widetilde{\bm{\Sigma}}'\Big)^K \mathbf{{e}}, %, \text{ where } 
    \end{flalign}
    where $\widetilde{\bm{\Sigma}}' = \mu \big(\mathbf{I}_d - \mathbf{Q}\big)$ with   
    %\begin{flalign*}
    %    \widetilde{\bm{\Sigma}}' = & \mu \Big(\mathbf{{I}}_d - (1 - \mu^2)\overline{\mathbf{{e}}} \overline{\mathbf{{e}}}^\top  - \rho^2 (1 - \mu^2) \overline{\mathbf{{e}}}_{\perp} \overline{\mathbf{{e}}}_{\perp}^\top 
    %    \\
    %    & \quad \quad \quad \quad \quad \quad \quad \quad \quad + \rho \mu^2 \big(\overline{\mathbf{{e}}} \overline{\mathbf{{e}}}_{\perp}^\top + \overline{\mathbf{{e}}}_{\perp} \overline{\mathbf{{e}}}^\top\big)\Big)
    %\end{flalign*}
    %\begin{flalign}
    %\nonumber
    %& \mathbf{Q} = 
    %(1 - \mu^2)\overline{\mathbf{{e}}} \overline{\mathbf{{e}}}^\top + \rho^2 (1 - \mu^2) \overline{\mathbf{{e}}}_{\perp} \overline{\mathbf{{e}}}_{\perp}^\top 
    %\\ 
    %\label{eq:11-jan19}
    %& \quad \quad \quad \quad \quad \quad \quad \quad 
    %- \rho \mu^2 \big(\overline{\mathbf{{e}}} \overline{\mathbf{{e}}}_{\perp}^\top + \overline{\mathbf{{e}}}_{\perp} \overline{\mathbf{{e}}}^\top\big)
    %\end{flalign}
    \begin{multline}
        \label{eq:11-jan19}
        \mathbf{Q} = 
        (1 - \mu^2)\overline{\mathbf{{e}}} \overline{\mathbf{{e}}}^\top + \rho^2 (1 - \mu^2) \overline{\mathbf{{e}}}_{\perp} \overline{\mathbf{{e}}}_{\perp}^\top 
        \\ 
        \quad \quad \quad \quad \quad \quad \quad \quad 
        - \rho \mu^2 \big(\overline{\mathbf{{e}}} \overline{\mathbf{{e}}}_{\perp}^\top + \overline{\mathbf{{e}}}_{\perp} \overline{\mathbf{{e}}}^\top\big).
    \end{multline}
    %$\widetilde{\bm{\Sigma}}'$ is the covariance matrix of the weighted fine-tuning data. 
\end{theorem}
We prove Thm. \ref{thm-ploss} in \Cref{pf-thm-ploss}. The \textbf{main technical 
challenge} %in the proof 
is the evaluation of $\widetilde{\bm{\Sigma}}'$, viz., the covariance matrix of the \textit{weighted} fine-tuning data; see \Cref{lem1} for this.
%\ali{
%\begin{enumerate}
%    \item Define $\widetilde{\bm{\Sigma}}$ as the \enquote{weighted} Task 2 covariance for ease of presentation? 
%    \item Replace the closed form for $\widetilde{\bm{\Sigma}}$ in the \cref{thm-ploss} statement as $\widetilde{\bm{\Sigma}}' := \mathbb{E}_{\widetilde{\mathbf{{x}}} \sim \widetilde{\mathcal{P}}}\Big[w(\widetilde{\bx}, \widetilde{\text{y}}) \widetilde{\mathbf{{x}}} \widetilde{\mathbf{{x}}}^\top \Big]$.
%\end{enumerate}
%}

%Comparing Eqs. (\ref{eq:6-jan17}) and (\ref{eq:8-jan17}), we claim that $\overline{\eta} = \frac{1}{2}$ and $\widehat{\eta} = \frac{1}{2 \mu}$ are comparable learning rates for vanilla FT and \method, respectively. This is because with these two choices, 
Now, we are going to compare vanilla FT (\ref{eq:6-jan17}) with $\overline{\eta} = \frac{1}{2}$ and \method \space (\ref{eq:8-jan17}) with $\widehat{\eta} = \frac{1}{2 \mu}$. We believe these are comparable learning rates for vanilla FT and \method because 
the resultant matrices (Eqs. (\ref{eq:10-jan17}) and (\ref{eq:11-jan17})) dictating the convergence of both methods have exactly two non-zero eigenvalues and the corresponding eigenvectors lie in the span of $\overline{\mathbf{{e}}}$ and $\overline{\mathbf{{e}}}_{\perp}$. %(this will come up subsequently). 
%(see Eq. (\ref{eq:10-jan17}) and (\ref{eq:11-jan17})). 
Plugging in $\overline{\eta} = \frac{1}{2}$ into Eq. (\ref{eq:6-jan17}), we get:
\begin{flalign}
    \label{eq:10-jan17}
    \overline{\bm{\theta}}_K = \widetilde{\bm{\theta}}_{*} + \mathbf{P}^K \mathbf{e}, \text{ with } \mathbf{P} = - \rho  \big(\overline{\mathbf{{e}}} \overline{\mathbf{{e}}}_{\perp}^\top + \overline{\mathbf{{e}}}_{\perp} \overline{\mathbf{{e}}}^\top\big)
\end{flalign}
for vanilla FT. Plugging in $\widehat{\eta} = \frac{1}{2 \mu}$ into Eq. (\ref{eq:8-jan17}), we get:
\begin{equation}
    \label{eq:11-jan17}
    \widehat{\bm{\theta}}_K = \widetilde{\bm{\theta}}_{*} + \mathbf{Q}^K \mathbf{e}, \text{ with } \mathbf{Q} \text{ given by Eq. (\ref{eq:11-jan19})}
\end{equation}
%\begin{flalign}
%\label{eq:11-jan17}
%& \widehat{\bm{\theta}}_K = \widetilde{\bm{\theta}}_{*} + \mathbf{Q}^K \mathbf{e}, \text{ with } %\mathbf{Q} \text{ given by Eq. (\ref{eq:11-jan19})}
%\\ 
%\nonumber
%& \mathbf{Q} = 
%(1 - \mu^2)\overline{\mathbf{{e}}} \overline{\mathbf{{e}}}^\top + \rho^2 (1 - \mu^2) \overline{\mathbf{{e}}}_{\perp} \overline{\mathbf{{e}}}_{\perp}^\top 
%\\ 
%\label{eq:11-jan19}
%& \quad \quad \quad \quad \quad \quad \quad \quad - \rho \mu^2 \big(\overline{\mathbf{{e}}} \overline{\mathbf{{e}}}_{\perp}^\top + \overline{\mathbf{{e}}}_{\perp} \overline{\mathbf{{e}}}^\top\big)
%\end{flalign}
for \method. The non-zero eigenvalues of $\mathbf{P}$ are $\mp \rho$ and the corresponding eigenvectors are $\frac{1}{\sqrt{2}}\big(\overline{\mathbf{{e}}} \pm \overline{\mathbf{{e}}}_{\perp}\big)$. Using this in (\ref{eq:10-jan17}) and simplifying, we get for vanilla FT:
\small
\begin{equation}
    \overline{\bm{\theta}}_K = \widetilde{\bm{\theta}}_{*} + {\rho^K} \Big(\mathds{1}\big(K \text{ is even}\big) \mathbf{e} - \mathds{1}\big(K \text{ is odd}\big) \|\mathbf{e}\|_2 \overline{\mathbf{{e}}}_{\perp}\Big).
\end{equation}
\normalsize
%for vanilla FT. 
\begin{remark}[\textbf{Vanilla FT}]
    \label{rmk-1}
    Since $\rho < 1$, $\overline{\bm{\theta}}_K$  converges to $\widetilde{\bm{\theta}}_{*}$ rapidly, %which we cannot slow down. 
    and we cannot impede this convergence. %to $\widetilde{\bm{\theta}}_{*}$.
\end{remark}
Note that (we use $\bm{\Sigma} \succeq \mathbf{I}_d$ below):
\begin{equation}
    \label{eq:13-jan18}
    \text{err}_\text{tot}(\widetilde{\bm{\theta}}_{*}) = \text{err}_1(\widetilde{\bm{\theta}}_{*}) = \mathbf{e}^\top {\bm{\Sigma}} \mathbf{e} \geq  \|\mathbf{e}\|_2^2.
\end{equation}
%where the last step follows because $\bm{\Sigma} \succeq \textbf{I}_d$.

On the other hand, the non-zero eigenvalues and corresponding eigenvectors of $\mathbf{Q}$ are not as straightforward to compute. 
We do this computation in \Cref{lem3} with the re-parameterization of $\mu = \sqrt{\frac{\beta (1-\rho^2)}{(1+\beta)(1-\beta \rho^2)}}$ for some $\beta \in (0,1]$.\footnote{The corresponding temperature is $\tau = \frac{2 \beta (1-\rho^2) \|\mathbf{e}\|_2^2}{(1 - \beta^2 \rho^2)}$.} 
%To that end, for some $\beta \in (0,1]$ let $\mu = \sqrt{\frac{\beta (1-\rho^2)}{(1+\beta)(1-\beta \rho^2)}}$.\footnote{The corresponding temperature is $\tau = \frac{2 \beta (1-\rho^2) \|\mathbf{e}\|_2^2}{(1 - \beta^2 \rho^2)}$.} 
%Then, as per \Cref{lem3}, the eigen-decomposition of $\mathbf{Q}$ is $\widehat{\lambda}_1 \widehat{\mathbf{v}}_1 \widehat{\mathbf{v}}_1^\top + \widehat{\lambda}_2 \widehat{\mathbf{v}}_2 \widehat{\mathbf{v}}_2^\top$, where 
%$$\widehat{\lambda}_1 = \frac{1 + \beta \rho^2}{1 + \beta}; \widehat{\mathbf{{v}}}_1 = \frac{\overline{\mathbf{{e}}}}{\sqrt{1+\beta^2 \rho^2}} - \frac{(\beta \rho) \overline{\mathbf{{e}}}_{\perp} }{\sqrt{1+\beta^2 \rho^2}} \text{ and }
%$$
%$$
%\widehat{\lambda}_2 = \rho^2\Bigg(\frac{1 - \beta}{1 - \beta \rho^2}\Bigg); \widehat{\mathbf{{v}}}_2 = -\frac{(\beta \rho) \overline{\mathbf{{e}}}}{\sqrt{1+\beta^2 \rho^2}} - \frac{\overline{\mathbf{{e}}}_{\perp}}{\sqrt{1+\beta^2 \rho^2}}.
%$$
Using this in Eq. (\ref{eq:11-jan17}) and simplifying, we get for \method:
%\small
\begin{equation}
\label{eq:14-jan17}
\widehat{\bm{\theta}}_K = \widetilde{\bm{\theta}}_{*} + \Bigg(\frac{\widehat{\lambda}_1^K + \widehat{\lambda}_2^K \beta^2 \rho^2}{1 + \beta^2 \rho^2} \Bigg) \mathbf{{e}} -\beta \rho \Bigg(\frac{\widehat{\lambda}_1^K -  \widehat{\lambda}_2^K}{1 + \beta^2 \rho^2} \Bigg) \|\mathbf{{e}}\| \overline{\mathbf{{e}}}_{\perp},
\end{equation}
%\normalsize
where $\widehat{\lambda}_1 = \frac{1 + \beta \rho^2}{1 + \beta}$ { and } $\widehat{\lambda}_2 = \rho^2\Big(\frac{1 - \beta}{1 - \beta \rho^2}\Big)$.

\begin{remark}[\methodbold's \textbf{trajectory}]
\label{rmk-2}
{Note that we can control $\widehat{\lambda}_1$ by varying $\beta$. Specifically, we can make $\widehat{\lambda}_1$ arbitrarily close to $1$ by choosing a small enough  $\beta$}. On the other hand, $\frac{1 - \beta}{1 - \beta \rho^2} < \frac{1 + \beta \rho^2}{1 + \beta} = \widehat{\lambda}_1$ and so, $\widehat{\lambda}_2 < \rho^2 \widehat{\lambda}_1$. %Hence, beyond a certain number of iterations $K$,  \Cref{eq:14-jan17} effectively becomes (after recalling $\textbf{\textup{e}} = \bm{\theta}_{\ast} - \widetilde{\bm{\theta}}_{*}$):
%\small
%\begin{equation}
%    \widehat{\bm{\theta}}_K \approx \widetilde{\bm{\theta}}_{*} + \gamma(K)  \big(\bm{\theta}_{\ast} - \widetilde{\bm{\theta}}_{*}\big) \\ - \beta \rho \gamma(K) \big\|\bm{\theta}_{\ast} - \widetilde{\bm{\theta}}_{*}\big\|_2 \overline{\textbf{\textup{e}}}_{\perp, 2},
%\end{equation}
%\normalsize
Hence, beyond a certain number of iterations $K$, Eq. (\ref{eq:14-jan17}) %effectively 
becomes:
\begin{equation}
    \label{eq:15-jan18}
    \widehat{\bm{\theta}}_K \approx {\bm{\theta}}_K := \widetilde{\bm{\theta}}_{*} + \gamma(K, \beta) \Big(\mathbf{{e}} - \beta \rho \|\mathbf{{e}}\|_2 \overline{\mathbf{{e}}}_{\perp}\Big),
\end{equation}
with $\gamma(K, \beta) := \Big(\frac{\widehat{\lambda}_1^K}{1 + \beta^2 \rho^2}\Big)$. Because we can control $\widehat{\lambda}_1$ by varying $\beta$, we can control $\gamma(K, \beta)$. Thus, \textbf{we can stall convergence along $\big(\mathbf{{e}} - \beta \rho \|\mathbf{{e}}\|_2 \overline{\mathbf{{e}}}_{\perp}\big)$,\footnote{This direction is the top eigenvector of $\mathbf{Q}$. Since $\widetilde{\bm{\Sigma}}' = \mu \big(\mathbf{I}_d - \mathbf{Q}\big)$, this is also the eigenvector of $\widetilde{\bm{\Sigma}}'$ with the smallest eigenvalue.} 
impeding the convergence of $\widehat{\bm{\theta}}_K$ to $\widetilde{\bm{\theta}}_{*}$}. 
\end{remark} 

\begin{remark}[\textbf{\methodbold \space goes beyond model averaging}]
    \label{rmk-3}
    If we perform model averaging between   
    %the initial pre-trained model $\bm{\theta}_{*}$ and the final fine-tuned model $\widetilde{\bm{\theta}}_{*}$ obtained after vanilla FT with parameter $\omega$, then our 
    $\bm{\theta}_{*}$ and $\widetilde{\bm{\theta}}_{*}$ with parameter $\omega \in [0,1]$, then our averaged model is:
    \begin{equation}
        \label{eq:16-jan18}
        \bm{\theta}_\textup{avg}(\omega) = \omega \bm{\theta}_{*} + (1-\omega) \widetilde{\bm{\theta}}_{*} = \widetilde{\bm{\theta}}_{*} + \omega \mathbf{{e}}. 
    \end{equation}
    Comparing the above with Eq. (\ref{eq:15-jan18}), we see that \method \space goes beyond model averaging because of the component along $\overline{\mathbf{{e}}}_{\perp}$. But we can make ${\bm{\theta}}_K$ (Eq. (\ref{eq:15-jan18})) $\to$ $\bm{\theta}_\textup{avg}(\omega)$ by choosing $\beta \to 0$ and $K$ such that $\gamma(K, \beta) \to \omega$. So, we expect \methodbold \space \textbf{to be at least as powerful as model averaging}.
\end{remark}
As per \Cref{lem-4}, the minimum total error on both tasks with optimally tuned model averaging is given by:
%\small
\begin{equation}
    \label{eq:17-jan18}
    \min_{\omega \in [0,1]} \textup{err}_\textup{tot}\big(\bm{\theta}_\textup{avg}(\omega)\big) =  \Bigg(\frac{\overline{\mathbf{{e}}}^\top \bm{\Sigma} \overline{\mathbf{{e}}}}{\overline{\mathbf{{e}}}^\top \bm{\Sigma} \overline{\mathbf{{e}}} + 1}\Bigg) \|{\mathbf{{e}}}\|_2^2 < \|{\mathbf{{e}}}\|_2^2,
\end{equation}
%\normalsize
where recall that $\bm{\Sigma}$ is the covariance matrix of the pre-training  data. On the other hand, using Eqs. (\ref{eq:4-jan18}) and (\ref{eq:13-jan18})
\begin{equation}
    \label{eq:18-jan18}
    \textup{min}\Big(\textup{err}_\textup{tot}(\bm{\theta}_{\ast}),  \textup{err}_\textup{tot}(\widetilde{\bm{\theta}}_{*})\Big) = \|{\mathbf{{e}}}\|_2^2.
\end{equation}
%\ali{
%\begin{enumerate}
%    \item Simplify the eigen-decomposition of $\bQ$, sweep the details under the carpet (between \cref{eq:13-jan18} and \cref{rmk-3}).
%    \item Present a simplified view, but make sure we can keep the \enquote{stalling the learning along some directions} argument.
%    \item \cref{rmk-3} is important to claim that \enquote{We are beyond model averaging}.
%\end{enumerate}
%}
\begin{remark}[\textbf{Error comparison}]
    \label{rmk-4}
    By comparing Eqs. (\ref{eq:17-jan18}) and (\ref{eq:18-jan18}), we see that optimally tuned model averaging attains a smaller total error than both $\bm{\theta}_{\ast}$ (i.e., the pre-trained model) and $\widetilde{\bm{\theta}}_{*}$ to which vanilla FT converges rapidly (\Cref{rmk-1}). More importantly, following our discussion in \Cref{rmk-3}, we conclude that \textbf{optimally tuned \methodbold's total error is 
    %\space can attain a total error 
    at least as good as the one in Eq. (\ref{eq:17-jan18}}).
\end{remark}

%\subsection{Main result: identity covariance}

%\subsection{Corollary: general positive-definite covariance matrix}

%\paragraph{Intuition: sample weighting as model averaging}
%Surprisingly, our convergence analysis for the linear model and the quadratic loss point out to an interesting relationship between sample weighting and model averaging. Note that this holds particularly for a linear model when the data is distributed a standard normal vector.   

\section{Conclusion}
In this paper, we studied the problem of catastrophic forgetting in pre-trained models during fine-tuning when we do not have access to the pre-training data.
%in a \emph{data-oblivious} setting. 
To mitigate this issue, we proposed \method, a method which upweights easy samples based on the pre-trained loss values. Empirically, we showed that \method, on average, outperforms relevant baselines and is also complementary to these baselines in both vision and language settings. 
%several data-oblivious baselines in both vision and language modeling, striking a balance between preserving pre-trained performance and learning a new task. 
We also theoretically analyzed \method for linear models.

\paragraph{Discussion and limitations.} We would like to conclude with an overview of our work's limitations and potential future directions. In lay terms, mitigating forgetting of pre-training capabilities comes at the cost of relatively lower fine-tuning performance. \method maintains this balance by sacrificing performance on \emph{hard samples from the fine-tuning data}. \cref{tab:hard_samples_flow} indicates that \method has lower accuracy on samples with high pre-training loss (\enquote{hard samples}) compared to standard FT. Our method selectively down-weighs samples with high pre-training losses for preserving pre-training performance. An interesting future direction is improving performance on such samples while maintaining or improving overall performance. On the theoretical side, we hope to extend our analysis to generalized linear models (GLMs) and even non-linear models.

\begin{table}[htb!]
\centering
\caption{{
\textbf{Comparison of \method{} and Standard-FT on hard samples across three vision datasets.}\
We evaluate performance on the top 10\% hardest samples (those with the highest pre-trained losses). Indeed, the samples with high pre-training losses have lower accuracy when using \method compared to standard fine-tuning (FT). This is an unsurprising outcome of our approach; we sacrifice performance on hard examples of the fine-tuning data to maintain performance on the pre-training data.
}}
\vspace{0.1 cm}
\small
\resizebox{0.96\columnwidth}{!}{
\begin{tabular}{lccc}
\toprule
\textbf{Dataset}           & \textbf{\# of Hard Samples} & \textbf{Standard FT} & \textbf{FLOW} \\
\midrule
CIFAR-10       & 1000 & 86.60 & 30.70 \\
CIFAR-100      & 1000 & 56.40 & 21.30 \\
Stanford Cars  & 805  & 71.30 & 13.18 \\
\bottomrule
\end{tabular}
}
\label{tab:hard_samples_flow}
\end{table}

\section*{Acknowledgments}
{This research was supported by NSF EnCORE Tripods (2217069) and NSF AI Institute for the Foundations of Machine Learning (2019844). Research of Ali Kavis is funded in part by the Swiss National Science Foundation (SNSF) under grant number P500PT\_217942. The authors are grateful to anonymous reviewers for their feedback on improving this paper.}

\section*{Impact Statement}

This paper presents work whose goal is to advance the field of machine learning. There are potential societal consequences 
of our work, none of which we feel must be specifically highlighted here.

%\bibliography{example_paper}
\bibliography{refs}
\bibliographystyle{icml2025}

%%%%%%%%%%%%%%%%%%%%%%%%%%%%%%%%%%%%%%%%%%%%%%%%%%%%%%%%%%%%%%%%%%%%%%%%%%%%%%%
%%%%%%%%%%%%%%%%%%%%%%%%%%%%%%%%%%%%%%%%%%%%%%%%%%%%%%%%%%%%%%%%%%%%%%%%%%%%%%%
% APPENDIX
%%%%%%%%%%%%%%%%%%%%%%%%%%%%%%%%%%%%%%%%%%%%%%%%%%%%%%%%%%%%%%%%%%%%%%%%%%%%%%%
%%%%%%%%%%%%%%%%%%%%%%%%%%%%%%%%%%%%%%%%%%%%%%%%%%%%%%%%%%%%%%%%%%%%%%%%%%%%%%%
\newpage
\appendix
\onecolumn
%\section{You \emph{can} have an appendix here.}

%You can have as much text here as you want. The main body must be at most $8$ pages long. For the final version, one more page can be added. If you want, you can use an appendix like this one.  

%The $\mathtt{\backslash onecolumn}$ command above can be kept in place if you prefer a one-column appendix, or can be removed if you prefer a two-column appendix.  Apart from this possible change, the style (font size, spacing, margins, page numbering, etc.) should be kept the same as the main body.

\begin{center}
    {\LARGE \bf Appendix}
\end{center}

\section*{Table of Contents}
\vspace{1mm}
\begin{itemize}
    \item \Cref{app:related-work}: Extended Related Work %\nameref{app:related-work}
    \item \Cref{alg:multi-head}: Our Algorithm in the Presence of Task-Specific Model Components
    \item \Cref{pf-prop-pi}: Proof of Proposition~\ref{prop-pi}
    \vspace{0.1 cm}
    \item \Cref{pf-thm-ploss}: Proof of Theorem~\ref{thm-ploss}
    \vspace{0.1 cm}
    \item \Cref{sec:gen_sigma}: Difficulty in the Analysis with a General Covariance Matrix $\widetilde{\bm{\Sigma}}$
    \vspace{0.1 cm}
    \item \Cref{lemma-sec}: Lemmas Used and Their Proofs
    \vspace{0.1 cm}
    \item \Cref{app:experimental-details}: Experimental Details
    \begin{itemize}
        \vspace{0.1 cm}
        \item \Cref{app:add-baseline-detials}: Baseline Details
        \vspace{0.1 cm}
        \item \Cref{app:language-hyperparameters}: Language Model Hyper-Parameters
        \vspace{0.1 cm}
        \item \Cref{app:further-language-evaluation-details}: Language Model Evaluation Details
        \vspace{0.1 cm}
        \item \Cref{app:vision-hyperparameters}: Vision Model Implementation Details
        % \vspace{0.1 cm}
        % \item \Cref{app:further-vision-evaluation-details}: Vision Evaluation Details
    \end{itemize}
    \vspace{0.1 cm}
    \item \Cref{add-vis-results}: Detailed Vision Results and Ablations
    \begin{itemize}
        \item \Cref{app:dist}: {Comparison with a Distillation-Based Method for Mitigating Forgetting}
    \end{itemize}
    \vspace{0.1 cm}
    \item \Cref{add-lang-results}: Additional Language Model Results and Ablations
    \begin{itemize}
        \vspace{0.1 cm}
        \item \Cref{app:extended-commonsense-results}: Extended Commonsense Reasoning Results
        \vspace{0.1 cm}
        \item  \Cref{app:token-wise-ablation}: Token-wise Sample Weighting Ablations
        \vspace{0.1 cm}
        % \item \Cref{lora-ablation}: LoRA Fine-tuning Ablations
        % \vspace{0.1 cm}
        \item \Cref{app:extended-wa-results}: Extended Weight Averaging Results
        \vspace{0.1 cm}
    \end{itemize}
\end{itemize}

\clearpage

\section{Extended Related Work}
\label{app:related-work}
The majority of the approaches for mitigating %catastrophic 
forgetting assume task-specific knowledge access to different extents; either (a subset of) the pre-training dataset itself or some information/statistic computed from pre-training data. Below, we describe the data-aware approaches based on how they make use of task-specific knowledge.

\textbf{Regularization-based methods. } This line of work aims to preserve performance on previously learned tasks by keeping the (fine-tuned) model parameters close to the pre-trained model. The key idea is to introduce task-specific regularization in the fine-tuning phase which will penalize updates along the ``important'' directions for the old tasks %, i.e., pre-trained model 
\citep{ahn2019uncertainty}. 
% The simplest example is the standard $\ell_2$ regularization, i.e., we optimize the regularized loss function $\mathcal L(\theta) = \ell_B(\theta) + \frac{\lambda}{2} \| \theta - \theta^* \|_2^2$, where $\theta^*$ is the weights of the base model and $\| \cdot \|_2$ is the Euclidean norm. Minimizing for the $\ell_2$-regularized loss precludes drifting away from $\theta^*$. The main downside of naive regularization is that the update is constrained uniformly across all directions; one could design more sophisticated and mathematically-grounded techniques in the presence of task-specific information. 
\citet{kirkpatrick2016overcoming} introduces the elastic weight consolidation (EWC) algorithm, which estimates the important direction per-task by calculating a diagonal approximation to the Fisher information matrix (FIM), which  
% using per-sample gradients. 
% In simple terms, the approximate Fisher information matrix acts as a guide to infer important directions of update for the old task.  
% The approximation is calculated via the gradients of the previous task loss and requires access to dataset of the old task(s).
% Once computed, it 
acts as the weight matrix for the regularization term.
% , i.e., $\ell_{\text{EWC}} (\theta) = \ell_{FT} (\theta) + (\theta - \theta^*)^\top F_A (\theta - \theta^*)$ where $F_{PT}$ is the approximate FIM for the pre-trained (PT) model and $\ell_{FT}$ is the original loss for the fine-tuning (FT) task. 
Several variants of EWC have been subsequently proposed \citep{schwarz2018progress, ritter2018online, Lee2020ContinualLW, liu2018rotate}. \citet{zenke2017continual, aljundi2018memory} adopt online strategies to infer the importance of each parameter by their variational effect on the model outputs. In a spirit similar to EWC, \citet{lee2017overcoming} incrementally matches the posterior of the pre-trained model and the new task by assuming Gaussian posteriors. 

\textbf{Optimization-driven methods. } Another perspective to mitigating forgetting is guiding the optimization process by constraining the algorithms directly as opposed to manipulating the loss function. 
The core idea is to keep track of \enquote{important directions} for the old tasks, and train on the new task \enquote{orthogonally.} 
This could be done by storing prior data samples or gradients in a buffer \citep{lopezpaz2017gradient, farajtabar2020orthogonal, chaudhry2018efficient} or by incrementally expanding the subspace of important directions without storing task-specific information \citep{zeng2019continual, wang2021training, wang2023orthogonal}.

\textbf{Replay-based methods. } Drawing inspiration from the complementary learning systems theory \citep{mcclelland1995why}, a more direct approach is to introduce samples from old tasks into the training process for the new task. Samples are selected in a streaming fashion or by manually crafting a subset on demand, stored in dedicated buffers 
% (sometimes referred to as episodic memory or reservoirs), 
and replayed during the fine-tuning. 
The intuition is that the task-specific representations are refreshed periodically through historical data.
% , enabling sequential learning without losing prior capabilities on old tasks.

Replay-based methods consist of two fundamental components: data selection and data reiteration mechanisms. 
When the data is received in a streaming fashion, information has to be buffered online \citep{riemer2019learning, chaudhry2019continual, isle2018selective, delange2021continual}. 
In the case when datasets are available on demand, \citet{rebuffi2017icarl} selects samples which are \enquote{representative} of their respective class, while others focus on inducing diversity \citep{aljundi2019gradient, Bang2021RainbowMC} and balance \citep{borsos2020coresets, tiwari2021gcr} across buffered data. 
For the scenarios in which storage is limited, \citet{caccia2020online, wang2022memory} develop compression methods for buffered data.
% Depending on the particular learning setup, the buffered data is replayed into the fine-tuning process in a uniquely-crafted manner. \citep{rolnick2019experience} uses the data buffer to maintain a correction mechanism to address distribution shift in off-policy learning, \citep{}. Some work develop compression methods for buffered data \citep{caccia2020online, wang2022memory} when storage is limited.
As a complimentary component to the data selection process, how the buffered data is replayed plays a significant role in the success of such methods. A fundamental idea, which has several interpretations across the board, is knowledge distillation \citep{hinton2015distilling}. Prior work argues that augmenting fine-tuning with knowledge distillation shows great performance on the forgetting front \citep{lopezpaz2017gradient, chaudhry2018efficient, rebuffi2017icarl, Jung2017LessforgetfulLF, Triki2017EncoderBL, li2016learning, Lee2019OvercomingCF, dhar2019learning}. 

An orthogonal research direction focuses on maintaining a generative model that could reliably output pseudo-samples that are representative of the dataset of the old tasks \citep{kemker2018fearnet, wu2018incremental}. 
Note that generative approaches are prone to scalability issues and distribution shifts. 

\textbf{Architecture-driven methods. } Another technique to limit the interference between tasks is allocating a separate trainable set of parameters per task. 
This could be done by initializing a sub-networks per new task \citep{Rusu2016ProgressiveNN, aljundi2017expertgate, patrick2020routing, rajasegaran2019random, Ramesh2021ModelZA, wang2023incorporating, wang2022coscl}, gradually expanding the parameters of a base network \citep{yoon2018lifelong, Ostapenko2019LearningTR, hung2019compacting}, or segregating a fixed model into task-specific subset of parameters \citep{mallya2018piggyback, kang2022forgetfree, serra2018overcoming, worstman2020supermasks, Mallya2017PackNetAM, gurbuz2022nispa, jung2020continual}. While some parameters are task-specific, 
% and are not altered by the training process for other tasks, 
parts of the overall model could be shared to enable knowledge transfer. 
The main downside associated is that the task identity must be available during inference to (de)activate relevant sub-networks, hindering versatility. 
\citet{aljundi2017expertgate} develop dedicated strategies to overcome the need for task identification by automatizing task-specific parameter activation.

\textbf{\method and its connection to other %modern
ML applications. } {Our approach has connections to factuality in LLM training \citep{ghosal2024understanding, gekhman2024finetuning}. For a fine-tuning task on a factual knowledge-based dataset, samples that are not properly represented in the pre-training distribution could be considered hard. In the data-oblivious setting, one principled way would be to rank \emph{hardness} based on pre-training loss values; therefore, \method would identify such samples and guide the fine-tuning process such that those samples will not unexpectedly deteriorate the performance.

Although fundamentally different in its objective, machine unlearning is another avenue of application for our sample weighting approach which could be interpreted as a means of \emph{soft unlearning}. While unlearning and forgetting sound similar, they are not the same in spirit; in unlearning we deliberately induce “forgetting” on some samples, whereas in our context, forgetting is an undesirable side effect that we want to avoid. Hence, it might not be straightforward extend techniques of one into another, one could extend our sample-wise weighting scheme with appropriate modifications to help selectively unlearn specific samples.}

\section{Our Algorithm in the Presence of Task-Specific Model Components}
\label{alg:multi-head}
Suppose our model is parameterized by $\bm{\theta} = {\mathbf{U}} \cup {\mathbf{V}}$, where $\mathbf{U}$ is the common/shared part of the model for all tasks (i.e., this part remains the same for all tasks), and $\mathbf{V}$ is the task-specific part of the model. In particular, in our vision experiments, the models have task-specific prediction heads (i.e., softmax layers) and batch-norm (BN). The modified version of \Cref{alg:method} in the presence of task-specific components is stated in \Cref{alg:example-ii}. The main differences from \Cref{alg:method} are steps (i) and (iv) -- these steps optimize the task-specific part for the new task with uniform weighting. %over the samples of the new task.
It is worth mentioning that if our model consists of task-specific prediction heads -- which is the case in our vision experiments -- then steps (i) and (iv) are just vanilla linear probing with the pre-trained body and the body learned after fine-tuning, respectively. 

 \begin{algorithm}[h] %!htb
    \caption{\textbf{F}ine-tuning with Pre-trained \textbf{L}oss-\textbf{O}riented \textbf{W}eighting (\methodbold)}
    \label{alg:example-ii}
 \begin{algorithmic}
    \STATE {\bfseries Input:} Pre-trained model $\bm{\theta}_{*}^{(1)} = {\mathbf{U}}_{*}^{(1)} \cup {\mathbf{V}}_{*}^{(1)}$, dataset $\{(\mathbf{x}_i, \text{y}_i)\}_{i=1}^n$ for the new task, and temperature parameter $\tau$.
    \vspace{0.25 cm}
    \STATE $f_i({\mathbf{U}}, {\mathbf{V}}) \rightarrow$ $i^\text{th}$ sample’s loss at $\bm{\theta} = {\mathbf{U}} \cup {\mathbf{V}}$, with a non-negative loss function (e.g., cross-entropy loss).
    \vspace{0.25 cm}
    \STATE \textbf{Step (i)} {Fine-tune task-specific part for new task with vanilla unweighted loss:} ${\mathbf{V}}_{*}^{(2)} := \text{argmin }_{\mathbf{V}} \sum_{i=1}^n f_i({\mathbf{U}}_{*}^{(1)}, {\mathbf{V}}).$
    \vspace{0.25 cm}
    \STATE \textbf{Step (ii)} Compute sample weights: $w_i = \exp\left(- {f_i({\mathbf{U}}_{*}^{(1)}, {\mathbf{V}}_{*}^{(2)})}\big/{\tau}\right)$. 
    \vspace{0.25 cm}
    \STATE \textbf{Step (iii)} Fine-tune full model \textbf{with weighted loss}: $\overline{\mathbf{U}}_{*}^{(2)}, \overline{\mathbf{V}}_{*}^{(2)} := \text{argmin }_{\mathbf{U}, \mathbf{V}} \sum_{i=1}^n w_i f_i({\mathbf{U}}, {\mathbf{V}})$.
    \vspace{0.25 cm}
    {\STATE \textbf{Step (iv)} {Fine-tune task-specific part for new task  \textit{using the learned common part} with vanilla unweighted loss:} $\widehat{\mathbf{V}}_{*}^{(2)} := \text{argmin }_{\mathbf{V}} \sum_{i=1}^n f_i(\overline{\mathbf{U}}_{*}^{(2)}, {\mathbf{V}}).$}
    \vspace{0.25 cm}
    \STATE \textbf{Output:} New model for 
    \begin{itemize}
        \item Original/pre-training task is $\widehat{\bm{\theta}}_{*}^{(1)} = \overline{\mathbf{U}}_{*}^{(2)} \cup {\mathbf{V}}_{*}^{(1)}$.
        \item New/fine-tuning task is $\widehat{\bm{\theta}}_{*}^{(2)} = \overline{\mathbf{U}}_{*}^{(2)} \cup \widehat{\mathbf{V}}_{*}^{(2)}$.
    \end{itemize}
 \end{algorithmic}
 \label{alg:vision_algo}
 \end{algorithm}

\begin{remark}
    In all our vision experiments (with task-specific parts), we set $\tau = \textup{median}(f_i({\mathbf{U}}_{*}^{(1)}, {\mathbf{V}}_{*}^{(2)}))$ (similar to \Cref{rmk-tau}).
\end{remark}

%\section{Proofs of the Results in Section~\ref{sec: theory}}
\section{Proof of Proposition~\ref{prop-pi}}
\label{pf-prop-pi}
\begin{proof}
    We wish to minimize $g(\bm{\pi}) = \sum_{i=1}^n \pi_i f_i(\bm{\theta}^{*}) + \tau \sum_{i=1}^n \pi_i \log \pi_i$ subject to $\sum_{i=1}^n \pi_i = 1$ and $\pi_i \geq 0$ for all $i \in [n]$. The proof is a straightforward application of Lagrangian multipliers. It is enough to enforce $\sum_{i=1}^n \pi_i = 1$ only ($\pi_i \geq 0$ for all $i \in [n]$ will also follow). For that, the Lagrangian function is:
    \begin{equation}
        J(\bm{\pi}, \lambda) = \sum_{i=1}^n \pi_i f_i(\bm{\theta}^{*}) + \tau \sum_{i=1}^n \pi_i \log \pi_i + \lambda \Big(\sum_{i=1}^n \pi_i - 1\Big),
    \end{equation}
    where $\lambda$ is the Lagrangian multiplier. Now, at the optimal point $\bm{\pi}^{*} = [\pi_1^{*}, \ldots, \pi_n^{*}]^\top$, we must have:
    \begin{equation}
        \frac{\partial J}{\partial \pi_i}\Bigg|_{\pi_i^{*}} = f_i(\bm{\theta}^{*}) + \tau \left(1 + \log \pi_i^{*}\right) + \lambda = 0,
    \end{equation}
    for all $i \in [n]$. Simplifying, we get:
    \begin{equation}
        \pi_i^{*} = \frac{1}{Z} \exp\left(-\frac{f_i(\bm{\theta}^{*})}{\tau}\right),
    \end{equation}
    where $Z = \exp\left(\left(1 + \frac{\lambda}{\tau}\right)\right)$ is the normalizing constant. To have $\sum_{i=1}^n \pi_i = 1$, we get $Z = \sum_{j=1}^n \exp \left( - \frac{f_j({\mathbf{{W}}}^{*})}{\tau} \right)$. Also, note that we are good with the non-negativity constraints.
\end{proof}

\section{Proof of Theorem~\ref{thm-ploss}}
\label{pf-thm-ploss}
\begin{proof}
Note that:
\begin{equation}
    \widehat{\text{err}}_2(\widehat{\bm{\theta}}) = \big(\widehat{\bm{\theta}} - \widetilde{\bm{\theta}}_{*}\big)^\top \mathbb{E}_{\widetilde{\mathcal{D}}}\Big[w(\widetilde{\mathbf{x}}, \widetilde{\text{y}}) \widetilde{\mathbf{x}} \widetilde{\mathbf{x}}^\top\Big] \big(\widehat{\bm{\theta}} - \widetilde{\bm{\theta}}_{*}\big).
\end{equation}
Also, after plugging in $\widetilde{\text{y}} = \big\langle \widetilde{\bm{\theta}}_{*}, \widetilde{\mathbf{x}} \big\rangle$, we get:
$$w(\widetilde{\mathbf{x}}, \widetilde{\text{y}}) = \exp\Bigg(-\frac{\big(\langle \bm{\theta}_{\ast} - \widetilde{\bm{\theta}}_{*}, \widetilde{\mathbf{x}} \rangle\big)^2}{\tau}\Bigg).$$
Recall $\mathbf{e} := \bm{\theta}_{\ast} - \widetilde{\bm{\theta}}_{*}$ and  $\overline{\mathbf{e}} := \frac{\mathbf{e}}{\|\mathbf{e}\|_2}$. Suppose $\tau = \alpha \|\mathbf{e}\|_2^2$, for some $\alpha > 0$. Then $w(\widetilde{\mathbf{x}}, \widetilde{\text{y}}) = \exp\Big(-\frac{(\langle \overline{\mathbf{e}}, \widetilde{\mathbf{x}} \rangle)^2}{\alpha}\Big)$, and we can focus on
\begin{equation}
    \label{eq:23-jan30}
    \widetilde{\bm{\Sigma}}' := \mathbb{E}_{\widetilde{\mathbf{x}} \sim \widetilde{\mathcal{P}}}\Bigg[\exp\Bigg(-\frac{\big(\langle \overline{\mathbf{e}}, \widetilde{\mathbf{x}} \rangle\big)^2}{\alpha}\Bigg) \widetilde{\mathbf{x}} \widetilde{\mathbf{x}}^\top \Bigg].
\end{equation}
Let $\mu = \big(\frac{\alpha}{\alpha+2}\big)^{1/2} = \Big(\frac{\tau}{\tau + 2\|\mathbf{e}\|_2^2}\Big)^{1/2}$. As per \Cref{lem1}, we have:
\begin{equation}
    \label{eq:8-jan14}
    \widetilde{\bm{\Sigma}}' = \mu \big(\mathbf{I}_d - \mathbf{Q}\big),
\end{equation}
where 
\begin{equation}
    \mathbf{Q} = (1 - \mu^2)\overline{\mathbf{{e}}} \overline{\mathbf{{e}}}^\top + \rho^2 (1 - \mu^2) \overline{\mathbf{{e}}}_{\perp} \overline{\mathbf{{e}}}_{\perp}^\top - \rho \mu^2 \big(\overline{\mathbf{{e}}} \overline{\mathbf{{e}}}_{\perp}^\top + \overline{\mathbf{{e}}}_{\perp} \overline{\mathbf{{e}}}^\top\big).
\end{equation}
So if we minimize $\widehat{\text{err}}_2(\widehat{\bm{\theta}})$ with GD starting from {$\widehat{\bm{\theta}}_0 = \bm{\theta}_{\ast}$} and using a constant learning rate $\widehat{\eta}$, our iterate $\widehat{\bm{\theta}}_K$ at the $K^\text{th}$ iteration satisfies:
\begin{flalign}
    \widehat{\bm{\theta}}_K - \widetilde{\bm{\theta}}_{*} & = \Big(\mathbf{I}_d - 2 \widehat{\eta} \widetilde{\bm{\Sigma}}'\Big)^K \big(\bm{\theta}_{\ast} - \widetilde{\bm{\theta}}_{*}\big) = \Big(\mathbf{I}_d - 2 \widehat{\eta} \widetilde{\bm{\Sigma}}'\Big)^K \mathbf{e},
\end{flalign}
where the last step follows by recalling that $\bm{\theta}_{\ast} - \widetilde{\bm{\theta}}_{*} = \mathbf{e}$, and $\widetilde{\bm{\Sigma}}'$ is given by \cref{eq:8-jan14}.
\end{proof}

%\section{Derivation of the Weighted Covariance Matrix $\widetilde{\bm{\Sigma}}'$ for a General Covariance Matrix $\widetilde{\bm{\Sigma}}$}
\section{Difficulty in the Analysis with a General Covariance Matrix $\widetilde{\bm{\Sigma}}$}
\label{sec:gen_sigma}
We will first derive the \textit{weighted} (fine-tuning) data covariance matrix $\widetilde{\bm{\Sigma}}'$ in the context of \Cref{thm-ploss} for a general (fine-tuning) data covariance matrix $\widetilde{\bm{\Sigma}}$. Specifically, following the proof of \Cref{thm-ploss}, we have:
\begin{equation}
    \widetilde{\bm{\Sigma}}' := \mathbb{E}_{\widetilde{\mathbf{x}} \sim \mathcal{N}(\vec{\bm{0}}_d, \widetilde{\bm{\Sigma}})}\Bigg[\exp\Bigg(-\frac{\big(\langle {\mathbf{e}}, \widetilde{\mathbf{x}} \rangle\big)^2}{\tau}\Bigg) \widetilde{\mathbf{x}} \widetilde{\mathbf{x}}^\top \Bigg].
\end{equation}
Note that $\widetilde{\mathbf{x}} = \widetilde{\bm{\Sigma}}^{1/2} \mathbf{z}$, where ${\mathbf{z}} \sim \mathcal{N}(\vec{0}_d, \mathbf{I}_d)$. Using this above, we get:
\begin{equation}
    \widetilde{\bm{\Sigma}}' = \widetilde{\bm{\Sigma}}^{1/2} \mathbb{E}\Bigg[\exp\Bigg(-\frac{\big(\langle {\mathbf{e}}, \widetilde{\bm{\Sigma}}^{1/2} \mathbf{z} \rangle\big)^2}{\tau}\Bigg) \mathbf{z} \mathbf{z}^\top \Bigg] \widetilde{\bm{\Sigma}}^{1/2} = \widetilde{\bm{\Sigma}}^{1/2} \mathbb{E}\Bigg[\exp\Bigg(-\frac{\big(\langle {\widetilde{\bm{\Sigma}}^{1/2} \mathbf{e}}, \mathbf{z} \rangle\big)^2}{\tau}\Bigg) \mathbf{z} \mathbf{z}^\top \Bigg] \widetilde{\bm{\Sigma}}^{1/2},
\end{equation}
where the last step follows by using the symmetry of $\widetilde{\bm{\Sigma}}$. Let $\tau = \alpha \big\|\widetilde{\bm{\Sigma}}^{1/2} \mathbf{e}\big\|_2^2$ for some $\alpha > 0$. Also, let $\mathbf{r} := {(\widetilde{\bm{\Sigma}}^{1/2} \mathbf{e})}/{\|\widetilde{\bm{\Sigma}}^{1/2} \mathbf{e}\|_2}$. In that case, we have:
\begin{equation}
    \label{eq:9-jan11}
    \widetilde{\bm{\Sigma}}' = \widetilde{\bm{\Sigma}}^{1/2} \mathbf{M} \widetilde{\bm{\Sigma}}^{1/2}, \text{ where } \mathbf{M} := \mathbb{E}\Bigg[\exp\Bigg(-\frac{\big(\langle \mathbf{r}, \mathbf{z} \rangle\big)^2}{\alpha}\Bigg) \mathbf{z} \mathbf{z}^\top \Bigg].
\end{equation}

Suppose $\{{\mathbf{{r}}}_{\perp, j}\}_{j=1}^{d-1}$ is an orthonormal basis for the subspace of $\mathbb{R}^d$ orthogonal to $\mathbf{r}$; so $\langle {\mathbf{{r}}}_{\perp, j}, \mathbf{r} \rangle = 0$ $\forall$ $j \in [d-1]$ and $\langle {\mathbf{{r}}}_{\perp, j}, {\mathbf{{r}}}_{\perp, k} \rangle = \mathds{1}(j=k)$ $\forall$ $j, k \in [d-1]$. {Note that $\{\mathbf{r}, \mathbf{r}_{\perp, 1}, \ldots, \mathbf{r}_{\perp, d-1}\}$ forms an orthonormal basis for $\mathbb{R}^d$.}  
Then, as per \Cref{lem1-jan30}, we have that $\mathbf{r}$ is an eigenvector of $\mathbf{M}$ with eigenvalue $\big(\frac{\alpha}{\alpha+2}\big)^{3/2}$, and each $\mathbf{r}_{\perp, j}$ is an eigenvector of $\mathbf{M}$ with eigenvalue $\big(\frac{\alpha}{\alpha+2}\big)^{1/2}$. For brevity, let $\mu = \big(\frac{\alpha}{\alpha+2}\big)^{1/2}$. Then, we can write:
\begin{equation}
    \label{eq:10-jan12}
    \mathbf{M} = \mu^3 \mathbf{r} \mathbf{r}^\top + \mu \sum_{j=1}^{d-1} \mathbf{r}_{\perp, j} \mathbf{r}_{\perp, j}^\top = \mu^3 \mathbf{r} \mathbf{r}^\top + \mu \Big(\mathbf{I}_d - \mathbf{r} \mathbf{r}^\top\Big),
\end{equation}
where the last step follows because $\{\mathbf{r}, \mathbf{r}_{\perp, 1}, \ldots, \mathbf{r}_{\perp, d-1}\}$ forms an orthonormal basis for $\mathbb{R}^d$, due to which $\mathbf{r} \mathbf{r}^\top + \sum_{j=1}^{d-1} \mathbf{r}_{\perp, j} \mathbf{r}_{\perp, j}^\top = \mathbf{I}_d$. Simplifying \cref{eq:10-jan12} a bit, we get:
\begin{equation}
    \mathbf{M} = \mu \Big(\mathbf{I}_d - (1 - \mu^2)\mathbf{r} \mathbf{r}^\top\Big).
\end{equation}
Plugging this into \cref{eq:9-jan11} and recalling that $\mathbf{r} := {(\widetilde{\bm{\Sigma}}^{1/2} \mathbf{e})}/{\|\widetilde{\bm{\Sigma}}^{1/2} \mathbf{e}\|_2}$, we get:
\begin{equation}
    \label{eq:11-jan11}
    \widetilde{\bm{\Sigma}}' = \mu \mathbf{B}, \text{ where } \mathbf{B} := \Bigg(\widetilde{\bm{\Sigma}} - (1 - \mu^2)\frac{{\widetilde{\bm{\Sigma}} \mathbf{e} \mathbf{e}^\top \widetilde{\bm{\Sigma}}}}{\mathbf{e}^\top \widetilde{\bm{\Sigma}} \mathbf{e}}\Bigg).
\end{equation}
\Cref{eq:11-jan11} is the weighted covariance matrix for a general $\widetilde{\bm{\Sigma}}$. 

\begin{remark}[\textbf{Difficulty with general $\widetilde{\bm{\Sigma}}$}]
It is hard to proceed with the analysis after this point because it %appears 
is difficult to characterize the eigen-spectrum of $\mathbf{B}$ in general, without assuming any relation between $\widetilde{\bm{\Sigma}}$ and $\mathbf{e}$. This is what we meant in \Cref{rmk-just-jan30}.
\end{remark}

\section{Lemmas Used and Their Proofs}
\label{lemma-sec}
\begin{lemma}
    \label{lem1}
    In the proof of \Cref{thm-ploss}, recall that $\tau = \alpha \|\mathbf{{e}}\|_2^2$. Then, we have:
    \begin{flalign*}
        \widetilde{\bm{\Sigma}}' & := \mathbb{E}_{\widetilde{\mathbf{{x}}} \sim \widetilde{\mathcal{P}}}\Bigg[\exp\Bigg(-\frac{\big(\langle \overline{\mathbf{{e}}}, \widetilde{\mathbf{{x}}} \rangle\big)^2}{\alpha}\Bigg) \widetilde{\mathbf{{x}}} \widetilde{\mathbf{{x}}}^\top \Bigg] 
        = \mu \Big(\mathbf{{I}}_d - (1 - \mu^2)\overline{\mathbf{{e}}} \overline{\mathbf{{e}}}^\top - \rho^2 (1 - \mu^2) \overline{\mathbf{{e}}}_{\perp} \overline{\mathbf{{e}}}_{\perp}^\top + \rho \mu^2 \big(\overline{\mathbf{{e}}} \overline{\mathbf{{e}}}_{\perp}^\top + \overline{\mathbf{{e}}}_{\perp} \overline{\mathbf{{e}}}^\top\big)\Big),
    \end{flalign*}
    where $\mu = \big(\frac{\alpha}{\alpha+2}\big)^{1/2} = \Big(\frac{\tau}{\tau + 2\|\mathbf{{e}}\|_2^2}\Big)^{1/2}$.
\end{lemma}
\begin{proof}
    Recall that $\overline{\mathbf{{e}}}$ and $\overline{\mathbf{{e}}}_{\perp}$ are orthogonal to each other and both are unit-norm. Suppose $\{\overline{\mathbf{{e}}}_{\perp, 3}, \overline{\mathbf{{e}}}_{\perp, 4}, \ldots, \overline{\mathbf{{e}}}_{\perp, d}\}$ is an orthonormal basis for the $(d-2)$-dimensional subspace of $\mathbb{R}^d$ orthogonal to $\overline{\mathbf{{e}}}$ and $\overline{\mathbf{{e}}}_{\perp}$. Thus, $\{\overline{\mathbf{{e}}}, \overline{\mathbf{{e}}}_{\perp}, \overline{\mathbf{{e}}}_{\perp, 3}, \overline{\mathbf{{e}}}_{\perp, 4},  \ldots, \overline{\mathbf{{e}}}_{\perp, d}\}$ is an orthonormal basis for $\mathbb{R}^d$. Then using \Cref{lem0}, we can write:
    \begin{equation}
        \widetilde{\mathbf{x}} = \text{z}_1 \overline{\mathbf{{e}}} + \Big(\rho \text{z}_1 + \sqrt{1-\rho^2} \text{z}_2\Big) \overline{\mathbf{{e}}}_{\perp} + \sum_{j=3}^d \text{z}_j \overline{\mathbf{{e}}}_{\perp, j},
    \end{equation}
    where $\{\text{z}_j\}_{j=1}^d \underset{\text{iid}}{\sim} \mathcal{N}(0,1)$.
    
    Using independence and zero-mean nature of $\{\text{z}_j\}_{j=1}^d$, we get:
    \begin{multline}
        \label{eq:33-jan14}
        \widetilde{\bm{\Sigma}}' = \underbrace{\mathbb{E}\Big[\exp\Big(-\frac{\text{z}_1^2}{\alpha}\Big)\text{z}_1^2\Big]}_{:=\textup{T}_1} \overline{\mathbf{{e}}} \overline{\mathbf{{e}}}^\top + \underbrace{\mathbb{E}\Big[\exp\Big(-\frac{\text{z}_1^2}{\alpha}\Big)\text{z}_1 \Big(\rho \text{z}_1 + \sqrt{1-\rho^2} \text{z}_2\Big) \Big]}_{:=\textup{T}_2} \big(\overline{\mathbf{{e}}} \overline{\mathbf{{e}}}_{\perp}^\top + \overline{\mathbf{{e}}}_{\perp} \overline{\mathbf{{e}}}^\top\big) \\ 
        + \underbrace{\mathbb{E}\Big[\exp\Big(-\frac{\text{z}_1^2}{\alpha}\Big)\Big(\rho \text{z}_1 + \sqrt{1-\rho^2} \text{z}_2\Big)^2 \Big]}_{:=\textup{T}_3}\overline{\mathbf{{e}}}_{\perp} \overline{\mathbf{{e}}}_{\perp}^\top + \sum_{j=3}^d \underbrace{\mathbb{E}\Big[\exp\Big(-\frac{\text{z}_1^2}{\alpha}\Big) \Big]}_{:=\textup{T}_4} \underbrace{\mathbb{E}\big[\text{z}_j^2\big]}_{=1}   \overline{\mathbf{{e}}}_{\perp, j} \overline{\mathbf{{e}}}_{\perp, j}^\top.
    \end{multline} 
    Note that (we use the independence of $\text{z}_1$ and $\text{z}_2$):
    \begin{equation}
        \label{eq:34-jan14}
        \textup{T}_2 = \rho \textup{T}_1 + \sqrt{1-\rho^2} \mathbb{E}\Big[\exp\Big(-\frac{\text{z}_1^2}{\alpha}\Big)\text{z}_1\Big] \underbrace{\mathbb{E}\big[\text{z}_2\big]}_{=0} = \rho \textup{T}_1,
    \end{equation}
    and 
    \begin{equation}
        \label{eq:35-jan14}
        \textup{T}_3 = \rho^2 \textup{T}_1 + 2 \rho \sqrt{1-\rho^2} \Big[\exp\Big(-\frac{\text{z}_1^2}{\alpha}\Big)\text{z}_1\Big] \underbrace{\mathbb{E}\big[\text{z}_2\big]}_{=0} + (1-\rho^2) \textup{T}_4  \underbrace{\mathbb{E}\big[\text{z}_2^2\big]}_{=1} = \rho^2 \textup{T}_1 + (1-\rho^2) \textup{T}_4.
    \end{equation}
    In the above two equations, we have again used the independence of $\text{z}_1$ and $\text{z}_2$. Now we will compute $\textup{T}_1$ and $\textup{T}_4$. We have:
    \begin{equation}
        \label{eq:36-jan14}
        \textup{T}_1 = \Bigg(\frac{1}{\sqrt{2 \pi}} \int_{-\infty}^{\infty} \text{z}_1^2 \exp\Big(-\text{z}_1^2 \Big(\frac{1}{\alpha} + \frac{1}{2}\Big)\Big) \text{d} \text{z}_1\Bigg) = \Big(\frac{\alpha}{\alpha+2}\Big)^{3/2},
    \end{equation}
    and
    \begin{equation}
        \label{eq:37-jan14}
        \textup{T}_4 = \Bigg(\frac{1}{\sqrt{2 \pi}} \int_{-\infty}^{\infty} \exp\Big(-\text{z}_1^2 \Big(\frac{1}{\alpha} + \frac{1}{2}\Big)\Big) \text{d} \text{z}_1\Bigg) = \Big(\frac{\alpha}{\alpha+2}\Big)^{1/2}.
    \end{equation}
    Recall that $\mu = \big(\frac{\alpha}{\alpha+2}\big)^{1/2}$. Plugging this into Equations (\ref{eq:34-jan14}) to (\ref{eq:37-jan14}) gives us:
    \begin{equation}
        \textup{T}_1 = \mu^3, \textup{T}_2 = \rho \mu^3, \textup{T}_3 = \rho^2 \mu^3 + (1 - \rho^2) \mu, \text{ and } \textup{T}_4 = \mu.
    \end{equation}
    Plugging this into \cref{eq:33-jan14} gives us:
    \begin{equation}
        \widetilde{\bm{\Sigma}}' = \mu^3 \overline{\mathbf{{e}}} \overline{\mathbf{{e}}}^\top + \rho \mu^3 \big(\overline{\mathbf{{e}}} \overline{\mathbf{{e}}}_{\perp}^\top + \overline{\mathbf{{e}}}_{\perp} \overline{\mathbf{{e}}}^\top\big) + \Big(\rho^2 \mu^3 + (1 - \rho^2) \mu \Big) \overline{\mathbf{{e}}}_{\perp} \overline{\mathbf{{e}}}_{\perp}^\top + \mu \sum_{j=3}^d \overline{\mathbf{{e}}}_{\perp, j} \overline{\mathbf{{e}}}_{\perp, j}^\top.
    \end{equation}
    Recall that $\{\overline{\mathbf{{e}}}, \overline{\mathbf{{e}}}_{\perp}, \overline{\mathbf{{e}}}_{\perp, 3}, \overline{\mathbf{{e}}}_{\perp, 4}, \ldots, \overline{\mathbf{{e}}}_{\perp, d}\}$ is an orthonormal basis for $\mathbb{R}^d$. Thus, $\sum_{j=3}^d \overline{\mathbf{{e}}}_{\perp, j} \overline{\mathbf{{e}}}_{\perp, j}^\top = \mathbf{I}_d - \overline{\mathbf{{e}}} \overline{\mathbf{{e}}}^\top - \overline{\mathbf{{e}}}_{\perp} \overline{\mathbf{{e}}}_{\perp}^\top$. Using this above, we get:
    \begin{equation}
        \widetilde{\bm{\Sigma}}' = \mu \Big(\mathbf{I}_d - (1 - \mu^2)\overline{\mathbf{{e}}} \overline{\mathbf{{e}}}^\top - \rho^2 (1 - \mu^2) \overline{\mathbf{{e}}}_{\perp} \overline{\mathbf{{e}}}_{\perp}^\top + \rho \mu^2 \big(\overline{\mathbf{{e}}} \overline{\mathbf{{e}}}_{\perp}^\top + \overline{\mathbf{{e}}}_{\perp} \overline{\mathbf{{e}}}^\top\big)\Big).
    \end{equation}
    This finishes the proof.
\end{proof}

\begin{lemma}
    \label{lem0}
    Suppose $\{\overline{\mathbf{{e}}}, \overline{\mathbf{{e}}}_{\perp}, \overline{\mathbf{{e}}}_{\perp, 3}, \overline{\mathbf{{e}}}_{\perp, 4}, \ldots, \overline{\mathbf{{e}}}_{\perp, d}\}$ is an orthonormal basis for $\mathbb{R}^d$. If $\widetilde{\mathbf{x}} \sim \mathcal{N}(\vec{\bm{0}}_d, \widetilde{\bm{\Sigma}})$, then we can write:
    \begin{equation}
        \label{eq:37-jan29}
        \widetilde{\mathbf{x}} = \textup{z}_1 \overline{\mathbf{{e}}} + \Big(\rho \textup{z}_1 + \sqrt{1-\rho^2} \textup{z}_2\Big) \overline{\mathbf{{e}}}_{\perp} + \sum_{j=3}^d \textup{z}_j \overline{\mathbf{{e}}}_{\perp, j},
    \end{equation}
    where $\{\textup{z}_j\}_{j=1}^d \underset{\textup{iid}}{\sim} \mathcal{N}(0,1)$.
\end{lemma}
\begin{proof}
    If $\widetilde{\mathbf{x}}$ is as per \cref{eq:37-jan29}, then clearly $\widetilde{\mathbf{x}}$ is a zero-mean Gaussian. All that remains to show is that $$\mathbb{E}\Big[\widetilde{\mathbf{x}} \widetilde{\mathbf{x}}^\top\Big] = \widetilde{\bm{\Sigma}} = \mathbf{{I}}_d + \rho  \big(\overline{\mathbf{{e}}} \overline{\mathbf{{e}}}_{\perp}^\top + \overline{\mathbf{{e}}}_{\perp} \overline{\mathbf{{e}}}^\top\big).$$
    
    Using independence and zero-mean nature of $\{\text{z}_j\}_{j=1}^d$, we get:
    \begin{multline}
        \label{eq:34-jan15}
        \mathbb{E}\Big[\widetilde{\mathbf{x}} \widetilde{\mathbf{x}}^\top\Big] = \underbrace{\mathbb{E}\big[\text{z}_1^2\big]}_{=1} \overline{\mathbf{{e}}} \overline{\mathbf{{e}}}^\top + \underbrace{\mathbb{E}\Big[\text{z}_1 \Big(\rho \text{z}_1 + \sqrt{1-\rho^2} \text{z}_2\Big) \Big]}_{:=\textup{(A)}} \big(\overline{\mathbf{{e}}} \overline{\mathbf{{e}}}_{\perp}^\top + \overline{\mathbf{{e}}}_{\perp} \overline{\mathbf{{e}}}^\top\big)  
        + \underbrace{\mathbb{E}\Big[\Big(\rho \text{z}_1 + \sqrt{1-\rho^2} \text{z}_2\Big)^2 \Big]}_{:=\textup{(B)}}\overline{\mathbf{{e}}}_{\perp} \overline{\mathbf{{e}}}_{\perp}^\top \\ 
        + \sum_{j=3}^d \underbrace{\mathbb{E}\big[\text{z}_j^2\big]}_{=1}   \overline{\mathbf{{e}}}_{\perp, j} \overline{\mathbf{{e}}}_{\perp, j}^\top.
    \end{multline} 
    Note that (we use the independence of $\text{z}_1$ and $\text{z}_2$):
    \begin{equation}
        \textup{(A)} = \rho \underbrace{\mathbb{E}[\text{z}_1^2]}_{=1} + \sqrt{1-\rho^2} \underbrace{\mathbb{E}[\text{z}_1]}_{=0} \underbrace{\mathbb{E}\big[\text{z}_2\big]}_{=0} = \rho,
    \end{equation}
    and 
    \begin{equation}
        \textup{(B)} = 
        \rho^2 \underbrace{\mathbb{E}\big[\text{z}_1^2\big]}_{=1} + 2 \rho \sqrt{1-\rho^2} \underbrace{\mathbb{E}\big[\text{z}_1\big]}_{=0} \underbrace{\mathbb{E}\big[\text{z}_2\big]}_{=0} + (1-\rho^2)  \underbrace{\mathbb{E}\big[\text{z}_2^2\big]}_{=1} = 1.
    \end{equation}
    Plugging this into \cref{eq:34-jan15}, we get:
    \begin{equation}
        \mathbb{E}\Big[\widetilde{\mathbf{x}} \widetilde{\mathbf{x}}^\top\Big] = \overline{\mathbf{{e}}} \overline{\mathbf{{e}}}^\top + \rho  \big(\overline{\mathbf{{e}}} \overline{\mathbf{{e}}}_{\perp}^\top + \overline{\mathbf{{e}}}_{\perp} \overline{\mathbf{{e}}}^\top\big)  
        + \overline{\mathbf{{e}}}_{\perp} \overline{\mathbf{{e}}}_{\perp}^\top + \sum_{j=3}^d   \overline{\mathbf{{e}}}_{\perp, j} \overline{\mathbf{{e}}}_{\perp, j}^\top.
    \end{equation}
    Recall that $\{\overline{\mathbf{{e}}}, \overline{\mathbf{{e}}}_{\perp}, \overline{\mathbf{{e}}}_{\perp, 3}, \overline{\mathbf{{e}}}_{\perp, 4},  \ldots, \overline{\mathbf{{e}}}_{\perp, d}\}$ is an orthonormal basis for $\mathbb{R}^d$. Thus, $\sum_{j=3}^d \overline{\mathbf{{e}}}_{\perp, j} \overline{\mathbf{{e}}}_{\perp, j}^\top = \mathbf{I}_d - \overline{\mathbf{{e}}} \overline{\mathbf{{e}}}^\top - \overline{\mathbf{{e}}}_{\perp} \overline{\mathbf{{e}}}_{\perp}^\top$. Using this above, we get:
    \begin{equation}
        \mathbb{E}\Big[\widetilde{\mathbf{x}} \widetilde{\mathbf{x}}^\top\Big] = \mathbf{I}_d + \rho  \big(\overline{\mathbf{{e}}} \overline{\mathbf{{e}}}_{\perp}^\top + \overline{\mathbf{{e}}}_{\perp} \overline{\mathbf{{e}}}^\top\big) = \widetilde{\bm{\Sigma}}.
    \end{equation}
    This finishes the proof.
\end{proof}

\begin{lemma}
    \label{lem3}
    Recall that 
    \begin{equation*}
        \mathbf{{Q}} = (1 - \mu^2)\overline{\mathbf{{e}}} \overline{\mathbf{{e}}}^\top  + \rho^2 (1 - \mu^2) \overline{\mathbf{{e}}}_{\perp} \overline{\mathbf{{e}}}_{\perp}^\top - \rho \mu^2 \big(\overline{\mathbf{{e}}} \overline{\mathbf{{e}}}_{\perp}^\top + \overline{\mathbf{{e}}}_{\perp} \overline{\mathbf{{e}}}^\top\big).
    \end{equation*}
    Let $$\mu = \sqrt{\frac{\beta (1-\rho^2)}{(1+\beta)(1-\beta \rho^2)}}$$
    for some $\beta \in (0,1]$. In that case, the eigenvalues of $\mathbf{{Q}}$ are:
    \begin{equation*}
    \widehat{\lambda}_1 = \frac{1 + \beta \rho^2}{1 + \beta} \text{ and } \widehat{\lambda}_2 = \rho^2\Bigg(\frac{1 - \beta}{1 - \beta \rho^2}\Bigg),
    \end{equation*}
    and the corresponding eigenvectors are:
    $$\widehat{\mathbf{{v}}}_1 = \frac{1}{\sqrt{1+\beta^2 \rho^2}} \overline{\mathbf{{e}}} - \frac{\beta \rho}{\sqrt{1+\beta^2 \rho^2}} \overline{\mathbf{{e}}}_{\perp} \text{ and } \widehat{\mathbf{{v}}}_2 = -\frac{\beta \rho}{\sqrt{1+\beta^2 \rho^2}} \overline{\mathbf{{e}}} - \frac{1}{\sqrt{1+\beta^2 \rho^2}} \overline{\mathbf{{e}}}_{\perp}.$$
\end{lemma}
\begin{proof}
$\mathbf{Q}$ is a rank-2 matrix and its two eigenvectors will be in the span of $\overline{\mathbf{{e}}}$ and $\overline{\mathbf{{e}}}_{\perp}$. In particular, an eigenvector of $\mathbf{Q}$ is of the form $[\overline{\mathbf{{e}}}, \overline{\mathbf{{e}}}_{\perp}] \mathbf{b}$, where $\mathbf{b} \in \mathbb{R}^{2 \times 1}$ is an eigenvector of the $2 \times 2$ matrix:
\begin{equation}
    \mathbf{A} := \begin{bmatrix}
                    (1 - \mu^2) & - \rho \mu^2 \\
                    - \rho \mu^2 & \rho^2 (1 - \mu^2)
                \end{bmatrix}.
\end{equation}
Also, the corresponding eigenvalues of $\mathbf{Q}$ are the corresponding eigenvalues of $\mathbf{A}$. It can be verified that the eigenvalues of $\mathbf{A}$ are:
\begin{equation}
    \widehat{\lambda}_1 = \frac{(1+\rho^2)(1-\mu^2)}{2} + \sqrt{\frac{(1-\rho^2)^2(1-\mu^2)^2}{4} + \rho^2 \mu^4}.
\end{equation}
and
\begin{equation}
    \widehat{\lambda}_2 = \frac{(1+\rho^2)(1-\mu^2)}{2} - \sqrt{\frac{(1-\rho^2)^2(1-\mu^2)^2}{4} + \rho^2 \mu^4}.
\end{equation}
The corresponding eigenvectors of $\mathbf{A}$ are:
\begin{equation}
    \widehat{\mathbf{b}}_1 = \frac{1}{\sqrt{b_{1,1}^2 + b_{1,2}^2}}
                    \begin{bmatrix}
                    %\frac{(1-\rho^2)(1-\mu^2)}{2} + \sqrt{\frac{(1-\rho^2)^2(1-\mu^2)^2}{4} + \rho^2 \mu^4} \\
                    %-\rho \mu^2,
                    b_{1,1}
                    \\
                    b_{1,2}
                \end{bmatrix}
\end{equation}
where %$b_1 = {\frac{(1-\rho^2)^2(1-\mu^2)^2}{2} + 2 \rho^2 \mu^4 + (1-\rho^2) (1-\mu^2) \sqrt{\frac{(1-\rho^2)^2(1-\mu^2)^2}{4} + \rho^2 \mu^4}}$, and
$b_{1,1} = \frac{(1-\rho^2)(1-\mu^2)}{2} + \sqrt{\frac{(1-\rho^2)^2(1-\mu^2)^2}{4} + \rho^2 \mu^4}$ and $b_{1,2} = -\rho \mu^2$, and 
\begin{equation}
    \widehat{\mathbf{b}}_2 = \frac{1}{\sqrt{b_{2,1}^2 + b_{2,2}^2}}
                    \begin{bmatrix}
                    %\frac{(1-\rho^2)(1-\mu^2)}{2} - \sqrt{\frac{(1-\rho^2)^2(1-\mu^2)^2}{4} + \rho^2 \mu^4} \\
                    %-\rho \mu^2
                    b_{2,1} \\
                    b_{2,2}
                \end{bmatrix},
\end{equation}
where %$b_2 = {\frac{(1-\rho^2)^2(1-\mu^2)^2}{2} + 2 \rho^2 \mu^4 - (1-\rho^2) (1-\mu^2) \sqrt{\frac{(1-\rho^2)^2(1-\mu^2)^2}{4} + \rho^2 \mu^4}}$
$b_{2,1} = \frac{(1-\rho^2)(1-\mu^2)}{2} - \sqrt{\frac{(1-\rho^2)^2(1-\mu^2)^2}{4} + \rho^2 \mu^4}$ and $b_{2,2} = -\rho \mu^2$. Thus, the eigenvalues of $\mathbf{Q}$ are $\widehat{\lambda}_1$ and $\widehat{\lambda}_2$; the corresponding eigenvectors are  $\widehat{\mathbf{v}}_1 = [\overline{\mathbf{{e}}}, \overline{\mathbf{{e}}}_{\perp}] \widehat{\mathbf{b}}_1$ and $\widehat{\mathbf{v}}_2 = [\overline{\mathbf{{e}}}, \overline{\mathbf{{e}}}_{\perp}] \widehat{\mathbf{b}}_2$. Note that:
$$\frac{(1-\rho^2)(1-\mu^2)}{2} \leq \sqrt{\frac{(1-\rho^2)^2(1-\mu^2)^2}{4} + \rho^2 \mu^4} \leq \frac{(1-\rho^2)(1-\mu^2)}{2} + \rho \mu^2.$$
Let us set $\sqrt{\frac{(1-\rho^2)^2(1-\mu^2)^2}{4} + \rho^2 \mu^4} = \frac{(1-\rho^2)(1-\mu^2)}{2} + \beta \rho^2 \mu^2$, for some $\beta \in (0,1]$. That gives us:
\begin{equation}
    %\mu^2 = \frac{1}{1 + \frac{1 - (\beta \rho)^2}{\beta(1-\rho^2)}}.
    \mu = \sqrt{\frac{\beta (1-\rho^2)}{(1+\beta)(1-\beta \rho^2)}}.
\end{equation}
In that case, we have:
\begin{equation}
    \label{eq:18-jan15}
    \widehat{\lambda}_1 = \frac{1 + \beta \rho^2}{1 + \beta} \text{ and } \widehat{\lambda}_2 = \rho^2\Bigg(\frac{1 - \beta}{1 - \beta \rho^2}\Bigg).
\end{equation}
Also,
\begin{equation}
    b_{1,1} = \frac{1-\rho^2}{(1+\beta)(1-\beta \rho^2)}, b_{1,2} = b_{2,2} = -\frac{\beta \rho (1-\rho^2)}{(1+\beta)(1-\beta \rho^2)}, \text{ and } b_{2,1} = -\frac{\beta^2 \rho^2 (1-\rho^2)}{(1+\beta)(1-\beta \rho^2)}.
\end{equation}
Therefore,
\begin{equation}
    \label{eq:20-jan15}
    \widehat{\mathbf{b}}_1 = \frac{1}{\sqrt{1+\beta^2 \rho^2}} \begin{bmatrix}
                    1
                    \\
                    -\beta \rho
                \end{bmatrix}
    \text{ and }
    \widehat{\mathbf{b}}_2 = \frac{1}{\sqrt{1+\beta^2 \rho^2}} \begin{bmatrix}
                    -\beta \rho
                    \\
                    -1
                \end{bmatrix}.
\end{equation}
Recall that the eigenvalues of $\mathbf{Q}$ are $\widehat{\lambda}_1$ and $\widehat{\lambda}_2$, and the corresponding eigenvectors are  $$\widehat{\mathbf{v}}_1 = [\overline{\mathbf{{e}}}, \overline{\mathbf{{e}}}_{\perp}] \widehat{\mathbf{b}}_1 = \frac{1}{\sqrt{1+\beta^2 \rho^2}} \overline{\mathbf{{e}}} - 
\frac{\beta \rho}{\sqrt{1+\beta^2 \rho^2}} \overline{\mathbf{{e}}}_{\perp} \text{ and } \widehat{\mathbf{v}}_2 = [\overline{\mathbf{{e}}}, \overline{\mathbf{{e}}}_{\perp}] \widehat{\mathbf{b}}_2 = -\frac{\beta \rho}{\sqrt{1+\beta^2 \rho^2}} \overline{\mathbf{{e}}} - 
\frac{1}{\sqrt{1+\beta^2 \rho^2}} \overline{\mathbf{{e}}}_{\perp}.$$
Finally, recall that $\mu = \sqrt{\frac{\beta (1-\rho^2)}{(1+\beta)(1-\beta \rho^2)}}$.
\end{proof}

{

%\color{red} \textbf{CHANGE FROM HERE}

%Plugging this into \cref{eq:11-jan15}, we get:
%\begin{flalign}
%    \widehat{\bm{\theta}}_K - \widetilde{\bm{\theta}}_{*} & = \widehat{\lambda}_1^K \langle \widehat{\textbf{v}}_1, \textbf{e} \rangle \widehat{\textbf{v}}_1 + \widehat{\lambda}_2^K \langle \widehat{\textbf{v}}_2, \textbf{e} \rangle \widehat{\textbf{v}}_2
%    \\
%    & = \Bigg(\frac{\widehat{\lambda}_1^K + \widehat{\lambda}_2^K \beta^2 \rho^2}{1 + \beta^2 \rho^2} \Bigg) \textbf{e} -\beta \rho \Bigg(\frac{\widehat{\lambda}_1^K -  \widehat{\lambda}_2^K}{1 + \beta^2 \rho^2} \Bigg) \|\textbf{e}\|_2 \overline{\textbf{\textup{e}}}_{\perp, 2},
%\end{flalign}
%where we have plugged in the values of $\widehat{\textbf{v}}_1$ and $\widehat{\textbf{v}}_2$, and we have also used the fact that $\overline{\textbf{e}} := \frac{\textbf{e}}{\|\textbf{e}\|_2}$. Recall that $\widehat{\lambda}_1 = \frac{1 + \beta \rho^2}{1 + \beta}$ and $\widehat{\lambda}_2 = \rho^2\Big(\frac{1 - \beta}{1 - \beta \rho^2}\Big)$. 

%Finally, recall that we made the series of substitutions $\tau = \alpha \|\textbf{e}\|_2^2$, $\mu = \big(\frac{\alpha}{\alpha+2}\big)^{1/2}$ and $\mu^2 = \frac{1}{1 + \frac{1 - (\beta \rho)^2}{\beta(1-\rho^2)}}$. Simplifying $\tau$ in terms of $\beta$, we get:
%\begin{equation}
%    \tau = \frac{2 \beta (1-\rho^2) \|\textbf{e}\|_2^2}{(1 - \beta^2 \rho^2)}.
%\end{equation}
%Also, our learning rate $\eta = \frac{1}{2 \mu}$ in terms of $\beta$ is $\eta = \frac{1}{2}\sqrt{\frac{(1+\beta)(1-\beta \rho^2)}{\beta (1-\rho^2)}}$.

%This finishes the proof.
}

\begin{lemma}
    \label{lem-4}
    Recall that the averaged model with parameter $\omega$ as defined in \cref{eq:16-jan18} was $$\bm{\theta}_\textup{avg}(\omega) = \omega \bm{\theta}_{*} + (1-\omega) \widetilde{\bm{\theta}}_{*} = \widetilde{\bm{\theta}}_{*} + \omega \mathbf{{e}}.$$
    We have:
    \begin{equation}
        \min_{\omega \in [0,1]} \textup{err}_\textup{tot}\big(\bm{\theta}_\textup{avg}(\omega)\big) =  \Bigg(\frac{\overline{\mathbf{{e}}}^\top \bm{\Sigma} \overline{\mathbf{{e}}}}{\overline{\mathbf{{e}}}^\top \bm{\Sigma} \overline{\mathbf{{e}}} + 1}\Bigg) \|{\mathbf{{e}}}\|_2^2,
    \end{equation}
    where recall that $\bm{\Sigma}$ is the covariance matrix of the pre-training data. 
\end{lemma}
\begin{proof}
    We have:
    \begin{equation}
    \textup{err}_\textup{tot}\big(\bm{\theta}_\textup{avg}(\omega)\big) = \textup{err}_1\big(\bm{\theta}_\textup{avg}(\omega)\big) + \textup{err}_2\big(\bm{\theta}_\textup{avg}(\omega)\big) = \big(\bm{\theta}_\textup{avg}(\omega) - \bm{\theta}_{\ast}\big)^\top \bm{\Sigma} \big(\bm{\theta}_\textup{avg}(\omega) - \bm{\theta}_{\ast}\big) + \big(\bm{\theta}_\textup{avg}(\omega) - \widetilde{\bm{\theta}}_{*}\big)^\top \widetilde{\bm{\Sigma}} \big(\bm{\theta}_\textup{avg}(\omega) - \widetilde{\bm{\theta}}_{*}\big).
    \end{equation}
    Plugging in the value of $\bm{\theta}_\textup{avg}(\omega)$ and using the value of $\widetilde{\bm{\Sigma}}$ from \cref{eq:1-jan15} above, we get:
    \begin{flalign}
    \label{eq:14-dec28}
    \text{err}_\text{tot}(\bm{\theta}_\textup{avg}(\omega)) &= (1-\omega)^2 {\mathbf{e}}^\top \bm{\Sigma} {\mathbf{e}} + \omega^2 \|{\mathbf{e}}\|_2^2.
    \end{flalign}
    It can be verified (with elementary calculus) that the optimal value of $\omega$ that minimizes the RHS in \cref{eq:14-dec28} is $\omega^{\ast} = \frac{{\mathbf{e}}^\top \bm{\Sigma} {\mathbf{e}}}{{\mathbf{e}}^\top \bm{\Sigma} {\mathbf{e}} + \|{\mathbf{e}}\|_2^2}$. Plugging this into \cref{eq:14-dec28} and simplifying a bit yields the desired result.
\end{proof}

\begin{lemma}
    \label{lem1-jan30}
    Suppose $\alpha > 0$ and ${\mathbf{r}} \in \mathbb{R}^d$ is a unit-norm vector, i.e., $\|{\mathbf{r}}\|_2 = 1$.
    Let $${\mathbf{M}} := \mathbb{E}\Bigg[\exp\Bigg(-\frac{\big(\langle {\mathbf{r}}, {\mathbf{z}} \rangle\big)^2}{\alpha}\Bigg) {\mathbf{z}} {\mathbf{z}}^\top \Bigg],$$
    where ${\mathbf{z}} \sim \mathcal{N}(\vec{0}_d, {\mathbf{I}}_d)$.
    ${\mathbf{r}}$ is an eigenvector of ${\mathbf{M}}$ with eigenvalue $\big(\frac{\alpha}{\alpha+2}\big)^{3/2}$. 
    Further, the eigenvectors of ${\mathbf{M}}$ in the subspace of $\mathbb{R}^d$ orthogonal to ${\mathbf{r}}$ all have eigenvalues $\big(\frac{\alpha}{\alpha+2}\big)^{1/2}$.
\end{lemma}
\begin{proof}
    We have:
    \begin{flalign}
        \label{eq:10}
        \mathbb{E}\Big[{\mathbf{M}} {\mathbf{r}}\Big] &= \mathbb{E}\Bigg[\exp\Bigg(-\frac{\big(\langle {\mathbf{r}}, {\mathbf{z}} \rangle\big)^2}{\alpha}\Bigg) \langle {\mathbf{r}}, {\mathbf{z}} \rangle {\mathbf{z}} \Bigg].
    \end{flalign}
    Suppose $\{{\mathbf{{r}}}_{\perp, j}\}_{j=1}^{d-1}$ is an orthonormal basis for the subspace orthogonal to $\mathbf{r}$; so $\langle {\mathbf{{r}}}_{\perp, j}, \mathbf{r} \rangle = 0$ $\forall$ $j \in [d-1]$ and $\langle {\mathbf{{r}}}_{\perp, j}, {\mathbf{{r}}}_{\perp, k} \rangle = \mathds{1}(j=k)$ $\forall$ $j, k \in [d-1]$. Then, note that:
    \begin{equation}
        {\mathbf{z}} = \langle \mathbf{{r}}, {\mathbf{z}} \rangle {\mathbf{r}} + \sum_{j = 1}^{d-1} \langle \mathbf{{r}}_{\perp, j}, {\mathbf{z}} \rangle \mathbf{{r}}_{\perp, j}.
    \end{equation}
    Since ${\mathbf{z}} \sim \mathcal{N}(\vec{0}_d, \mathbf{I}_d)$, $\langle \mathbf{{r}}, {\mathbf{z}} \rangle$ and $\{\langle \mathbf{{r}}_{\perp, j}, {\mathbf{z}} \rangle\}_{j=1}^{d-1}$ are i.i.d. $\mathcal{N}(0, 1)$. Using all of this in \cref{eq:10}, we get:
    \begin{flalign}
        \mathbb{E}\Big[{\mathbf{M}} {\mathbf{r}}\Big] &= \mathbb{E}\Bigg[\exp\Bigg(-\frac{\big(\langle \mathbf{{r}}, {\mathbf{z}} \rangle\big)^2}{\alpha}\Bigg) \big(\langle \mathbf{{r}}, {\mathbf{z}} \rangle\big)^2 \Bigg] \mathbf{{r}} + \sum_{j=1}^{d-1}\underbrace{\mathbb{E}\Bigg[\exp\Bigg(-\frac{\big(\langle \mathbf{{r}}, {\mathbf{z}} \rangle\big)^2}{\alpha}\Bigg) \langle \mathbf{{r}}, {\mathbf{z}} \rangle \langle \mathbf{{r}}_{\perp, j}, {\mathbf{z}} \rangle \Bigg]}_{=0 \text{ ($\langle \mathbf{{r}}, {\mathbf{z}} \rangle$ and $\langle \mathbf{{r}}_{\perp, j}, {\mathbf{z}} \rangle$ are independent)}} \mathbf{{r}}_{\perp, j}
        \\
        & = \mathbb{E}_{Z \sim \mathcal{N}(0,1)}\Bigg[\exp\Big(-\frac{Z^2}{\alpha}\Big) Z^2\Bigg] \mathbf{{r}} \quad \quad \quad  \quad \quad \quad \quad \quad \quad \text{(because $\langle \mathbf{{r}}, {\mathbf{z}} \rangle \sim \mathcal{N}(0,1)$)}
        \\
        & = \Bigg(\frac{1}{\sqrt{2 \pi}} \int_{-\infty}^{\infty} {z}^2 \exp\Big(-{z}^2 \Big(\frac{1}{\alpha} + \frac{1}{2}\Big)\Big) {dz}\Bigg) \mathbf{{r}}
        \\
        & = \Big(\frac{\alpha}{\alpha+2}\Big)^{3/2} \mathbf{{r}}.
    \end{flalign}
    So $\mathbf{{r}}$ is an eigenvector of $\mathbf{M}$ with eigenvalue $\big(\frac{\alpha}{\alpha+2}\big)^{3/2}$.
    \\
    \\
    Next, note that:
    \begin{multline}
        \mathbb{E}\Big[\mathbf{M} {\mathbf{r}}_{\perp, 1}\Big] = \underbrace{\mathbb{E}\Bigg[\exp\Bigg(-\frac{\big(\langle \mathbf{{r}}, {\mathbf{z}} \rangle\big)^2}{\alpha}\Bigg) \langle {\mathbf{r}}_{\perp, 1}, {\mathbf{z}} \rangle \langle {\mathbf{r}}, {\mathbf{z}} \rangle \Bigg]}_{=0} {\mathbf{r}} 
        + \mathbb{E}\Bigg[\exp\Bigg(-\frac{\big(\langle \mathbf{{r}}, {\mathbf{z}} \rangle\big)^2}{\alpha}\Bigg) \big(\langle {\mathbf{r}}_{\perp, 1}, {\mathbf{z}} \rangle\big)^2 \Bigg] {\mathbf{r}}_{\perp, 1}
        \\
        \sum_{j=2}^{d-1}\underbrace{\mathbb{E}\Bigg[\exp\Bigg(-\frac{\big(\langle \mathbf{{r}}, {\mathbf{z}} \rangle\big)^2}{\alpha}\Bigg) \langle {\mathbf{r}}_{\perp, 1}, {\mathbf{z}} \rangle \langle {\mathbf{r}}_{\perp, j}, {\mathbf{z}} \rangle \Bigg]}_{=0} {\mathbf{r}}_{\perp, j}.
    \end{multline}
    In the above equation, the first and last terms are $0$ because $\langle \mathbf{{r}}, {\mathbf{z}} \rangle$ and $\{\langle \mathbf{{r}}_{\perp, j}, {\mathbf{z}} \rangle\}_{j=1}^{d-1}$ are i.i.d. $\mathcal{N}(0, 1)$; using this fact again, we get:
    \begin{flalign}
        \mathbb{E}\Big[\mathbf{M} {\mathbf{r}}_{\perp, 1}\Big] & = \mathbb{E}_{Z \sim \mathcal{N}(0,1)}\Bigg[\exp\Big(-\frac{Z^2}{\alpha}\Big) \Bigg] \underbrace{\mathbb{E}_{\bar{Z} \sim \mathcal{N}(0,1)}\big[\bar{Z}^2\big]}_{=1}{\mathbf{r}}_{\perp, 1} 
        \\
        & = \Bigg(\frac{1}{\sqrt{2 \pi}} \int_{-\infty}^{\infty} \exp\Big(-{z}^2 \Big(\frac{1}{\alpha} + \frac{1}{2}\Big)\Big) {dz}\Bigg)  {\mathbf{r}}_{\perp, 1} 
        \\
        & = \Big(\frac{\alpha}{\alpha+2}\Big)^{1/2} {\mathbf{r}}_{\perp, 1}.
    \end{flalign}
    Similarly, we can show that for $j = \{2,\ldots,d-1\}$, we have:
    \begin{equation}
        \mathbb{E}\Big[\mathbf{M} {\mathbf{r}}_{\perp, j}\Big] = \Big(\frac{\alpha}{\alpha+2}\Big)^{1/2} {\mathbf{r}}_{\perp, j}.
    \end{equation}
    So for all $j \in [d-1]$, ${\mathbf{r}}_{\perp, j}$ is an eigenvector of $\mathbf{M}$ with eigenvalue $\big(\frac{\alpha}{\alpha+2}\big)^{1/2}$. Thus, the eigenvectors of $\mathbf{M}$ in the subspace orthogonal to $\mathbf{r}$ all have eigenvalues $\big(\frac{\alpha}{\alpha+2}\big)^{1/2}$.
\end{proof}

%\newpage

\section{Experimental Details}
\label{app:experimental-details}

In this section, we further discuss the experimental setup of \method's usage in both language and vision settings, specifically covering the following:

\begin{itemize}
    \item \Cref{app:add-baseline-detials}: Baseline Details.
    \item \Cref{app:language-hyperparameters}: Language Model Hyper-Parameters.
    \item \Cref{app:further-language-evaluation-details}: Language Model Evaluation Details.
    \item \Cref{app:vision-hyperparameters}: Vision Model Implementation Details.
    % \item \Cref{app:further-vision-evaluation-details}: Assessment criteria details for vision tasks
\end{itemize}

\subsection{Baseline Details}
\label{app:add-baseline-detials}

In this section, we further discuss the baselines mentioned in \cref{sec: experiments}. %, providing a b explanation of each method.

\paragraph{Linear Probing:} In our vision experiments, we define linear probing as freezing the body of the pre-trained model, initializing a new (task-specific) head and batch normalization layers, and training only the new head and batch normalization layers. 
%In our language experiments, we define linear probing as freezing the body of the pre-trained model and then training only the (common) head which is initialized to the pre-trained model's head. 

%For our vision experiments, given a pre-trained model $\bm{\theta}_{*} = {\mathbf{U}}_{*} \cup {\mathbf{V}}_{*}$, we let ${\mathbf{U}}_{*}$ denote the body of the pre-trained model and ${\mathbf{V}}_{*}$ denote the head and batch normalization layers of the model. In our vision experiments, we define linear probing fine-tuning as freezing the weights of the model ${\mathbf{U}}_{*}$, initializing a new random head and batch normalization layers $\hat{\mathbf{V}}_{*}$, and training only the new head and batch normalization layers. For our language model experiments, given a pre-trained model $\bm{\theta}_{*} = {\mathbf{U}}_{*} \cup {\mathbf{V}}_{*}$, we let ${\mathbf{U}}_{*}$ denote the body of the pre-trained model and ${\mathbf{V}}_{*}$ denote only the head of the model. In our language experiments, we define linear probing fine-tuning as freezing the weights of the model ${\mathbf{U}}_{*}$, initializing from the pre-trained head ${\mathbf{V}}_{*}$, and then training only the head. 

\paragraph{${\ell_2}$ regularization:} 
% \draft{The natural extension of EWC. We adopt the baseline of $\ell_2$ regularization from \citet{kirkpatrick2016overcoming}, where for a given pre-trained model $\bm{\theta}^*$ the fine-tuning loss is defined as}
Based on \citet{kirkpatrick2016overcoming}, we perform $\ell_2$ regularization as a baseline in the data-oblivious setting. Specifically, the $\ell_2$-regularized loss is:
\begin{equation}
    \mathcal L(\bm{\theta}) = \sum_{i=1}^n f_i(\bm{\theta}) + \lambda \|\bm{\theta} - \bm{\theta}^*\|_2^2
\end{equation}
where $f_i$ is the $i^\text{th}$ sample's loss, $\bm{\theta}^*$ is the pre-trained model, and $\lambda$ is the regularization parameter. Intuitively, as $\lambda$ increases, our model stays closer to the pre-trained model, mitigating forgetting at the expense of target domain %fine-tuning 
performance.

\paragraph{LoRA \cite{hu2022lora}:} Recently, \citet{biderman2024lora} showed  that fine-tuning language models with LoRA \cite{hu2022lora} effectively mitigates forgetting. 
% at the expense of poorer target domain performance in comparison to standard fine-tuning. 
Following a similar setup as us, \citet{biderman2024lora} fine-tuned language models on MetaMathQA \cite{yu2023metamath} and then evaluated the fine-tuned model on several general capability tasks, viz., HellaSwag \cite{zellers2019hellaswag}, ARC-c \cite{Clark2018ThinkYH}, and WinoGrande \cite{sakaguchi2019winogrande}, and one target domain task, viz., GSM8K \cite{cobbe2021training}. %While \citet{biderman2024lora} focused on multiple ranks $r$ for LoRA \cite{hu2022lora}, our experiments for language focus on a single rank. %which we pick without knowledge of the pre-training data. 
Further details about experimental hyper-parameters can be found in \cref{app:language-hyperparameters}.

\paragraph{WiSE-FT \cite{wortsman2021robust}:} We also consider model averaging as a baseline, specifically focusing on WiSE-FT \cite{wortsman2021robust}. WiSE-FT is simply the convex combination of the model parameters shared between the two tasks, while the task-specific parts are not averaged. Specifically, we perform model averaging between the pre-trained model and the fine-tuned model. The convex combination parameter $\alpha$ of WiSE-FT is set to $0.5$ in our experiments, as we cannot optimize $\alpha$ in the data-oblivious setting.

\paragraph{Learning without Forgetting \cite{li2016learning}:} {This work considers using a distillation-based loss to mitigate forgetting when the data from the training of previous tasks is not available. Initially, they record the responses $y_o$ 
%of the original network on new task images for all old task outputs (pre-trained body  + task-specific model components per task), 
on the new task images with the old tasks' model parameters, then train the model using a combination of fine-tuning loss, distillation loss, and model regularization. For new tasks, they use standard cross-entropy loss, %$L_{new}(y_n, \hat{y}_n) = -y_n \cdot \log \hat{y}_n$, 
while for old tasks they employ distillation loss %$L_{old}(y_o, \hat{y}_o)$ 
that encourages the updated model's responses to match the recorded responses $y_o$ of the original model.  Their proposed method has a distillation loss scaling factor $\lambda_0$ and a %softening parameter $T$ for the output; $y’_i = (y’_i)^{1/T} / \sum_k^l (y’_k)^{1/T}$, 
temperature parameter $T$ in the distillation loss, introducing extra tunable parameters. 
Their loss induces joint optimization of shared parameters $\theta_s$, old task parameters $\theta_o$, and new task parameters $\theta_n$ using only new task data. Note that unlike \method, they also update the old task parameters $\theta_o$.}

%For our experiments, given a pre-trained model $\bm{\theta}_{*} = {\mathbf{U}}_{*} \cup {\mathbf{V}}_{*}$, where ${\mathbf{U}}_{*}$ are model weights shared between tasks and ${\mathbf{V}}_{*}$ are model weight unique to each task, and a fine-tuned model defined as $\hat{\bm{\theta}}_{*} = {\hat{\mathbf{U}}}_{*} \cup {\hat{\mathbf{V}}}_{*}$, we define a WiSE-FT \cite{wortsman2021robust} model $\bar \theta$ as the following
%\begin{equation}
%    \overline{\bm{\theta}} = \left(\alpha \hat{\bm{U}}_{*} + (1 - \alpha) \bm{U}_{*} \right)\cup {\mathbf{V}}_{*} \cup {\hat{\mathbf{V}}}_{*}
%\end{equation}
%where the hyper-parameter $\alpha \in [0,1]$. Following WiSE-FT \cite{wortsman2021robust}, we take $\alpha = 0.5 $ for our experiments.

\subsection{Language Model Hyper-Parameters}
\label{app:language-hyperparameters} 

For both Gemma 2 2B \cite{gemmateam2024gemma2improvingopen} and Llama 3.2 3B \cite{grattafiori2024llama3herdmodels}, we run hyper-parameter sweeps on learning rates for each baseline. For standard fine-tuning, $\ell_2$ regularization, and \method, we do a learning rate sweep in [1e-4, 2e-5, 1e-5, 5e-6], and for LoRA ($r=64$) we do a sweep in [2e-4, 2e-1], following the learning rates used in \cite{biderman2024lora}. We then select the learning rate that results in the best GSM8K \cite{cobbe2021training} accuracy, oblivious to general capability metrics. We report the hyper-parameters used for our Gemma 2 2B \cite{gemmateam2024gemma2improvingopen} experiments in \cref{tab:gemma-hyperparams} and for Llama 3.2 3B \cite{grattafiori2024llama3herdmodels} in \cref{tab:llama-hyperparams}.

\begin{table}[h!]
    \centering
    \caption{The hyper-parameters used to train Gemma 2 2B in our experiments. Note that the learning rate selected is based on the best results on GSM8K after fine-tuning the method on MetaMathQA.}
    \vspace{0.2cm}
    \begin{tabular}{lcccc}
        \toprule
        \textbf{Hyper-parameter} & \textbf{Standard Fine-tuning} & \textbf{LoRA ($r=64$)} & \textbf{\(\ell_2\)-Reg.} & \textbf{\methodbold (Ours)} \\
        \midrule
        Learning Rate & 1e-5 & 2e-4 & 5e-6 & 5e-6 \\
        \midrule
        \multirow{1}{*}{Learning Rate Scheduler} & \multicolumn{4}{c}{Cosine} \\
        \multirow{1}{*}{Batch Size} & \multicolumn{4}{c}{128} \\
        \multirow{1}{*}{Optimizer} & \multicolumn{4}{c}{AdamW} \\
        \multirow{1}{*}{Weight Decay} & \multicolumn{4}{c}{0.00} \\
        \multirow{1}{*}{Warmup Ratio} & \multicolumn{4}{c}{0.03} \\
        \multirow{1}{*}{Epochs} & \multicolumn{4}{c}{2} \\
        \multirow{1}{*}{Max Sequence Length} & \multicolumn{4}{c}{1024} \\
        \multirow{1}{*}{Seed} & \multicolumn{4}{c}{42} \\
        \bottomrule
    \end{tabular}
    \label{tab:gemma-hyperparams}
\end{table}

% \begin{table}[h!]
%     \centering
%     \caption{The hyper-parameters used to train Gemma 2 2B \cite{gemmateam2024gemma2improvingopen} in our experiments. Note that the learning rate selected is based on the best results on GSM8K \cite{cobbe2021training} after fine-tuning the method on MetaMathQA \cite{yu2023metamath}.}
%     \vspace{0.2cm}
%     \begin{tabular}{lccccc}
%         \toprule
%         \textbf{Hyper-parameter} & \textbf{Standard Fine-tuning} & \textbf{Linear Probing} & \textbf{LoRA ($r=64$)} & \textbf{\(\ell_2\)-Reg.} & \textbf{\methodbold (Ours)} \\
%         \midrule
%         Learning Rate & 1e-5 & 1e-4 & 2e-4 & 5e-6 & 5e-6 \\
%         \midrule
%         \multirow{1}{*}{Learning Rate Scheduler} & \multicolumn{5}{c}{Cosine} \\
%         \multirow{1}{*}{Batch Size} & \multicolumn{5}{c}{128} \\
%         \multirow{1}{*}{Optimizer} & \multicolumn{5}{c}{AdamW} \\
%         \multirow{1}{*}{Weight Decay} & \multicolumn{5}{c}{0.00} \\
%         \multirow{1}{*}{Warmup Ratio} & \multicolumn{5}{c}{0.03} \\
%         \multirow{1}{*}{Epochs} & \multicolumn{5}{c}{2} \\
%         \multirow{1}{*}{Max Sequence Length} & \multicolumn{5}{c}{1024} \\
%         \multirow{1}{*}{Seed} & \multicolumn{5}{c}{42} \\
%         \bottomrule
%     \end{tabular}
%     \label{tab:gemma-hyperparams}
% \end{table}

\begin{table}[h!]
    \centering
    \caption{The hyper-parameters used to train Llama 3.2 3B in our experiments. Note that the learning rate selected is based on the best results on GSM8K after fine-tuning the method on MetaMathQA.}
    \vspace{0.2cm}
    \begin{tabular}{lccccc}
        \toprule
        \textbf{Hyper-parameter} & \textbf{Standard Fine-tuning} & \textbf{LoRA ($r=64$)} & \textbf{\(\ell_2\)-Reg.} & \textbf{\methodbold (Ours)} \\
        \midrule
        Learning Rate & 2e-5 & 2e-4 & 1e-5 & 1e-5 \\
        \midrule
        \multirow{1}{*}{Learning Rate Scheduler} & \multicolumn{4}{c}{Cosine} \\
        \multirow{1}{*}{Batch Size} & \multicolumn{4}{c}{128} \\
        \multirow{1}{*}{Optimizer} & \multicolumn{4}{c}{AdamW} \\
        \multirow{1}{*}{Weight Decay} & \multicolumn{4}{c}{0.00} \\
        \multirow{1}{*}{Warmup Ratio} & \multicolumn{4}{c}{0.03} \\
        \multirow{1}{*}{Epochs} & \multicolumn{4}{c}{2} \\
        \multirow{1}{*}{Max Sequence Length} & \multicolumn{4}{c}{1024} \\
        \multirow{1}{*}{Seed} & \multicolumn{4}{c}{42} \\
        \bottomrule
    \end{tabular}
    \label{tab:llama-hyperparams}
\end{table}

% \begin{table}[h!]
%     \centering
%     \caption{The hyper-parameters used to train Llama 3.2 3B \cite{grattafiori2024llama3herdmodels} in our experiments. Note that the learning rate selected is based on the best results on GSM8K \cite{cobbe2021training} after fine-tuning the method on MetaMathQA \cite{yu2023metamath}.}
%     \vspace{0.2cm}
%     \begin{tabular}{lccccc}
%         \toprule
%         \textbf{Hyper-parameter} & \textbf{Standard Fine-tuning} & \textbf{Linear Probing} & \textbf{LoRA ($r=64$)} & \textbf{\(\ell_2\)-Reg.} & \textbf{\methodbold (Ours)} \\
%         \midrule
%         Learning Rate & 2e-5 & 2e-5 & 2e-4 & 1e-5 & 1e-5 \\
%         \midrule
%         \multirow{1}{*}{Learning Rate Scheduler} & \multicolumn{5}{c}{Cosine} \\
%         \multirow{1}{*}{Batch Size} & \multicolumn{5}{c}{128} \\
%         \multirow{1}{*}{Optimizer} & \multicolumn{5}{c}{AdamW} \\
%         \multirow{1}{*}{Weight Decay} & \multicolumn{5}{c}{0.00} \\
%         \multirow{1}{*}{Warmup Ratio} & \multicolumn{5}{c}{0.03} \\
%         \multirow{1}{*}{Epochs} & \multicolumn{5}{c}{2} \\
%         \multirow{1}{*}{Max Sequence Length} & \multicolumn{5}{c}{1024} \\
%         \multirow{1}{*}{Seed} & \multicolumn{5}{c}{42} \\
%         \bottomrule
%     \end{tabular}
%     \label{tab:llama-hyperparams}
% \end{table}

For our WiSE-FT \cite{wortsman2021robust} model averaging experiments, we use $\alpha = 0.5$. For our LoRA \cite{hu2022lora} experiments, we use $\alpha = r = 64$. For $\ell_2$ regularization we use $\lambda = 1e-3$ which is taken from \cite{chen2024mofo}. Most training hyper-parameters for our language experiments are taken from \citet{chen2024mofo}, with the introduction of learning rate sweeps.

\subsection{Language Model Evaluation Details}
\label{app:further-language-evaluation-details}

As described in \cref{sec:llm-experiments-setup}, we create a commonsense reasoning metric composed of the following six metrics: ARC-e \cite{Clark2018ThinkYH}, ARC-c \cite{Clark2018ThinkYH}, HellaSwag \cite{zellers2019hellaswag}, PIQA \cite{Bisk2020}, SIQA \cite{sap2019social}, and OBQA \cite{OpenBookQA2018}. On top of the commonsense metric, we evaluate MMLU \cite{hendryckstest2021} and MBPP \cite{austin2021program} to estimate the general capabilities of a language model and to measure the effects of catastrophic forgetting when fine-tuning a model on MetaMathQA \cite{yu2023metamath}. We additionally use GSM8K \cite{cobbe2021training} to evaluate the target fine-tuning performance of a given fine-tuning method. We provide a brief describe each of these evaluation metrics:

\begin{enumerate}
    \item \textbf{HellaSwag} \cite{zellers2019hellaswag}: A benchmark designed to test commonsense reasoning. HellaSwag presents a context followed by several plausible endings, and the model must choose the most appropriate continuation.
    \item \textbf{ARC Easy} \cite{Clark2018ThinkYH}: A benchmark part of the AI2 reasoning challenge designed to test basic scientific reasoning and knowledge. ARC Easy presents 5,197 multiple-choice science questions drawn from grade 3-9 standardized tests, where each question typically includes a brief scientific scenario or statement followed by four possible answer choices. 
    \item \textbf{ARC Challenge} \cite{Clark2018ThinkYH}: A benchmark part of the AI2 reasoning challenge designed to test advanced scientific reasoning and knowledge application. ARC Challenge presents 2,590 multiple-choice science questions drawn from grade 3-9 standardized tests, where each question typically includes a scientific scenario or phenomenon followed by four possible answer choices. The questions in ARC Challenge are significantly more challenging than ARC Easy.
    \item \textbf{PIQA} \cite{Bisk2020}: A benchmark designed to evaluate physical commonsense understanding in natural language. PIQA presents a goal and two possible solutions, requiring models to choose the most appropriate solution that demonstrates an understanding of everyday physical interactions.
    \item \textbf{SIQA} \cite{sap2019social}: A benchmark designed to evaluate social commonsense intelligence and emotional reasoning. SIQA presents a social situation context followed by a question and three possible answers, requiring models to demonstrate an understanding of social interactions, emotional responses, and behavioral implications.
    \item \textbf{Open Book QA} \cite{OpenBookQA2018}: A benchmark designed to assess understanding of elementary science concepts in an open-book exam format. OBQA presents 5,957 multiple-choice questions paired with a small "book" of 1,326 core science facts, requiring models to combine these facts with common knowledge to arrive at correct answers.
    \item \textbf{MMLU} \cite{hendryckstest2021}: A benchmark designed to evaluate massive multitask language understanding. MMLU presents approximately 16,000 multiple-choice questions spanning 57 subjects including mathematics, philosophy, law, and medicine, requiring models to demonstrate broad knowledge and reasoning capabilities.
    \item \textbf{MBPP} \cite{austin2021program}: A benchmark designed to evaluate basic Python programming capabilities. The entire MBPP dataset presents 974 Python programming problems, where each problem includes a natural language task description and three test cases written as assert statements, requiring models to generate functionally correct Python code solutions. 
    \item \textbf{GSM8K} \cite{cobbe2021training}: A benchmark designed to evaluate multi-step mathematical reasoning capabilities. The GSM8K test set contains 1,000 grade school math word problems, where each problem requires 2-8 steps to solve using basic arithmetic operations (addition, subtraction, multiplication, division).
\end{enumerate}

We follow the standard evaluation process for each of these datasets and specifically use \texttt{lm-evaluation-harness} \cite{eval-harness} to evaluate our experiments.

\subsection{Vision Model Implementation Details}
\label{app:vision-hyperparameters}

We performed an extensive hyper-parameter search over six learning rates (\(\text{lrs} = [0.05, 0.01, 0.005, 0.001, 0.0005, 0.0001]\)), two models, and six datasets (i.e., 72 total runs per method) for standard fine-tuning, linear probing, and \method. We chose the best learning rates associated with the highest average score over all the target (fine-tuning) datasets. Since our method is data oblivious, we do not use the validation set of ImageNet-1K other than for evaluation. For training models with $\ell_2$-regularization, we adapted the same learning rates and other related hyperparameters used for standard fine-tuning. We searched for $\lambda$ using one dataset and ResNet50 model (\(\text{$\lambda$} = [0.002, 0.00001, 0.00002]\)) and chose 0.002 based on average accuracy over target data. We chose (\(\text{$\alpha$} = 0.05\)) for WiSE-FT following \citet{wortsman2021robust}. We present all the important training details in \Cref{table:app_vis_hyparams} for the ResNet models.

% \paragraph{Hyperparameter Details for ViT-B/16.} 
% All ViT-B/16 models were fine-tuned for 8 epochs on Food-101 with a same learning rate of 5e–5 except linear probing. This rate was used for standard fine-tuning $\ell_2$-regularization ($\lambda=0.002$) and our method \method. The linear probing baseline uses a learning rate of 1e-3. The LwF baseline additionally employed a weight decay of 5e–4. For \method{}, we set $\tau$ to the 80th percentile of the pre-trained loss.

\begin{table}[!ht]
    \centering
    \caption{Hyperparameter configurations for finetuning ResNet18 and ResNet50 on the image classification datasets.}
    \label{tab:hyperparams}
    \renewcommand{\arraystretch}{1.2}
    \begin{tabular}{l l | c c c c c c}
        \toprule
        Model & Hyperparameters & CIFAR10 & CIFAR100 & Flowers102 & Caltech101 & Dogs & Cars \\
        \midrule
        & \# GPUs & \multicolumn{6}{c}{1 A6000} \\
        & Optimizer & \multicolumn{6}{c}{SGD} \\
        & LR Schedule & \multicolumn{6}{c}{Cosine (except for Linear probing)} \\
        & Weight Decay & \multicolumn{6}{c}{0.0005} \\
        & Seed & \multicolumn{6}{c}{42} \\
        & $\lambda$ for $\ell_2$-Reg. & \multicolumn{6}{c}{0.002} \\
        & $\alpha$ for WiSE-FT. & \multicolumn{6}{c}{0.05} \\
         & $\tau$ (temperature) for \method & \multicolumn{6}{c}{median loss} \\
         \midrule
          & Epochs & 20 & 25 & 25 & 30 & 30 & 30 \\
        \midrule
        \multirow{3}{*}{\rotatebox[origin=c]{90}{ResNet18}} 
        & LR-Standard fine-tuning & 5E-3 & 1E-2 & 5E-2 & 5E-3 & 1E-3 & 5E-2 \\
        & LR-Linear probing & 5E-3 & 5E-3 & 5E-2 & 1E-2 & 5E-3 & 5E-2 \\
        & LR-\method & 1E-3 & 5E-3 & 5E-2 & 1E-2 & 5E-3 & 1E-2 \\
        
        \midrule
        \multirow{3}{*}{\rotatebox[origin=c]{90}{ResNet50}} 
        & LR-Standard fine-tuning & 5E-3 & 1E-3 & 1E-2 & 5E-3 & 5E-4 & 5E-2 \\
        & LR-Linear probing & 5E-2 & 5E-2 & 5E-2 & 5E-2 & 1E-2 & 5E-2 \\
        & LR-\method & 5E-4 & 1E-3 & 1E-2 & 1E-2 & 5E-3 & 1E-2 \\
        \bottomrule
    \end{tabular}
    \label{table:app_vis_hyparams}
\end{table}

%\paragraph{Hyperparameter details for CLIP ViT-B/32.} 
%\textcolor{red}{We used the following number of epochs: 10 for CIFAR-10, 15 for CIFAR-100, 20 for Flowers102 and Caltech101, 10 for Resisc45, 12 for Stanford Dogs, and 25 for Stanford Cars. For standard fine-tuning, we used a consistent learning rate of 5e-5 across all datasets. For linear probing, where only the classification head is trained, we used a learning rate of 1e-3. For our method \method, we used a learning rate of 1e-6. We used $\tau = 90^{\text{th}}$ percentile of the linear probing loss.} %Additionally, we further-tuned the head following \Cref{alg:vision_algo}. 
%These settings were chosen to ensure a fair comparison across baselines while reflecting the optimization sensitivity of each method.

\paragraph{Hyperparameter Details for ViT-B/16.} 
{%All ViT-B/16 models were 
We fine-tuned for 8 epochs on Food-101. 
We used a learning rate of 5e–5 for standard fine-tuning, $\ell_2$-regularization (with $\lambda = 0.002$), and our method \method. For linear probing, we used a higher learning rate of 1e–3. For LwF, we used the same learning rate as standard fine-tuning (5e–5), $T=2$, and $\lambda_0$ = 5e–4. Here we used a higher value of $\tau$ for \method and set it equal to $80^\text{th}$ percentile of the pre-trained loss values.}

\paragraph{Datasets}

\begin{enumerate}
    \item \textbf{ImageNet-1K} \cite{russakovsky2015imagenet} serves as the pre-training dataset for all our vision base models. It is a widely used large-scale image classification dataset, consisting of over a million images spanning 1000 classes.
    
    \item \textbf{CIFAR-10} \cite{Krizhevsky09learningmultiple} is a widely used dataset for image classification tasks. It consists of 60,000 32x32 color images divided into ten classes, with 6,000 images per class.
    
    \item \textbf{CIFAR-100} \cite{Krizhevsky09learningmultiple} extends CIFAR-10 by providing 100 classes containing 600 images each. This dataset is used for fine-grained image classification tasks.
    
   \item \textbf{Caltech101} \cite{Li2022} comprises images of a diverse range of objects across 101 categories with diverse set of image classes.

    \item \textbf{Flowers102} \cite{nilsback2008automated} comprises 102 categories of flowers, with each category containing between 40 to 258 images. This dataset is commonly used for fine-grained image classification and flower recognition tasks.
    
    \item \textbf{Cars} \cite{krause20133d} refers to the Stanford Cars dataset, which includes 16,185 images of 196 classes of cars. It provides a rich resource for fine-grained car classification task.
    
    \item \textbf{Dogs} \cite{parkhi2012dog} pertains to the Stanford Dogs dataset, containing 20,580 images of 120 breeds of dogs. This dataset is widely used for fine-grained dog breed classification and recognition tasks.

    %\item \textcolor{red}{\textbf{Resisc45} \cite{Cheng_2017} is a remote sensing image dataset comprising of images across 45 scene classes (e.g., airport, desert, stadium), each with 700 images, designed to benchmark scene classification in aerial imagery.}

    \item {\textbf{Food101} \cite{bossard14} is a large-scale dataset for food classification containing 101 categories with 1,000 images per class, commonly used to evaluate models on fine-grained object recognition tasks.}
    
\end{enumerate}

\newpage

\section{Detailed Vision Results and Ablations}
\label{add-vis-results}
{%First, we present the detailed version of the results of Tables \ref{table:main_vision_table} and \ref{table:vision_clip_table} in Tables \ref{table:app_vision_results1} \& \ref{table:app_vision_results3} and Tables \ref{table:app_clip_vision1} \& \ref{table:app_clip_vision2}, respectively.
First, we present the detailed version of the results of Table \ref{table:main_vision_table} in Tables \ref{table:app_vision_results1} and  \ref{table:app_vision_results3}.
}

\begin{table}[ht!]
\centering
 \caption{\textbf{ResNets: Target accuracies on each of the six datasets for the results in \Cref{table:main_vision_table}.}}
 %{\textbf{Performance of vision models pre-trained on ImageNet-1K and fine-tuned on six image classification tasks.} We evaluate catastrophic forgetting by measuring the top-1 accuracy on the ImageNet-1K validation set, reflecting the model's retained pre-trained capabilities. Task-specific learning is assessed through the average accuracy across six fine-tuning datasets. }
 \vspace{0.7em}
\footnotesize
\begin{tabular}{llccccccc}
\toprule
 & \textbf{Method} & \textbf{CIFAR-10} & \textbf{CIFAR-100} & \textbf{Flowers-102} & \textbf{Caltech-101} & \textbf{Dogs} & \textbf{Cars} & \textbf{Average} \\
\midrule
\multirow{5}{*}{\rotatebox{90}{\textbf{ResNet18}}} 
 & Linear probing                  & 81.32 & 60.06 & 87.20 & 91.15 & 78.50 & 43.23 & 73.57 \\
 & Standard FT             & 96.15 & 83.42 & 92.45 & 94.02 & 80.47 & 87.91 & 89.07 \\
 & $\ell_2$-Regularization & 95.53 & 81.82 & 92.11 & 94.23 & 80.27 & 84.78 & 88.12 \\
 & WiSE-FT & 91.47 & 65.90 & 87.28 & 91.40 & 82.48 & 62.88 & 80.23 \\
 & \methodbold (Ours) & 88.25 & 78.95 & 90.01 & 93.05 & 86.20 & 67.17 & 83.93 \\
 \midrule
\multirow{5}{*}{\rotatebox{90}{\textbf{ResNet50}}} 
& Linear probing & 86.62 & 67.80 & 83.64 & 93.45 & 85.76 & 41.97 & 76.45 \\
& Standard FT & 97.61 & 86.11 & 91.74 & 96.02 & 89.26 & 89.94 & 91.78 \\
& $\ell_2$-Regularization & 97.50 & 85.77 & 91.67 & 95.85 & 89.29 & 89.42 & 91.58 \\
& WiSE-FT & 94.65 & 72.55 & 71.95 & 93.73 & 92.52 & 62.89 & 81.38 \\
& \methodbold (Ours) & 91.11 & 79.42 & 86.78 & 94.45 & 91.16 & 74.59 & 86.25 \\
\bottomrule
\end{tabular}
\label{table:app_vision_results1}
\end{table}

\begin{table}[ht!]
\centering
\caption{\textbf{ResNets: Top-1 ImageNet-1K accuracy after fine-tuning on each target dataset for the results in \Cref{table:main_vision_table}. }}
\vspace{0.7em}
\footnotesize
\begin{tabular}{llccccccc}
\toprule
& \textbf{Method} & \textbf{CIFAR-10} & \textbf{CIFAR-100} & \textbf{Flowers-102} & \textbf{Caltech-101} & \textbf{Dogs} & \textbf{Cars} & \textbf{Average} \\
\midrule
\multirow{5}{*}{\rotatebox{90}{\textbf{ResNet18}}} 
 & Linear probing & 69.76 & 69.76 & 69.76 & 69.76 & 69.76 & 69.76 & 69.76 \\
 & Standard FT & 19.93 & 0.39 & 6.48 & 34.17 & 56.38 & 0.17 & 19.58 \\
 & $\ell_2$-Regularization & 37.86 & 29.86 & 19.34 & 46.67 & 58.34 & 16.64 & 34.78 \\
 & WiSE-FT & 62.24 & 47.65 & 49.98 & 64.70 & 67.34 & 33.03 & 54.15 \\
 & \methodbold (Ours) & 69.02 & 52.64 & 67.80 & 68.32 & 67.78 & 65.74 & 65.21 \\
% & LP & \textbf{69.76} & \textbf{69.76} & \textbf{69.76} & \underline{69.76} & \underline{69.76} & \textbf{69.76} & \textbf{69.76} \\
%  & Standard FT & 19.93 & 0.39 & 6.48 & 34.17 & 56.38 & 0.17 & 19.58 \\
%  & $\ell_2$-Reg. & 37.86 & 29.86 & 19.34 & 46.67 & 58.34 & 16.64 & 34.78 \\
%  & WiSE-FT & 62.24 & 47.65 & 49.98 & 64.70 & \textbf{67.34} & 33.03 & 54.15 \\
%  & \methodbold (Ours) & \underline{69.02} & \underline{52.64} & \underline{67.80} & \textbf{68.32} & 67.78 & \underline{65.74} & \underline{65.21} \\
\midrule
\multirow{5}{*}{\rotatebox{90}{\textbf{ResNet50}}} 
& Linear probing & 79.02 & 79.02 & 79.02 & 79.02 & 79.02 & 79.02 & 79.02 \\
& Standard FT & 16.89 & 35.95 & 61.01 & 40.51 & 66.93 & 0.21 & 36.91 \\
& $\ell_2$-Regularization & 33.98 & 47.16 & 62.85 & 43.42 & 67.03 & 14.27 & 44.78 \\
 % & L1 reg & 51.63 & 51.50 & 72.68 & 71.69 & 74.03 & 7.26 & 54.79 \\
& WiseFT ($\alpha=0.5$) & 61.40 & 73.04 & 76.33 & 73.25 & 77.36 & 8.55 & 61.65 \\
 % & WiseFT + Ours (seed=42) & 78.86 & 77.98 & 78.92 & 77.63 & 78.88 & 77.50 & 78.29 \\
 % & LPFT w/ our loss & 78.26 & 75.13 & 78.60 & 73.38 & 78.55 & 72.64 & 76.09 \\
 % & LPFT vanilla & 5.51 & 42.51 & 76.94 & 26.53 & 77.52 & 0.11 & 31.18 \\
& \methodbold (Ours) & 78.26 & 75.13 & 78.60 & 73.38 & 78.55 & 72.64 & 76.09 \\
\bottomrule
\end{tabular}
\label{table:app_vision_results3}
\end{table}

 %  Post finetuned CLIP IN-1K

{In \Cref{tab:tau_ablation}, we present a small ablation to compare the pre-training and fine-tuning performances with different values of $\tau$ in \method; recall that we prescribed selecting $\tau$ to be the median pre-training loss value. As we see, if we could tune $\tau$, then \method's results would be even better.}
 
% \tau ablation table
\begin{table}[ht!]
\centering
\caption{{
\textbf{\method ablation with different values of $\tau$:} Pre-training and fine-tuning accuracies as a function of $\tau$ set to different percentiles of the pre-training losses. The model is ResNet-50, pre-training dataset is ImageNet-1K (IN-1K), and the fine-tuning dataset is Caltech101 (Target). So if we could tune $\tau$, then \method's results would further improve.}
}
\vspace{0.2cm}
\small
\begin{tabular}{cccc}
\toprule
\textbf{IN-1K Accuracy} & \textbf{Target Accuracy} & \textbf{Average} &\textbf{$\tau$- Percentile (\%)} \\
\midrule
68.51 & 91.15 & 79.83 & 10 \\
64.13 & 91.01 & 77.57 & 30 \\
54.72 & 91.80 & 73.26 & 50 \\
45.59 & 92.91 & 69.25 & 70 \\
20.51 & 94.02 & 57.27 & 90 \\
\bottomrule
\end{tabular}
\label{tab:tau_ablation}
\end{table}

{Further, in \Cref{fig:baseline_figure}, we present another ablation study where we compare \method with different values of $\tau$, WiSE-FT with different values of the convex combination parameter, and random selection where we train on a random subset of the fine-tuning data to limit the drift from the pre-trained model.}

\begin{figure}[!ht]
    \centering
    \includegraphics[width=0.5\columnwidth]{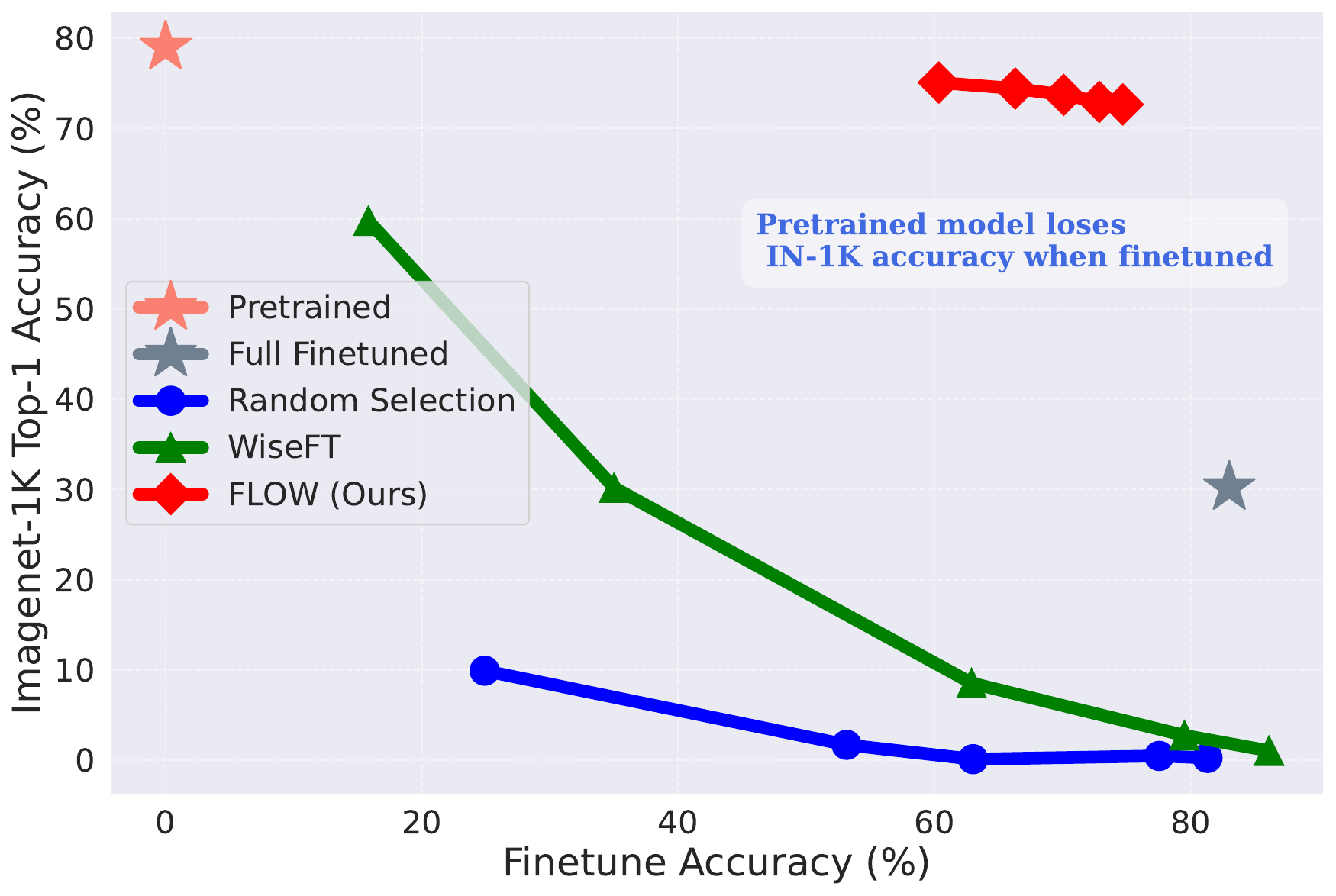}
     \caption{\textbf{Comparison of \methodbold with different values of $\tau$ and some other baselines with different hyper-parameter values.} This plot is for ResNet-50 on the Stanford cars dataset. \method's plot (in red) is with $\tau = \{10,20,30,40,50\}$ percentile of the per-sample losses. As the name \enquote{random selection} implies, we just pick a random subset of the fine-tuning data and train on this subset to limit the drift from the pre-trained model. To have some correspondence with our choice of $\tau$ for \method, we pick random $\{10,20,30,40,50\}$ \% of the data in \enquote{random selection}. As we see, \method significantly outperforms other methods.} 
    \label{fig:baseline_figure}
\end{figure}

%\newpage
{
\subsection{Comparison with a Distillation-Based Method for Mitigating Forgetting}
\label{app:dist}
Here, we compare our method \method against a distillation-based method for mitigating forgetting called \enquote{learning without forgetting} (LwF) \cite{li2016learning} in vision. More details about this method can be found in \Cref{app:add-baseline-detials}, but the important thing to note is that it updates the head corresponding to the pre-training data (which \method does not do) and also comes with more tunable parameters, additional memory, and evaluation overhead
compared to \method. We consider the ViT-B/16 \cite{dosovitskiy2020vit} model pre-trained on Imagenet-1K (IN-1K) and a single large fine-tuning dataset, Food101 \cite{bossard14}. The evaluation metric is the same as in \Cref{exp-vis}. Hyper-parameter details are in \Cref{app:vision-hyperparameters}. 

In Table \ref{table:lwf_table}, we list the accuracies of the pre-trained model, standard FT, linear probing, $\ell_2$-regularization, LwF, and \method. Note that \method outperforms LwF despite its simplicity. Specifically, \method achieves better performance on the forgetting front, while LwF does better on the fine-tuning task. 

\begin{table}[!htb]
\centering
 \caption{{{
 \textbf{Comparison with the distillation-based method (LwF) of \citet{li2016learning}.}
 \textbf{Bolded} and \underline{underlined} values indicate the \textbf{best} and \underline{second-best} %\textit{average} of ImageNet-1K (IN-1K) and fine-tuning 
 accuracies within each column (and for each model). 
 Deltas (in color) for IN-1K and target performance are computed w.r.t. the pre-trained and standard fine-tuned models. Note that \methodbold \textbf{outperforms LwF} despite not updating the head for the pre-training data, unlike LwF, and being more efficient than LwF, which comes with more tunable parameters, additional memory, and evaluation overhead.}
 }}
 \vspace{0.1 cm}
\small
\begin{tabular}{ll|l|l|c}
\toprule
 & \multicolumn{1}{c|}{{\textbf{Method}}} & \multicolumn{1}{c|}{{\textbf{IN-1K Accuracy}}} & \multicolumn{1}{c|}{{\textbf{Target Accuracy}}} & \textbf{Average} \\
\midrule
\multirow{6}{*}{\rotatebox{90}{\textbf{ViT-B/16}}} 
 & Pre-trained         & \textbf{81.10} \hfill {\tiny \textcolor{blue}{(+0.00)}} & --                 & --    \\
 & Standard FT         & 56.11 \hfill {\tiny \textcolor{red}{(-24.99)}}        & \underline{91.60} \hfill {\tiny \textcolor{blue}{(+0.00)}} & 73.86 \\
 & Linear Probe        & \textbf{81.10} \hfill {\tiny \textcolor{blue}{(+0.00)}} & 83.86 \hfill {\tiny \textcolor{red}{(-7.74)}}           & 82.48 \\
 & $\ell_2$-Reg.       & 59.18 \hfill {\tiny \textcolor{red}{(-21.92)}}        & \textbf{91.66} \hfill {\tiny \textcolor{blue}{(+0.06)}} & 75.42 \\
 & LwF                 & 76.39 \hfill {\tiny \textcolor{red}{(-4.71)}}         & 91.23 \hfill {\tiny \textcolor{red}{(-0.37)}}           & \underline{83.81} \\
 & \methodbold (Ours)  & \underline{77.94} \hfill {\tiny \textcolor{red}{(-3.16)}} & 90.57 \hfill {\tiny \textcolor{red}{(-1.03)}}           & \textbf{84.26} \\
\bottomrule
\end{tabular}
\label{table:lwf_table}
\end{table}

}

\section{Additional Language Model Results and Ablations}
\label{add-lang-results}

In this section, we discuss expanded results and further ablations of \method within our language experiments, specifically covering the following:

\begin{itemize}
    \item \Cref{app:extended-commonsense-results}: An expanded table of results on commonsense reasoning tasks along with other baselines.
    \item \Cref{app:token-wise-ablation}: An additional ablation on \textit{token}-wise weighting scheme for fine-tuning with language data.
    % \item \Cref{lora-ablation}: An additional ablation on the usage of LoRA \cite{hu2022lora}, with \method. 
    \item \Cref{app:extended-wa-results}: An expanded set of plots and results for the combination of \method with weight averaging techniques such as Wise-FT \cite{wortsman2021robust}.
\end{itemize}

\subsection{Extended Commonsense Reasoning Results}
\label{app:extended-commonsense-results}

As discussed in \cref{sec:llm-experiments-setup} and \cref{app:further-language-evaluation-details}, we evaluate \method and the other baselines on various commonsense reasoning tasks within fine-tuning with the procedure described in \cref{sec:llm-experiments-setup}. We include the exact results of these evaluation metrics for various baselines and \method in \cref{tab:commonsense-expanded}. We also include the results of commonsense reasoning metrics for the ablation combining \method with LoRA and $\ell_2$-regularization in \cref{tab:our-and-baseline-commonsense-expanded}.

\begin{table*}[h!]
    \centering
    \caption{\textbf{Extended commonsense reasoning metrics for \methodbold and other baselines within language modeling}. The performance on commonsense reasoning evaluations when fine-tuning Gemma 2 2B \cite{gemmateam2024gemma2improvingopen} and Llama 3.2 3B \cite{grattafiori2024llama3herdmodels} on MetaMathQA \cite{yu2023metamath}. We include the target domain evaluation GSM8K \cite{cobbe2021training} for convenience. The results show that \method can effectively mitigate catastrophic forgetting while still getting strong performance on our target fine-tuning task. }
    \begin{tabular}{lccccccccc}
         \toprule
         & \textbf{Method} &  \textbf{ARC-e} & \textbf{ARC-c} & \textbf{HellaSwag} & \textbf{PIQA} & \textbf{SIQA} &  \textbf{OBQA} & \textbf{Average} & \textbf{GSM8K} \\
        \midrule
        \multirow{6}{*}{\rotatebox{90}{Gemma 2 2B}}
         & Pre-trained & 80.18 & 46.84 & 54.95 & 78.67 & 51.33 & 31.40 & 57.23 & 24.49 \\
         & Standard Fine-tuning & 76.09 & 42.07 & 54.41 & 76.99 & 48.06 & 32.00 & 55.07 & 63.38 \\
         % & Linear Probing & 78.41 & 46.16 & 54.32 & 77.75 & 50.26 & 30.20 & 56.18 & 25.10 \\
         & WiSE-FT & 79.55 & 46.42 & 56.43 & 78.24 & 51.08 & 32.00 & 57.28 & 53.30 \\
         & LoRA ($r=64$) & 77.78 & 44.37 & 54.59 & 76.99 & 50.51 & 29.80 & 55.67 & 60.43 \\
         & $\ell_2$-Regularization  & 79.08 & 45.99 & 56.21 & 77.20 & 50.97 & 32.60 & 57.01 & 62.85 \\
         & \methodbold (Ours) & 79.76 & 47.18 & 56.23 & 77.69 & 51.48 & 33.20 & 57.59 & 62.55 \\
        \midrule
        \multirow{6}{*}{\rotatebox{90}{Llama 3.2 3B}}
         & Pre-trained & 74.54 & 42.15 & 55.31 & 76.66 & 47.03 & 31.20 & 54.48 & 26.01\\
         & Standard Fine-tuning & 70.03 & 34.22 & 52.02 & 74.16 & 45.24 & 28.40 & 50.68 & 66.95 \\
         % & Linear Probing & 71.76 & 40.87 & 53.91 & 76.06 & 46.32 & 30.80 & 53.29 & 27.52 \\
         & WiSE-FT & 75.63 & 40.79 & 55.18 & 76.93 & 47.34 & 31.40 & 54.54 & 57.01 \\
         & LoRA ($r=64$) & 71.38 & 37.88 & 55.01 & 76.55 & 47.39 & 30.40 & 53.10 & 63.84 \\
         & $\ell_2$-Regularization & 73.57 & 38.91 & 54.939 & 76.12 & 47.24 & 30.80 & 53.60 & 66.87 \\
         & \methodbold (Ours) &  74.96 & 39.68 & 55.39 & 76.01  & 47.80 & 32.00 & 54.30 & 65.58 \\
         \bottomrule
    \end{tabular}
    \label{tab:commonsense-expanded}
\end{table*}

\begin{table*}[h!]
    \centering
    \caption{\textbf{Extended commonsense reasoning metrics for combining \methodbold with other baselines}. The performance on commonsense reasoning evaluations when fine-tuning Gemma 2 2B \cite{gemmateam2024gemma2improvingopen} baselines in conjunction with \method on MetaMathQA \cite{yu2023metamath}. We include the target domain evaluation GSM8K \cite{cobbe2021training} for convenience. The results show that \method can effectively be used in conjunction with other methods that mitigate catastrophic forgetting.}
    \begin{tabular}{ccccccccc}
         \toprule
         \textbf{Method} &  \textbf{ARC-e} & \textbf{ARC-c} & \textbf{HellaSwag} & \textbf{PIQA} & \textbf{SIQA} &  \textbf{OBQA} & \textbf{Average} & \textbf{GSM8K} \\
        \midrule
        
         LoRA ($r=64$) & 77.78 & 44.37 & 54.59 & 76.99 & 50.51 & 29.80 & 55.67 & 60.43 \\
         LoRA ($r=64$) + \methodbold & 79.50 & 45.39 & 55.27 & 77.31 & 51.18 & 31.80 & 56.74 & 61.49 \\
         $\ell_2$-Regularization  & 79.08 & 45.99 & 56.21 & 77.20 & 50.97 & 32.60 & 57.01 & 62.85 \\
         $\ell_2$-Regularization + \methodbold & 79.67 & 47.10 & 56.38 & 77.48 & 51.13 & 33.40 & 57.53 & 62.02 \\
         \bottomrule
    \end{tabular}
    \label{tab:our-and-baseline-commonsense-expanded}
\end{table*}

\cref{tab:commonsense-expanded} shows a clear trend that \method, can strongly mitigate catastrophic forgetting in comparison to standard fine-tuning. For Gemma 2 2B \cite{gemmateam2024gemma2improvingopen}, we can see that \method only has $\sim$ 0.8\% reduction in the performance of the target fine-tuning while on average maintaining the commonsense reasoning abilities of the pre-trained model, a $\sim 2.52\%$ increase over standard fine-tuning. For Llama 3.2 3B \cite{grattafiori2024llama3herdmodels}, we can see that \method can again maintain the commonsense reasoning abilities of the base pre-trained model while only having a $\sim$1.4\% drop on target fine-tuning performance. Overall, \method strikes a strong balance between general capabilities and target fine-tuning performance compared to other baselines.

For experiments with Gemma 2 2B \cite{gemmateam2024gemma2improvingopen}, \method can on average maintain the best scores on commonsense reasoning tasks. Performing only $\sim 0.8 \%$ and $\sim 0.3 \%$ worse on GSM8K \cite{cobbe2021training} in comparison to standard fine-tuning and $\ell_2$ regularization, \method can improve on commonsense reasoning metrics by $\sim 2.42 \%$ and $\sim 0.58 \%$ respectively. Interestingly, in our Llama 3.2 3B \cite{grattafiori2024llama3herdmodels} experiments, we found that WiSE-FT \cite{wortsman2021robust} performed the strongest in preventing catastrophic forgetting of commonsense capabilities ($+0.04$ over the pre-trained model); however, this came at the cost of a significant decrease in GSM8K \cite{cobbe2021training} accuracy ($-9.94$ under standard fine-tuning). In comparison, \method effectively mitigated forgetting in commonsense reasoning metrics ($-0.18$ under the pre-trained model), while achieving significantly higher accuracy in GSM8K \cite{cobbe2021training} ($-1.37$ under standard fine-tuning).

\subsection{Token-wise Sample Weighting Ablations}
\label{app:token-wise-ablation}
{In the language experiments, \enquote{sample} for \method can be defined as an entire sequence or an individual token. The experiments in the main paper treat a sequence as a sample; in that case, the per-sample loss is the average loss over the tokens in the sequence. We call this \textit{sequence}-wise re-weighting. Instead, one could treat a token as a sample in which case the per-sample loss is just the token's loss. We call this \textit{token}-wise re-weighting. 
We run a small ablation on both \textit{sequence}-wise and \textit{token}-wise re-weighting by following a similar experimental setup as  \cref{sec:llm-experiments-setup}. We train a Gemma 2 2B \cite{gemmateam2024gemma2improvingopen} on MetaMathQA \cite{yu2023metamath} and evaluate it on several general capability and target domain evaluations. The results of this experiment are in \cref{tab:token-comparison}.}

\begin{table*}[h!]
    \centering
    \caption{The performance of Gemma 2B 2B on general capabilities metrics compared to target domain performance (GSM8K) when training on MetaMathQA. \textit{Pre-trained} is the base model performance of Gemma 2 2B, \textit{Standard} is the performance after full end-to-end fine-tuning, \textit{Sequence} is our sequence sample weighting schema with \method, and \textit{Token} is our token sample weighting schema with \method.  \textbf{Bold} and \underline{underlined} values indicate the \textbf{best} and \underline{second-best} results respectively within each evaluation metric.}
    \begin{tabular}{cccccccccc}
         \toprule
         \textbf{Method} &  \textbf{ARC-e} & \textbf{ARC-c} & \textbf{HellaSwag} & \textbf{PIQA} & \textbf{SIQA} &  \textbf{OBQA} & \textbf{MMLU} & \textbf{MBPP} & \textbf{GSM8K} \\
        \midrule
         Base & \textbf{80.18} & \underline{46.84} & \underline{54.95} & \textbf{78.67} & \underline{51.33} & 31.40 & \textbf{49.59} & \textbf{28.40} & 24.49 \\
         Standard & 76.09 & 42.07 & 54.41 & 76.99 & 48.06 & \textbf{32.80} & 45.59 & 16.80 & \textbf{63.38} \\
         Sequence & \underline{79.76} & \textbf{47.18} & \textbf{56.23} & 77.69 & \textbf{51.48} & \underline{33.20} & \underline{49.31} & \underline{26.80} & \underline{62.55} \\
         Token & 79.38 & 45.90 & 53.95 & \underline{78.29} & 51.28 & 31.80 & 48.75 & 22.00 & 23.73 \\
         \bottomrule
    \end{tabular}
    \label{tab:token-comparison}
\end{table*}

\begin{figure}[h!]
\centering
\begin{tabular}{cc}
\includegraphics[width=0.48\textwidth]{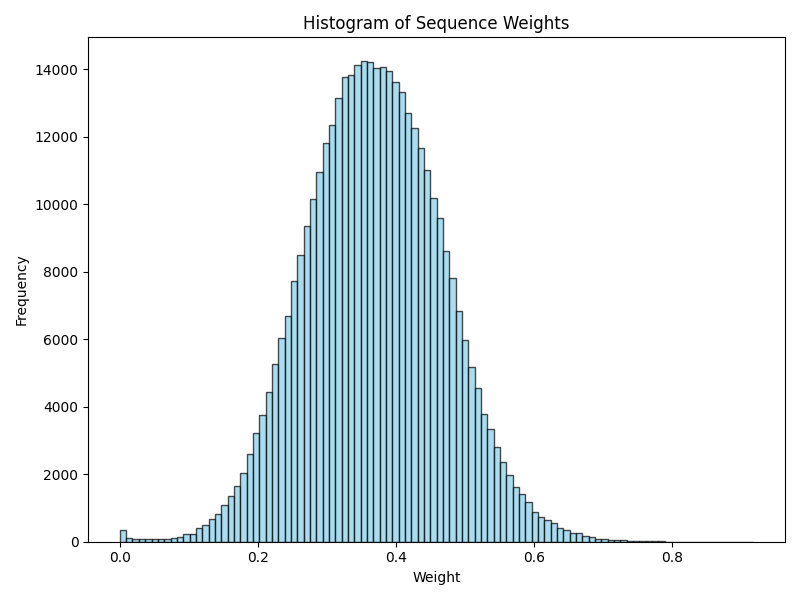} &
\includegraphics[width=0.48\textwidth]{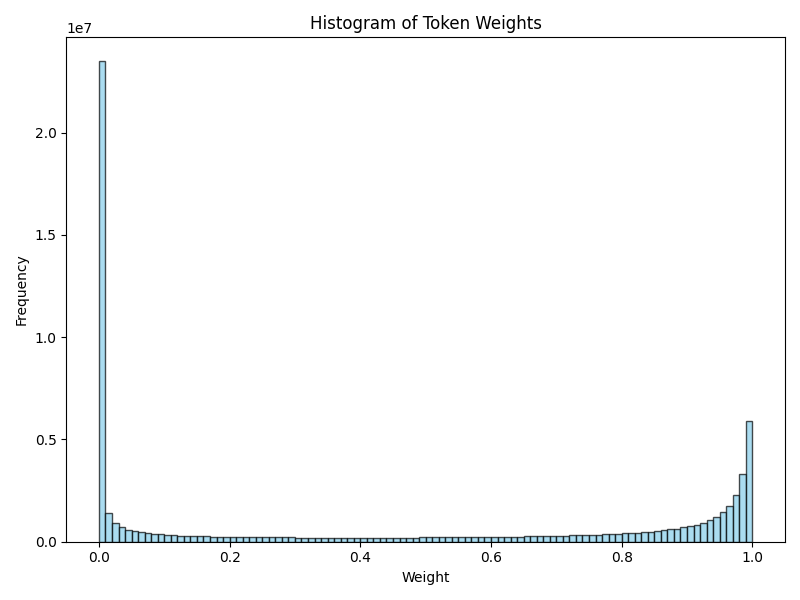} \\
\end{tabular}
\caption{Histograms comparing the sample-wise distribution of weights in \textit{sequence}-wise re-weighting schema for \method and token-wise distribution of weights \textit{token}-wise re-weighting schema for \method. The \textit{sequence}-wise weight distribution is given on the left, while the \textit{token}-wise weight distribution is given on the right.}
\label{fig:token-sequence-weight-comparison}
\end{figure}

While \textit{token}-wise sample re-weighting performs comparably or slightly worse than \textit{sequence}-wise sample re-weighting in terms of the catastrophic forgetting of general capabilities of Gemma 2 2B, it struggles to effectively learn the fine-tuning target domain of GSM8K. To further understand this problem, we compare the weight distributions between \textit{sequence}-wise and \textit{token}-wise re-weighting schema in \cref{fig:token-sequence-weight-comparison}. We can see that the \textit{sequence} weights appear Gaussian, while most of the \textit{token} weights are either 0 or 1. We speculate that \textit{token}-wise re-weighting will force any token not commonly appearing in the pre-training data to have a high loss or perplexity, which combined with our algorithm, will heavily down-weight them to almost zero. We further speculate that these tokens are essential to improving the performance of our target fine-tuning task and that using \method with a \textit{token}-wise scheme over-regularizes, preventing any meaningful learning of the target task. As \textit{sequence}-wise re-weighting significantly outperforms \textit{token}-wise re-weighting, we recommend using \textit{sequence}-wise re-weighting in \method for language models.

\subsection{Extended Weight Averaging Results}
\label{app:extended-wa-results}

As discussed in \cref{sec:results}, we further combine \method with WiSE-FT \cite{wortsman2021robust} to mitigate the effects of catastrophic forgetting when fine-tuning. In this section, we report the full results of combining \method and WiSE-FT to prevent catastrophic forgetting with Gemma 2 2B.

\begin{figure*}[h!]
    \centering
    \begin{tabular}{c@{\hspace{-2mm}}c@{\hspace{-2mm}}c}
        \includegraphics[width=0.3\textwidth]{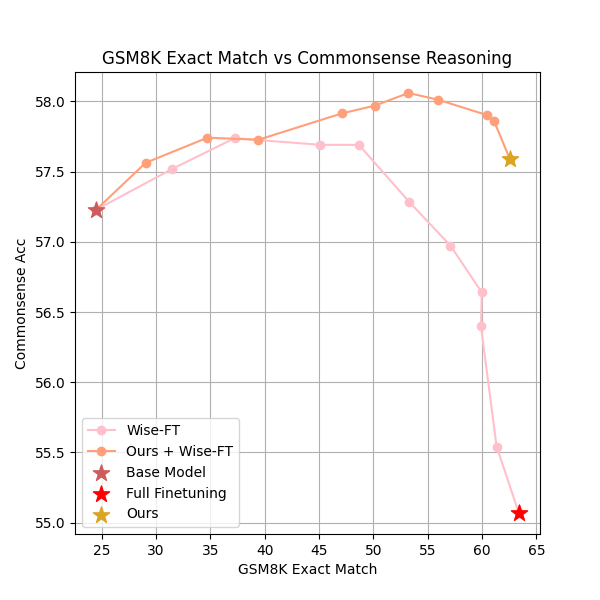} &
        \includegraphics[width=0.3\textwidth]{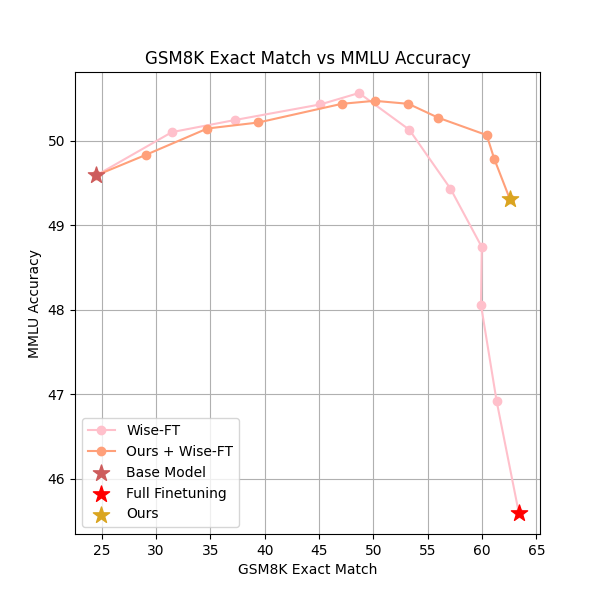} &
        \includegraphics[width=0.3\textwidth]{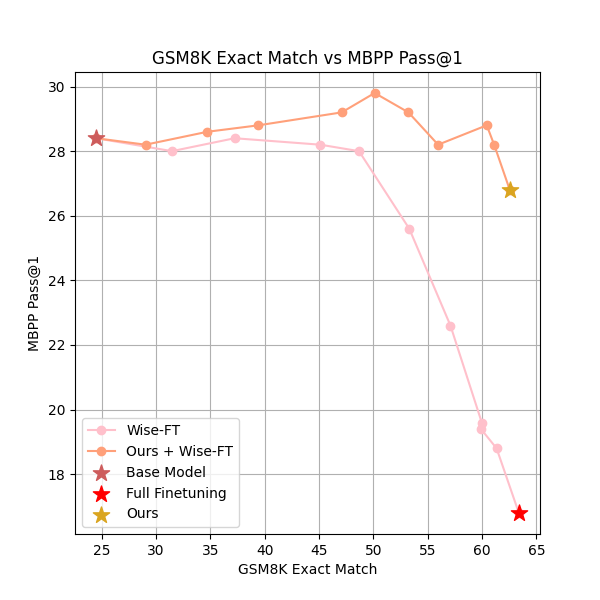} \\
        \noalign{\vspace{-3mm}} % Reduce vertical space
        \includegraphics[width=0.3\textwidth]{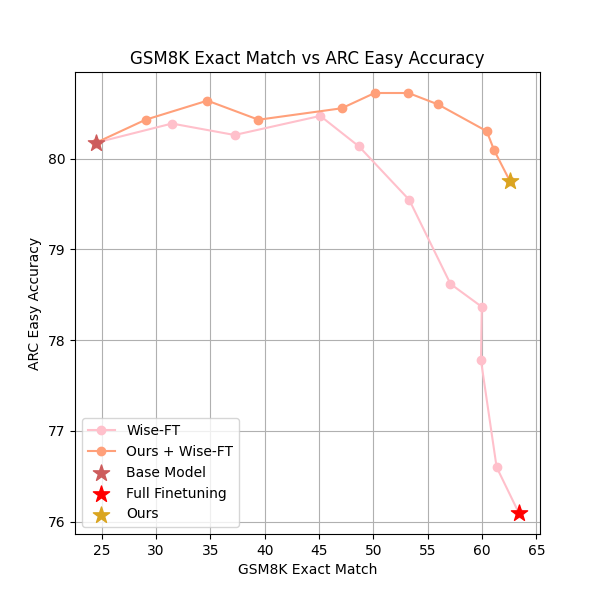} &
        \includegraphics[width=0.3\textwidth]{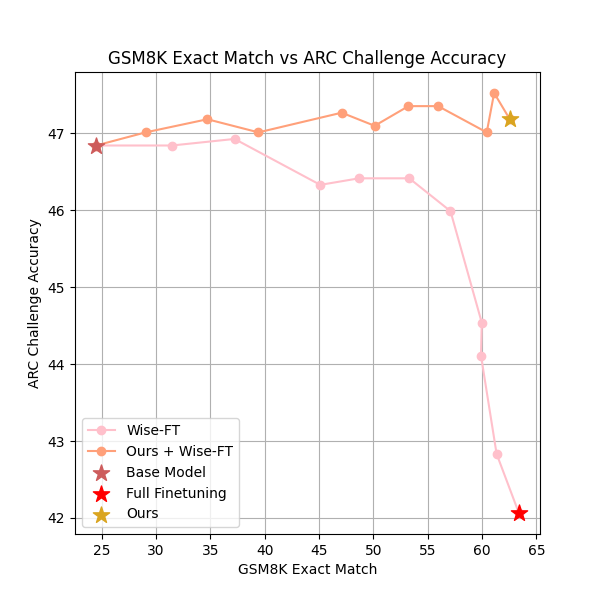} &
        \includegraphics[width=0.3\textwidth]{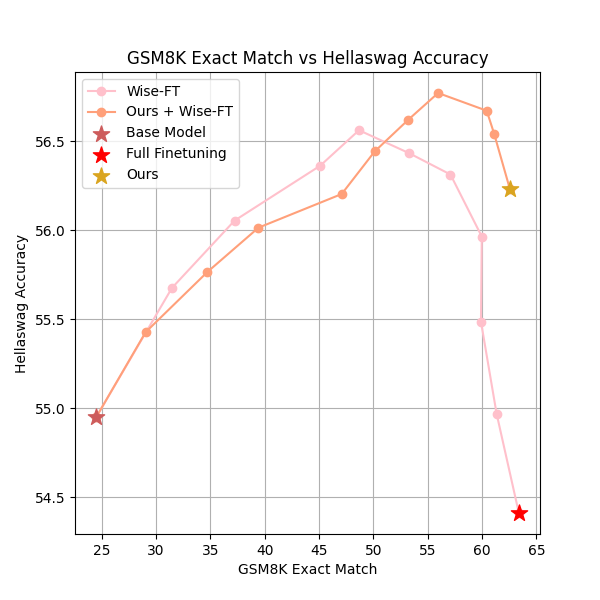} \\
        \noalign{\vspace{-3mm}} % Reduce vertical space
        \includegraphics[width=0.3\textwidth]{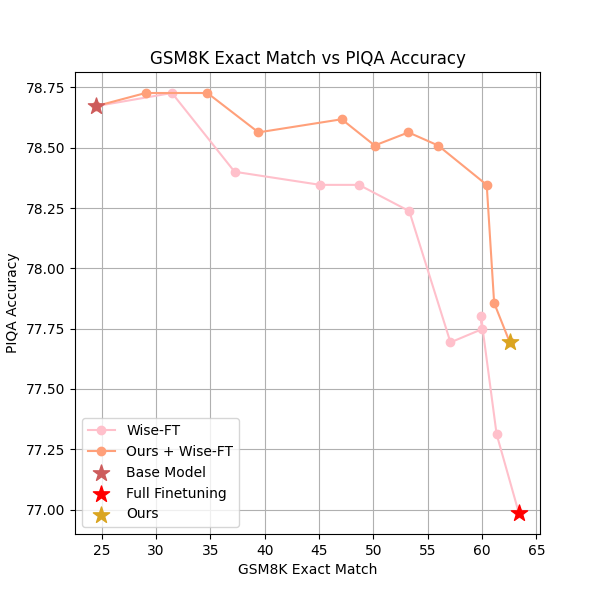} &
        \includegraphics[width=0.3\textwidth]{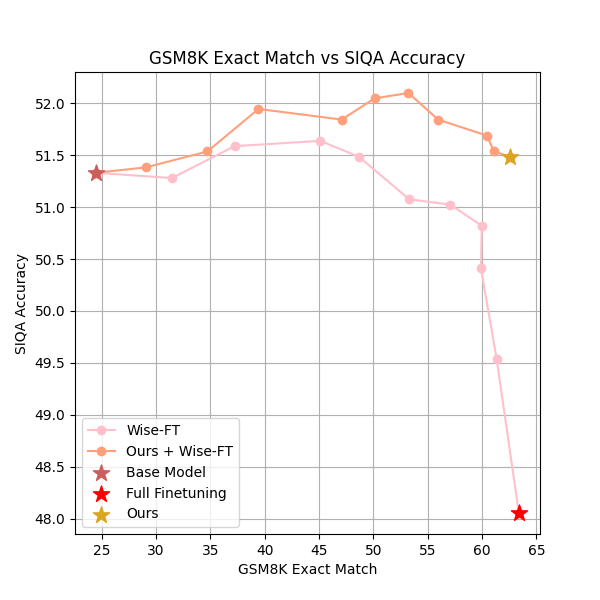} &
        \includegraphics[width=0.3\textwidth]{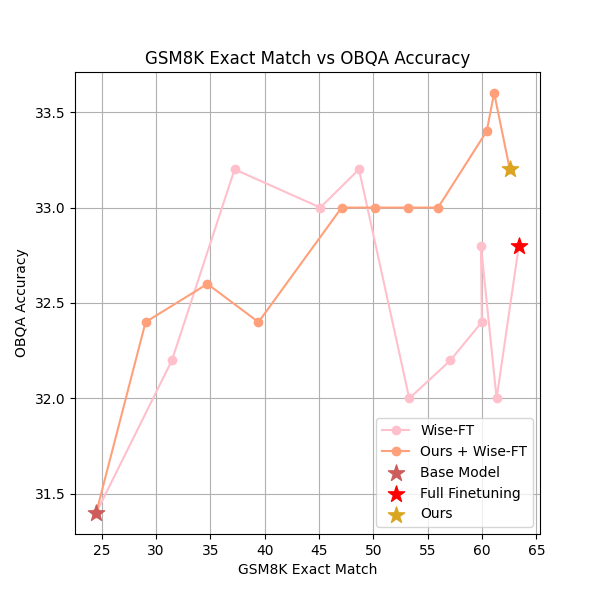} \\
        \noalign{\vspace{-3mm}} % Reduce vertical space
    \end{tabular}
    \caption{\textbf{\methodbold is complementary with model averaging (WiSE-FT) in language modeling.} We compare WiSE-FT \cite{wortsman2021robust} with a standard model fine-tuning and with \method after fine-tuning Gemma 2 2B on MetaMathQA. We use varying $\alpha \in [0,1]$ for WiSE-FT. The results indicate that combining Wise-FT with \method outperforms vanilla WiSE-FT with standard fine-tuning. }
    \label{fig:wiseft-and-ours-lang}
\end{figure*}

%%%%%%%%%%%%%%%%%%%%%%%%%%%%%%%%%%%%%%%%%%%%%%%%%%%%%%%%%%%%%%%%%%%%%%%%%%%%%%%
%%%%%%%%%%%%%%%%%%%%%%%%%%%%%%%%%%%%%%%%%%%%%%%%%%%%%%%%%%%%%%%%%%%%%%%%%%%%%%%

\end{document}